%% file: main.tex
\documentclass[11pt,a4paper]{article}
\usepackage[top=3cm, bottom=3cm, left=3cm, right=3cm]{geometry}

\usepackage{booktabs}
\usepackage{color}
\usepackage{latexsym} 
\usepackage{amsmath}               
\usepackage{amssymb}               
\usepackage{amsfonts}              
\usepackage{amsthm}                
\usepackage{tcolorbox}
\usepackage{times}
\usepackage[utf8]{inputenc}
\usepackage[T1]{fontenc}
\usepackage{enumitem}
\usepackage{xcolor}
\usepackage{thmtools}
\usepackage{thm-restate}
\usepackage[square,numbers]{natbib}
\usepackage[colorlinks=true,allcolors=blue]{hyperref}
\usepackage{cleveref}
\usepackage{url}            
\usepackage{booktabs} 
\usepackage{nicefrac}       
\usepackage{microtype}      
\usepackage[american]{babel}
\usepackage{thmtools}
\usepackage{booktabs}
\usepackage{algorithm}
\usepackage{algorithmic}
\usepackage[tikz]{bclogo}
\usepackage{multicol,tabularx,capt-of}
\usepackage{hhline}
\usepackage{multirow}
\usepackage{bibunits}

\renewenvironment{proof}[1][\proofname]{{\bfseries #1.}}{}

\makeatother

\input{macros_bis.tex}

\renewcommand{\leq}{\leqslant}
\renewcommand{\geq}{\geqslant}

\usepackage{authblk}

\author{\bfseries
Ulysse Marteau-Ferey}
\author{\bfseries
Francis Bach}
\author{\bfseries
Alessandro Rudi
}
\affil{
INRIA - D{\'e}partement d'Informatique de l'{\'E}cole Normale Sup{\'e}rieure \\
PSL Research University\\
Paris, France
}

\title{Non-parametric Models for Non-negative Functions}

\begin{document}

\maketitle

\begin{abstract}
Linear models have shown great effectiveness and flexibility in many fields such as machine learning, signal processing and statistics. They can represent  rich spaces of functions while preserving the convexity of the optimization problems where they are used, and are simple to evaluate, differentiate and integrate. However, for modeling non-negative functions, which are crucial for unsupervised learning, density estimation, or non-parametric Bayesian methods, linear models are not applicable directly. Moreover, current state-of-the-art models like generalized linear models either lead to non-convex optimization problems, or cannot be easily integrated. In this paper we provide the first model for non-negative functions which benefits from the same good properties of linear models. In particular, we prove that it admits a representer theorem and provide an efficient dual formulation for convex problems. We study its representation power, showing that the resulting space of functions is strictly richer than that of generalized linear models. Finally we extend the model and the theoretical results to functions with outputs in convex cones. The paper is complemented by an experimental evaluation of the model showing its effectiveness  in terms of formulation, algorithmic derivation and practical results on the problems of density estimation,  regression with heteroscedastic errors, and multiple quantile regression.

\end{abstract}

\begin{bibunit}[plainnat]

\section{Introduction}
The richness and flexibility of linear models, with the aid of possibly infinite-dimensional feature maps, allowed to achieve great effectiveness from a theoretical, algorithmic, and practical viewpoint in many supervised and unsupervised learning problems, becoming one of the workhorses of statistical machine learning in the past decades \cite{friedman2001elements,scholkopf2002learning}. Indeed linear models preserve convexity of the optimization problems where they are used. Moreover they can be evaluated, differentiated and also integrated very easily. 

Linear models are adapted to represent functions with unconstrained real-valued or vector-valued outputs. However, in some applications, it is crucial to learn functions with constrained outputs, such as functions which are non-negative or whose outputs are in a convex set, possibly with additional constraints like an integral equal to one, such as in density estimation, regression of multiple quantiles~\cite{bondell2010noncrossing}, and isotonic regression~\cite{barlow1972isotonic}. Note that the convex pointwise constraints on the outputs of the learned function must hold everywhere and not only on the training points. In this context, other models have been considered, such as generalized linear models \cite{mccullagh1989generalized}, at the expense of losing some important properties that hold for linear ones.


In this paper, we make the following contributions:
\begin{itemize}[leftmargin=.5cm]

    \item[-] We consider a class of models with non-negative outputs, as well as outputs in a chosen convex cone, which exhibit the same key properties of linear models. They can be used within empirical risk minimization with convex risks, preserving convexity. They are defined in terms of an arbitrary feature map and they can be evaluated, differentiated and integrated exactly.  
    
    \item[-] We derive a representer theorem for our models and provide a convex finite-dimensional dual formulation of the learning problem, depending only on the training examples. Interestingly, in the proposed formulation, the convex pointwise constraints on the outputs of the learned function are naturally converted to convex constraints on the coefficients of the model. 
    
    \item[-] We prove that the proposed model is a universal approximator and is strictly richer than commonly used generalized linear models. Moreover, we show that its Rademacher complexity is comparable with the one of linear models based on kernels.
  
    \item[-] To show the effectiveness of the method in terms of formulation, algorithmic derivation and practical results, we express naturally the problems of density estimation, regression with Gaussian heteroscedastic errors, and multiple quantile regression. We derive the corresponding learning algorithms for convex dual formulation, and compare it with standard techniques used for the specific problems on a few reference simulations.
  
\end{itemize}

\section{Background}

In a variety of fields ranging from {\em supervised learning}, to {\em Gaussian processes} \cite{williams2006gaussian}, {\em inverse problems}~\cite{engl1996regularization}, 
{\em scattered data approximation techniques} \cite{wendland2004scattered}, and {\em quadrature methods} to compute multivariate integrals~\cite{bach2017equivalence}, prototypical problems can be cast as
 \eqal{
 \label{eq:prototypical-problem}
 f^* \in \arg\min_{f \in {\cal  F} }\  L(f(x_1), \dots, f(x_n)) + \Omega(f).
 }
Here $L:\R^n \to \R$ is a (often convex) functional, ${\cal F}$ a class of real-valued functions, $x_1,\dots, x_n$ a given set of points in $\X$, and $\Omega$ a suitable regularizer \cite{scholkopf2002learning}.

 Linear models for the class of functions ${\cal F}$ are particularly  suitable to solve such problems.
 They are classically defined in terms of a feature map $\phi: \X \to \hh$ where $\X$ is the input space and $\hh$ is a separable Hilbert space. Typically, $\hh = \R^D$ with $D \in \N$, but $\hh$ can also be infinite-dimensional.  
A linear model is determined by a parameter vector $w \in \hh$ as
\eqal{
f_w(x) = \phi(x)^\top w,
}
leading to the space ${\cal F} = \{ f_w ~|~ w \in \hh\}$.
These models are particularly effective for problems in the form \cref{eq:prototypical-problem} because they satisfy the following key properties.

\paragraph{\bf P1. They preserve convexity of the loss function.} Indeed, given $x_1,\dots, x_n \in \X$, if
$L: \R^n \to \R$ is convex, then $ L(f_w(x_1),\dots,f_w(x_n))$ is convex in $w$.

\paragraph{\bf P2. They are universal approximators.} Under mild conditions on $\phi$ and $\hh$ (universality of the associated kernel function \cite{micchelli2006universal}) linear models can approximate any continuous function on $\X$. 
Moreover they can represent many classes of functions of interest, such as the class of polynomials, analytic functions, smooth functions on subsets of $\R^d$ or on manifolds, or Sobolev spaces  \cite{scholkopf2002learning}.

\paragraph {\bf P3. They admit a finite-dimensional representation.} Indeed, there is a so-called {\em representer theorem} \cite{cucker2002mathematical}.
Let $L$ be a possibly non-convex functional, ${\cal F} = \{ f_w ~|~ w \in \hh\}$, and assume $\Omega$ is an increasing function of $w^\top w$ (see \cite{scholkopf2001generalized} for more generality and details).
Then, the optimal solution $f^*$ of \eqref{eq:prototypical-problem} corresponds to $f^* = f_{w^*}$, with $w^* = \sum_{i=1}^n \alpha_i \phi(x_i)$, and $\alpha_1,\dots \alpha_n \in \R$. Denoting by $k$ the {\em kernel function} $k(x,x') := \phi(x)^\top \phi(x')$ for $x,x' \in \X$ (see, e.g., \cite{scholkopf2002learning}), $f^*$ can be rewritten as
\eqal{
\label{eq:linear-model-kernel-representation}
 f^*(x) = \sum_{i=1}^n \alpha_i k(x,x_i).
}
\paragraph{\bf P4. They are differentiable/integrable in closed form.} Assume that the kernel $k(x,x')$ is differentiable in the first variable. Then $\nabla_x f_{w^*}(x) = \sum_{i=1}^n \alpha_i \nabla_x k(x,x_i)$. Also the integral of $f_{w^*}$ can be computed in closed form if we know how to integrate $k$. Indeed, for $p :\X \to \R$ integrable, we have
$  \int f_{w^*}(x) p(x) dx = \sum_{i=1}^n \alpha_i \int k(x,x_i) p(x) dx.$

\paragraph{\bf Vector-valued models.} By juxtaposing scalar-valued linear models, we obtain a vector valued linear model, i.e. $f_{w_1\cdots w_p}:\X \to \R^p$ defined as $f_{w_1\cdots w_p}(x) = (f_{w_1}(x),\dots,f_{w_p}(x)) \in \R^p$.

\subsection{Models for non-negative functions or functions with constrained outputs}\label{sec:background-other-models}

While linear models provide a powerful formalization for functions from $\X$ to $\R$ or $\R^p$, in some important applications arising in the context of unsupervised learning, non-parametric Bayesian methods, or graphical models, additional conditions on the model are required. In particular, we will focus on {\em pointwise output constraints}. That is, given $\Y \subsetneq \R^p$, we want to obtain functions satisfying $f(x) \in \Y$ {\em for all} $x \in \X$. A prototypical example is the problem of density estimation.  
\begin{example}[density estimation problem]\label{ex:density-estimation}
The goal is to estimate the density of a probability $\rho$ on $\X$, given some i.i.d.~samples $x_1,\dots,x_n$. It can be formalized in terms of \cref{eq:prototypical-problem} (e.g., through maximum likelihood), with the constraint that $f$ is a density, i.e.,
$f(x) \geq 0, ~~ \forall x \in \X, \ \textrm{and} \  \int_{\X} f(x) dx = 1.$
\end{example}
Despite the similarity with \cref{eq:prototypical-problem}, linear models cannot be applied because of the constraint $f(x) \geq 0$.
Existing approaches to deal with the problem above are reported below, but lack some of the crucial properties {\bf P1-4} that make linear models so effective for problems of the form \cref{eq:prototypical-problem}.

\paragraph{\bf Generalized linear models (GLM).} Given a suitable map $\psi:\R^p \to \Y$, 
these models are of the form
$ f(x) = \psi(w^\top \phi(x))$.
In the case of non-negative functions, common choices are $\psi(z) = e^z$, leading to the {\em exponential family}, or the positive part function $\psi(z) = \max(0,z)$. GLM have an expressive power comparable to linear models, being able to represent a wide class of functions, and admit a finite-dimensional representation \cite{cheney2009course} (they thus satisfy {\bf P2} and {\bf P3}). However, in general they do not preserve convexity of the functionals where they are used  (except for specific cases, such as $L = -\sum_{i=1}^n \log z_i$ and $\psi(z) = e^z$ \cite{mccullagh1989generalized}).
Moreover they cannot be integrated in closed form, except for specific $\phi$, requiring some Monte Carlo approximations \cite{robert2013monte} (thus missing {\bf P1} and {\bf P4}).
An elegant way to obtain a GLM-like non-negative model is via {\em non-parametric mirror descent} \cite{yang2019learning} (see, e.g., their Example 4). A favorable feature of this approach is that the map $\psi$ is built implicitly according to the geometry of $\Y$. However, still the resulting model does not always satisfy  {\bf P3}, does not satisfy  {\bf P1}  and  {\bf P4} , and is only efficient in small-dimensional input spaces.

\paragraph{\bf Non-negative coefficients models (NCM).} Leveraging the finite-dimensional representation of linear models in \cref{eq:linear-model-kernel-representation}, the NCM models represent non-negative functions as
$f(x) = \sum_{i=1}^n \alpha_i k(x,x_i)$, with $\alpha_1,\dots \alpha_n \geq 0$,
given a kernel $k(x,x') \geq 0$ for any $x,x' \in \X$, such as the Gaussian kernel $e^{-\|x-x'\|^2}$ or the Abel kernel $e^{-\|x-x'\|}$. By construction these models satisfy {\bf P1}, {\bf P3}, {\bf P4}. However, they do not satisfy {\bf P2}. Indeed the fact that $\alpha_1,\dots,\alpha_n \geq 0$ does not allow cancellation effects and thus strongly constrains the set of functions that can be represented, as illustrated below.
\begin{example}\label{ex:NCM}
The NCM model cannot approximate arbitrarily well a function with a width strictly smaller that the width of the kernel. Take $k(x,x') = { e^{-\|x-x'\|^2}}$ and try to approximate the function ${ e^{-\|x\|^2/2}}$ on $[-1,1]$. Independently of the chosen $n$ or the chosen locations of the points $(x_i)_{i=1}^n$, it will not be possible to achieve an error smaller than a fixed constant (\cref{app:details-other-models} for a simulation).
\end{example} 
\paragraph{\bf Partially non-negative linear models (PNM).}
A partial solution  to have a linear model that is pointwise non-negative is to require non-negativity only on the observed points $(x_i)_{i=1}^n$. That is, the model is of the form $w^\top \phi(x)$, with $w \in \{ w \in \hh ~|~ w^\top \phi(x_1) \geq 0,\dots,w^\top \phi(x_n) \geq 0 \}$. While this model is easy to integrate in \cref{eq:prototypical-problem}, this does not guarantee the non-negativity outside of a neighborhood of $(x_i)_{i=1}^n$. It is possible to enrich this construction with a set of points that cover the whole space $\X$ (i.e., a fine grid, if $\X = [-1, 1]^d$), but this usually leads to exponential costs in the dimension of $\X$ and is not feasible when $d \geq 4$.
%

\section{Proposed Model for Non-negative Functions}







In this section we consider a non-parametric model for non-negative functions and we show that it enjoys the same benefits of linear models. In particular, we prove that it satisfies at the same time all the properties {\bf P1}, \dots, {\bf P4}. 
%
As linear models, the model we consider has a simple formulation in terms of a feature map $\phi: \X \to \hh$. 

Let ${\cal S}(\hh)$ be the set of bounded Hermitian linear operators from $\hh$ to $\hh$ (set of symmetric $D \times D$ matrices  if $\hh = \R^D$ with $D \in \N$) and denote by $A \succeq 0$ the fact that $A$ is a positive semi-definite operator (a positive semi-definite matrix, when $\hh$ is finite-dimensional) \cite{reed1980methods,horn2012matrix}. The model is defined for all $x \in \X$ as
\eqal{
\label{eq:model-non-negative}
f_A(x) = \phi(x)^\top A \phi(x), \qquad \textrm{where} \qquad A \in {\cal S}(\hh),~~A \succeq 0.
}
The proposed model
\footnote{Note that the model in \cref{eq:model-non-negative} has already been considered in \cite{bagnell2015learning} with a similar goal as ours. However, this workshop publication has only be lightly peer-reviewed, the representer theorem they propose is incorrect, the optimization algorithm is based on an incorrect representation and inefficient at best. See Appendix~\ref{app:kernelsos} for details.} 
is parametrized in terms of the operator (or matrix when $\hh$ is finite dimensional)~$A$, like in~\cite{blondel2015convex}, but with an additional positivity constraint. Note that, by construction, it is linear in $A$ and at the same time non-negative for any $x \in \X$, due to the positiveness of the operator~$A$, as reported below (the complete proof in \cref{app:proof-prop:model-non-negative-and-linear}).
\bp[Pointwise positivity  and linearity in the parameters]\label{prop:model-non-negative-and-linear}
Given $A, B \in {\cal S}(\hh)$ and $\alpha, \beta \in \R$, then $f_{\alpha A + \beta B}(x) = \alpha f_{A}(x) + \beta f_{B}(x)$. Moreover,  $
A \succeq 0  \ \Rightarrow \  f_A(x) \geq 0, ~ \forall x \in \X.$
\ep
An important consequence of linearity of $f_A$ in the parameter is that, despite the pointwise non-negativity in $x$,  it preserves {\bf P1}, i.e.,  the convexity of the functional where it is used. 
First define the set ${\cal S}(\hh)_+$ as ${\cal S}(\hh)_+ = \{A \in {\cal S}(\hh) ~|~ A \succeq 0\}$ and note that ${\cal S}(\hh)_+$ is convex \cite{boyd2004convex}. 
\bp[The model satisfies {\bf P1}]\label{prop:model-satisfies-P1}
Let $L:\R^n \to \R$ be a jointly convex function and $x_1,\dots, x_n \in \X$. 
Then the function $A \mapsto L(f_A(x_1),\dots, f_A(x_n))$ is convex on ${\cal S}(\hh)_+$.
\ep
\cref{prop:model-satisfies-P1} is proved in \cref{app:proof-prop:model-satisfies-P1}. The property above provides great freedom in choosing the functionals to be optimized with the proposed model. However, when $\hh$ has very high dimensionality or it is infinite-dimensional, the resulting optimization problem may be quite expensive. In the next subsection we provide a representer theorem and finite-dimensional representation for our model, that makes the optimization independent from the dimensionality of $\hh$.

\subsection{Finite-dimensional representations, representer theorem, dual formulation}

Here we will provide a finite-dimensional representation for the solutions of the following problem,
\eqal{
\label{eq:problem-for-representer}
 \inf_{A \succeq 0} L(f_A(x_1),\dots,f_A(x_n)) + \Omega(A),
}
given some points $x_1,\dots, x_n \in \hh$.
However, the existence and uniqueness of solutions for the problem above depend crucially on the choice of the regularizer $\Omega$ as it happens for linear models when~$\hh$ is finite-dimensional \cite{engl1996regularization}.
To derive a representer theorem for our model, we need to specify the class of regularizers we are considering. In the context of linear models a typical regularizer is Tikhonov regularization, i.e., $\Omega(w) = \la w^\top w$, for $w \in \hh$. Since the proposed model is expressed in terms of a symmetric operator (matrix, if $\hh$ is finite-dimensional), the equivalent of the Tikhonov regularizer is a functional that penalizes the squared Frobenius norm of $A$, i.e., $\Omega(A) = \la \tr(A^\top A)$, for $A \in {\cal S}(\hh)$ also written as $\Omega(A) = \la \|A\|^2_F$ \cite{engl1996regularization}. However, since $A$ is an operator, we can also consider different norms on its spectrum. From this viewpoint, an interesting regularizer corresponds to the {\em nuclear norm} $\|A\|_\star$, which induces sparsity on the spectrum of $A$, leading to low-rank solutions \cite{recht2010guaranteed,blondel2015convex}.
In this paper, for the sake of simplicity we will present the results for the following regularizer, which is the matrix/operator equivalent of the {\em elastic-net} regularizer \cite{zou2005regularization}:
\eqal{
\label{eq:base-regularizer}
\Omega(A) = \la_1 \|A\|_\star + \la_2 \|A\|^2_F, \quad \forall A \in {\cal S}(\hh),
}
 with $\la_1, \la_2 \geq 0$ and $\la_1+\la_2 > 0$. Note that $\Omega$ is strongly convex as soon as $\la_2 > 0$; we will therefore take $\lambda_2 > 0$ in practice in order to have easier optimization. 
 Recall the definition of the kernel $k(x,x') := \phi(x)^\top \phi(x')$, $x,x' \in \X$ \cite{scholkopf2002learning}.
 We have the following theorem.
\begin{restatable}[Representer theorem, {\bf P3}]{theorem}{thmnonneg}\label{thm:representer-non-negative}
Let $L:\R^n \to \R\cup \set{+\infty}$ be lower semi-continuous and bounded below, and $\Omega$ as in \cref{eq:base-regularizer}. Then \cref{eq:problem-for-representer} has a solution $A_*$ which can be written as
\eqal{
\label{eq:char-optimal-solution-representer}
 \sum_{i,j=1}^n \mathbf{B}_{ij} \phi(x_i)\phi(x_j)^\top, \qquad \textrm{for some matrix} ~ \mathbf{B} \in \R^{n\times n}, ~ \mathbf{B} \succeq 0.
}
$A_*$ is unique if $L$ is convex and $\la_2 > 0$. By \cref{eq:model-non-negative}, $A_*$ corresponds to a function of the form
$$f_*(x)= \sum_{i,j=1}^n \mathbf{B}_{ij} k(x,x_i)k(x,x_j),  \qquad \textrm{for some matrix} ~ \mathbf{B} \in \R^{n\times n}, ~ \mathbf{B} \succeq 0.$$

\end{restatable}
The proof of the theorem above is in \cref{app:proof-thm:representer-non-negative}, where it is derived for the more general class of spectral regularizers (this thus extends a result from~\cite{abernethy2009new}, from linear operators between potentially different spaces to positive self-adjoint operators). A direct consequence of \cref{thm:representer-non-negative} is the following finite-dimensional representation of the optimization problem in \cref{eq:problem-for-representer}. Denote by ${\bf K} \in \R^{n\times n}$ the matrix ${\bf K}_{i,j} = k(x_i, x_j)$ and assume w.l.o.g. that it is full rank (always true when the $n$ observations are distinct and $k$ is a {\em universal kernel} such as the Gaussian kernel \cite{micchelli2006universal}). Let ${\bf V}$ be the Cholesky decomposition of ${\bf K}$, i.e., ${\bf K} = {\bf V}^\top {\bf V}$. Define the finite dimensional model 
\eqal{
\label{eq:emp-feature-map-solution}
\tilde{f}_{\mathbf{A}}(x) = \Phi(x)^\top \mathbf{A} \Phi(x), \qquad {\mathbf A} \in \R^{n\times n},~ \mathbf{A} \succeq 0,
}
where $\Phi: \X \to \R^n$, defined as $\Phi(x) = \mathbf{V}^{-\top} v(x)$, with $ v(x) = \left(k(x,x_i)\right)_{i=1}^n \in \R^n$, is the classical \emph{empirical feature map}. In particular, $\tilde{f}_{\mathbf{A}} = f_A$ where $A$ is of the form \cref{eq:char-optimal-solution-representer} with $ \mathbf{B} = {\bf V}^{-1} \mathbf{A}{\bf V}^{-\top}$. We will say that $\tilde{f}_{\mathbf{A}}$ is a solution of \cref{eq:problem-for-representer} if the corresponding $A$ is a solution of \cref{eq:problem-for-representer}.

\bp[Equivalent finite-dimensional formulation in the primal]\label{prop:finite-dimensional-primal-characterization}
Under the assumptions of \cref{thm:representer-non-negative}, the following problem has at least one solution, which is unique if $\lambda_2 > 0$ and $L$ is convex :
\eqal{
\label{eq:finite-dimensional-primal-characterization}
 \min_{\mathbf{A} \succeq 0} L(\tilde{f}_\mathbf{A}(x_1), \dots, \tilde{f}_\mathbf{A}(x_n)) + \Omega(\mathbf{A}).
}
Moreover, for any given solution  $\mathbf{A}^* \in \R^{n\times n}$ of \cref{eq:finite-dimensional-primal-characterization}, the function $\tilde{f}_{\mathbf{A}^*}$ is a minimizer of \cref{eq:problem-for-representer}. 
Finally, note that problems \cref{eq:problem-for-representer} and \cref{eq:finite-dimensional-primal-characterization} have the same condition number if it is exists.
\ep
The proposition above (proof in \cref{app:proof-prop:finite-dimensional-primal-characterization})  characterizes the possibly infinite-dimensional optimization problem of \cref{eq:problem-for-representer} in terms of an optimization on $n\times n$ matrices. A crucial property is that the formulation in \cref{eq:finite-dimensional-primal-characterization} {\em preserves convexity}, i.e., it is convex as soon as $L$ is convex. To conclude,  \cref{app:proof-prop:finite-dimensional-primal-characterization} provides a construction for ${\bf V}$ valid for possibly rank-deficient ${\mathbf K}$.
We now provide a finer characterization in terms of a dual formulation on only $n$ variables.

\paragraph{Convex dual formulation.}
We have seen above that the problem in \cref{eq:problem-for-representer} admits a finite-dimensional representation and can be cast in terms of an equivalent problem on $n\times n$ matrices. Here, when $L$ is convex, we refine the analysis and provide a dual optimization problem on only $n$ variables. The dual formulation is particularly suitable when $L$ is a sum of functions as we will see later. 
%
%
In the following theorem $[{\bf A}]_{-}$ corresponds to the negative part\footnote{Given the eigendecomposition ${\bf A} = {\bf U} {\bf \Lambda}{\bf U}^{\top}$ with ${\bf U} \in \R^{n\times n}$ unitary and ${\bf \Lambda} \in \R^{n\times n}$ diagonal, then $[{\bf A}]_{-} = {\bf U} {\bf \Lambda}_{-}{\bf U}^{\top}$, with ${\bf \Lambda}_{-}$ diagonal, defined as $({\bf \Lambda}_{-})_{i,i} = \min(0,{\bf \Lambda}_{i,i})$ for $i=1,\dots,n$.} of ${\bf A} \in {\cal S}(\R^n)$.

\bt[Convex dual problem]\label{thm:dual}
Assume $L$ is convex, lower semi-continuous and bounded below. Assume $\Omega$ is of the form \cref{eq:base-regularizer} with $\lambda_2 > 0$. Assume that the problem has at least a strictly feasible point, i.e.,  there exists $A_0 \succeq 0$ such that $L$ is continuous in $ (f_{A_0}(x_1),...,f_{A_0}(x_n))\in \R^n$ (this condition is satisfied in simple cases; see examples in \cref{app:proof-thm:dual}).
Denoting with $L^*$ the Fenchel conjugate of $L$ (see \cite{boyd2004convex}), problem \cref{eq:finite-dimensional-primal-characterization} has the following dual formulation:
\begin{align}
\label{dualfinited}
\sup_{\alpha \in \R^n} -L^*(\alpha) - \tfrac{1}{2\la_2}\|[{\bf V}\diag(\alpha){\bf V}^\top+\la_1 {\bf I}]_{-}\|^2_{F},
\end{align}
and this supremum is atteined. Moreover, if $\alpha^* \in \R^n$ is a solution of \eqref{dualfinited}, a solution of  \eqref{eq:problem-for-representer} is obtained via \eqref{eq:char-optimal-solution-representer}, with ${\mathbf B} \in \R^{n\times n}$ defined as 
\begin{align}
    \label{expression_from_dual}
    {\mathbf B} =  \la_2^{-1} {\bf V}^{-1}[{\bf V}\diag(\alpha^*){\bf V}^\top+\la_1 {\bf I}]_{-} {\bf V}^{-\top}.
\end{align} 
\et
The result above (proof in \cref{app:proof-thm:dual}) is particularly interesting when $L$ can be written in terms of a sum of functions, i.e., $L(z_1,\dots, z_n) = \sum_{i=1}^n \ell_i(z_i)$ for some functions $\ell_i :\R \to \R$. Then the Fenchel dual is $L^*(\alpha) = \sum_{i=1}^n \ell^*_i(\alpha_i)$, where $\ell^*_i$ is the Fenchel dual of $\ell_i$, and the optimization can be carried by using accelerated proximal splitting methods as FISTA \cite{beck2009fast}, since $\|[{\bf V}\diag(\alpha){\bf V}^\top+\la_1 {\bf I}]_{-}\|^2_{F}$ is differentiable in $\alpha$. This corresponds to a complexity of $O(n^3)$ per iteration for FISTA, due to the computation of \cref{expression_from_dual}, and can be made comparable with fast algorithms for linear models based on kernels \cite{rudi2017falkon}, by using techniques from randomized linear algebra and Nystr\"om approximation \cite{halko2011finding} (see more details in \cref{app:proof-thm:dual}).

\section{Approximation Properties of the Model}

%
%

The goal of this section is to study the approximation properties of our model and to  understand its ``richness'', i.e., which functions it can represent. In particular, we will prove that, under mild assumptions on $\phi$, (a) the proposed model satisfies the property {\bf P2}, i.e., it is a {\em universal approximator} for non-negative functions, and (b) that it is strictly richer than the family of exponential models with the same $\phi$.
%
%
First, define the set of functions belonging to our model
\[{\cal F}^\circ_{\phi} = \{ f_A ~~|~~ A \in {\cal S}(\hh), ~ A \succeq 0, \|A\|_\circ < \infty\},\]
where $\|\cdot\|_\circ$ is a suitable norm for ${\cal S}(\hh)$. In particular, norms that we have seen to be relevant in the context of optimization are the nuclear norm $\|\cdot\|_\star$ and the Hilbert-Schmidt (Frobenius) norm $\|\cdot\|_{F}$. 
Given norms $\|\cdot\|_a,\|\cdot\|_b$, we denote the fact that $\|\cdot\|_a $ is stronger (or equivalent) than $\|\cdot\|_b$ with $\|\cdot\|_a \trianglerighteq \|\cdot\|_b$ (for example, $\|\cdot\|_\star \trianglerighteq \|\cdot\|_{F}$).
%
In the next theorem we prove that when the feature map is universal~\cite{micchelli2006universal}, such as the one associated to the Gaussian kernel $k(x,x') = \exp(-\|x-x'\|^2)$ or the Abel kernel $k(x,x') = \exp(-\|x-x'\|)$, then the proposed model is a universal approximator for non-negative functions over $\X$ (in particular, in the sense of {\em cc-universality}
\cite{micchelli2006universal,sriperumbudur2011universality}, see \cref{app:proof-thm:universality} for more details and the proof).  
%
%
\bt[Universality, {\bf P2}]\label{thm:universality}
Let $\hh$ be a separable Hilbert space, $\phi: \X \to \hh$ a universal map~\cite{micchelli2006universal}, and $\|\cdot\|_{\star} \trianglerighteq \|\cdot\|_\circ$. Then ${\cal F}^\circ_{\phi}$ is a universal approximator of non-negative functions over $\X$.
\et

The fact that the proposed model can approximate arbitrarily well any non-negative function on $\X$, when $\phi$ is universal, makes it a suitable candidate in the context of nonparametric approximation/interpolation or learning \cite{wendland2004scattered,tsybakov2008introduction} of non-negative functions.
In the following theorem, we give a more precise characterization of the functions contained in ${\cal F}^\circ_{\phi}$. Denote by ${\cal G}_\phi$ the set of linear models induced by $\phi$, i.e., ${\cal G}_\phi = \{w^\top\phi(\cdot) ~|~ w \in \hh\}$ and by ${\cal E}_\phi$ the set of {\em exponential models} induced by $\phi$,
\[{\cal E}_\phi ~~=~~ \{ ~e^f ~~|~~ f(\cdot) = w^\top\phi(\cdot),~~ w \in \hh~\}.\]

\bt[${\cal F}^\circ_{\phi}$ strictly richer than the exponential model]
\label{thm:model-richer-than-exponential}
Let $\|\cdot\|_{\star} \trianglerighteq \|\cdot\|_\circ$. Let $\X = [-R,R]^d$, with $R > 0$. Let $\phi$ such that $W^{m}_2(\X) = {\cal G}_\phi$, for some $m > 0$, where $W^m_2(\X)$ is the {\em Sobolev space} of smoothness $m$ \cite{adams2003sobolev}. Let $x_0 \in \X$. The following hold: 
\begin{enumerate}
\item[(a)]  ${\cal E}_\phi \subsetneq {\cal F}^\circ_{\phi};$
\item[(b)] the function $f_{x_0}(x) = e^{-\|x- x_0\|^{-2}} \in C^\infty(\X)$ satisfies $f_{x_0} \in {\cal F}^\circ_{\phi}$ and $f_{x_0} \notin {\cal E}_\phi$.
 \end{enumerate}
\et
\cref{thm:model-richer-than-exponential} shows that if $\phi$ is rich enough, then the space of exponential models is strictly contained in the space of functions associated to the proposed model. In particular, the proposed model can represent functions that are exactly zero on some subset of $\X$ as showed by the example $f_{x_0}$ in \cref{thm:model-richer-than-exponential}, while the exponential model can represent only strictly positive functions, by construction. 
Discussion on the condition $W^{m}_2(\X) = {\cal G}_\phi$, proof of \cref{thm:model-richer-than-exponential} and its generalization to $\X \subseteq \R^d$ are in App.~\ref{app:proof-thm:model-richer-than-exponential}. Here we note only that the condition  $W^{m}_2(\X) = {\cal G}_\phi$ is quite mild and satisfied by many kernels such as the Abel kernel $k(x,x') = \exp(-\|x-x'\|)$ \cite{wendland2004scattered,berlinet2011reproducing}. 
We conclude with a bound on the {\em Rademacher complexity} \cite{boucheron2005theory} of ${\cal F}^\circ_\phi$, which is a classical component for deriving generalization bounds \cite{shalev2014understanding}. Define ${\cal F}^\circ_{\phi,L} = \{f_A ~|~ A\succeq 0, \|A\|_\circ \leq L\}$, for $L > 0$. \cref{thm:rademacher} shows that the Rademacher complexity of ${\cal F}^\circ_{\phi,L}$ depends on $L$ and not on the dimensionality of $\X$, as for regular kernel methods \cite{boucheron2005theory}.

\bt[Rademacher complexity of ${\cal F}^\circ_\phi$]\label{thm:rademacher}\label{thm:rademacher-correct}
 Let $\|\cdot\|_\circ \trianglerighteq \|\cdot\|_{F}$ and $\sup_{x \in \X} \|\phi(x)\| \leq c < \infty$. Let $(x_i)_{i=1}^n$ be i.i.d.~samples, $L \geq 0$. 
 The Rademacher complexity of ${\cal F}^\circ_{\phi,L}$ on $(x_i)_{i=1}^n$ is upper bounded by $\frac{2Lc^2}{\sqrt{n}}$ (proof in \cref{app:proof-thm:rademacher}).
\et

\section{Extensions: Integral Constraints and Output in Convex Cones}
In this section we cover two extensions. The first one generalizes the optimization problem in \cref{eq:problem-for-representer} to include linear constraints on the integral of the model, in order to deal with problems like density estimation in \cref{ex:density-estimation}. The second formalizes models with outputs in convex cones, which is crucial when dealing with problems like multivariate quantile estimation \cite{bondell2010noncrossing}, detailed in \cref{sec:experiments}. 

\paragraph{\bf Constraints on the integral and other linear constraints.} We can extend the definition of the problem in \cref{eq:problem-for-representer} to take into account constraints on the integral of the model. Indeed by linearity of integration and trace, we have the following (proof in \cref{app:proof-prop:model-satisfies-P4}). 
\bp[Integrability in closed form, {\bf P4}]\label{prop:model-satisfies-P4}
Let $A \in {\cal S}(\hh)$ with $A$ bounded and $\phi$ uniformly bounded. Let $p:\X \to \R$  be an integrable function. There exists a trace class operator $W_p \in {\cal S}(\hh)$ such that
$\int_\X f_A(x) p(x) dx ~=~ \tr( A  W_p )$ and $ W_p ~=~ \int_\X \phi(x) \phi(x)^\top \  p(x) dx.$
\ep
The result can be extended to derivatives and more general linear functionals on $f_A$ (see \cref{app:proof-prop:model-satisfies-P4}). In particular, note that if we consider the {\em empirical feature map} $\Phi$ in \cref{eq:emp-feature-map-solution}, which characterizes the optimal solution of \cref{eq:problem-for-representer}, by \cref{thm:representer-non-negative}, we have that $W_p$ is defined explicitly as $W_p = {\bf V}^{-\top} {\bf M}_p {\bf V}^{-1}$ with $({\bf M}_p)_{i,j} = \int k(x,x_i) k(x,x_j) p(x) dx$, for $i,j=1,\dots,n$ and it is computable in closed form. Then, assuming an equality and an inequality constraint on the integral w.r.t. two functions $p$ and $q$ and two values $c_1, c_2 \in \R$, 
the resulting problem takes the following finite-dimensional form
\eqal{
\label{eq:prototypical-problem-integral-constraints}
\min_{{\bf A} \in {\cal S}(\R^n)} & ~~~L(\tilde{f}_\mathbf{A}(x_1), \dots, \tilde{f}_\mathbf{A}(x_n)) ~+~ \Omega(\mathbf{A}), \\
\text{s.t.} & ~~~ {\bf A} \succeq 0, ~ \tr({\bf A} W_p) = c_1, ~ \tr({\bf A} W_q) \leq c_2. \nonumber 
}
\paragraph{\bf Representing function with outputs in convex polyhedral cones.} 
We represent a vector-valued function with our model as the juxtaposition of $p$ scalar valued models, with $p \in \N$, as follows
\[f_{A_1 \cdots A_p}(x) = (f_{A_1}(x),\dots,f_{A_p}(x)) \in \R^p , \qquad \forall ~ x \in \X.\]
We recall that a convex polyhedral cone $\Y$ is defined by a set of inequalities as follows
\eqal{
\label{eq:convex-polyhedral-Y}
\Y = \{ y \in \R^p ~|~ {c^1} {}^\top y \geq 0, \dots, {c^h} {}^\top y \geq 0 \},
}
for some $c^1,\dots,c^h \in \R^p$ and $h \in \N$. Let us now focus on a single constraint $c^\top y \geq 0$. Note that, by definition of positive operator (i.e., $A \succeq 0$ implies $v^\top A v \geq 0$ for any $A$), we have that
$\sum_{s=1}^p c_s A_s \succeq 0$ implies $\phi(x)^\top(\sum_{s=1}^p c_s A_s) \phi(x) \geq 0$ for any $x \in \X$, which, by linearity of the inner product and the definition of $f_{A_1 \cdots A_p}$ is equivalent to $c^\top f_{A_1 \cdots A_p}(x) \geq 0$. From this reasoning we derive the following proposition (see complete proof in \cref{app:proof-prop:convex-cone}).
\bp\label{prop:convex-cone}
Let $\Y$ be defined as in \cref{eq:convex-polyhedral-Y}. Let $A_1,\dots, A_p \in {\cal S}(\hh)$. Then the following holds
\eqals{
{ \sum_{s=1}^p} c^t_s A_s \succeq 0 \quad \forall t=1,\dots,h \qquad \Rightarrow \qquad f_{A_1 \cdots A_p}(x) \in \Y \quad  \forall x \in \X.
}
\ep
Note that the set of constraints on the l.h.s.~of the equation above defines in turn a convex set on $A_1,\dots,A_p$. This means that we can use it to constrain a convex optimization problem over the space of the proposed vector-valued models as follows
\eqal{
\label{eq:prototypical-problem-convex-cones}
\underset{A_1,\dots, A_p \in {\cal S}(\hh)}{\min} & ~~~ L(f_{A_1 \cdots A_p}(x_1),\dots, f_{A_1 \cdots A_p}(x_n)) ~+~ { \sum_{s=1}^p}\Omega(A_s)  
\\
\text{s.t.} & ~~~{ \sum_{s=1}^p} c^t_s A_s \succeq 0, 
\quad \forall ~ t = 1,\dots,h. \nonumber
}
By \cref{prop:convex-cone}, the function $f_{A^*_1 \cdots A^*_p}$, where $(A^*_1,\dots, A^*_p)$ is the minimizer above, will be a function with output in $\Y$.
Moreover, the formulation above admits a finite-dimensional representation analogous to the one for non-negative functions, as stated below (see proof in \cref{app:proof-thm:representer-convex-cone})
\bt[Representer theorem for model with output in convex polyhedral cones]
\label{thm:representer-convex-cone}
Under the assumptions of \cref{thm:representer-non-negative}, the problem in \cref{eq:prototypical-problem-convex-cones} admits a minimizer $(A^*_1,\cdots,A^*_p)$ of the form
\[A^*_s = \sum_{i,j=1}^n [\mathbf{B}_{s}]_{i,j} \phi(x_i)\phi(x_j)^\top \implies (f_{*}(x))_s = \sum_{i,j=1}^n [\mathbf{B}_{s}]_{i,j} k(x_i,x)k(x_j,x),\quad s = 1,...,p,\]
where $f_* := f_{(A_1^*,...,A_p^*)}$ is the corresponding function and the ${\bf B}_s \in {\cal S}(\R^n)$ are symmetric $n \times n$ matrices which satisfy the conic constraints $\sum_{s=1}^p{c^t_s {\bf B}_s} \succeq 0,~t = 1,...,h$.
\et
\br[Efficient representations when the ambient space of $\Y$ is high-dimensional]
When $p \gg h$, or when $\Y$ is a polyhedral cone with $\Y \subset {\cal G}$ and ${\cal G}$ an infinite-dimensional space, it is still possible to have an efficient representation of functions with output in $\Y$ by using the representation of $\Y$ in terms of {\em conical hull} \cite{boyd2004convex}, i.e., $\Y = \{\sum_{i=1}^t \alpha_i y_i ~|~ \alpha_i \geq 0\}$ for some $y_1,\dots, y_t$ and $t \in \N$.
In particular, given $A_1,\dots,A_t \succeq 0$, the model 
$f_{A_1\dots A_t}(x) = { \sum_{i=1}^t} f_{A_i}(x) y_i$
satisfies $f_{A_1\dots A_t}(x) \in \Y$ for any $x \in \X$. Moreover it is possible to derive a representer theorem as \cref{thm:representer-convex-cone}.
\er
\begin{remark} By extending this approach, we believe it is possible to model (a) functions with output in the cone of positive semidefinite matrices, (b) convex functions. We leave this for future work.
\end{remark}

\section{Numerical Simulations}\label{sec:experiments}


In this section, we provide illustrative experiments on the problems of density estimation, regression with Gaussian heteroscedastic errors, and multiple quantile regression. We derive the algorithm according to the finite-dimensional formulation in \cref{eq:prototypical-problem-integral-constraints} for non-negative functions with constraints on the integral, and to \cref{eq:prototypical-problem-convex-cones} with the finite-dimensional representation suggested by \cref{thm:representer-convex-cone}. Optimization is performed applying FISTA \cite{beck2009fast} on the dual of the resulting formulations. More details on implementation and the specific formulations are given below and in \cref{app:details-experiments}. 
The algorithms are compared with careful implementations of \cref{eq:prototypical-problem} with the models presented in \cref{sec:background-other-models}, i.e., partially non-negative models (PNM), non-negative coefficients models (NCM) and generalized linear models (GLM). For all methods we used $\Omega(A) = \la \left(\|A\|_*+\frac{0.01}{2}\|A\|_F^2\right)$ or $\Omega(w) = \la \|w\|^2$. We used the Gaussian kernel $k(x,x') = \exp(-\|x-x'\|^2/(2\sigma^2))$ with width $\sigma$. Full cross-validation has been applied to each model independently, to find the best $\la$ (see \cref{app:details-experiments}). 

\begin{figure}[t!]
    \centering
    {\tiny ${}~~~\quad{}$ PNM ${}~\qquad\qquad\qquad\qquad\qquad\qquad\qquad{}$ NCM ${}\qquad\qquad\qquad\qquad\qquad\qquad\quad~~~{}$ GLM ${}\qquad\qquad\qquad\qquad\qquad\qquad\quad~~~{}$ Our Model}\\
    \hspace{-0.4cm}\vspace{-0.1cm}
    \includegraphics[width=0.248\textwidth]{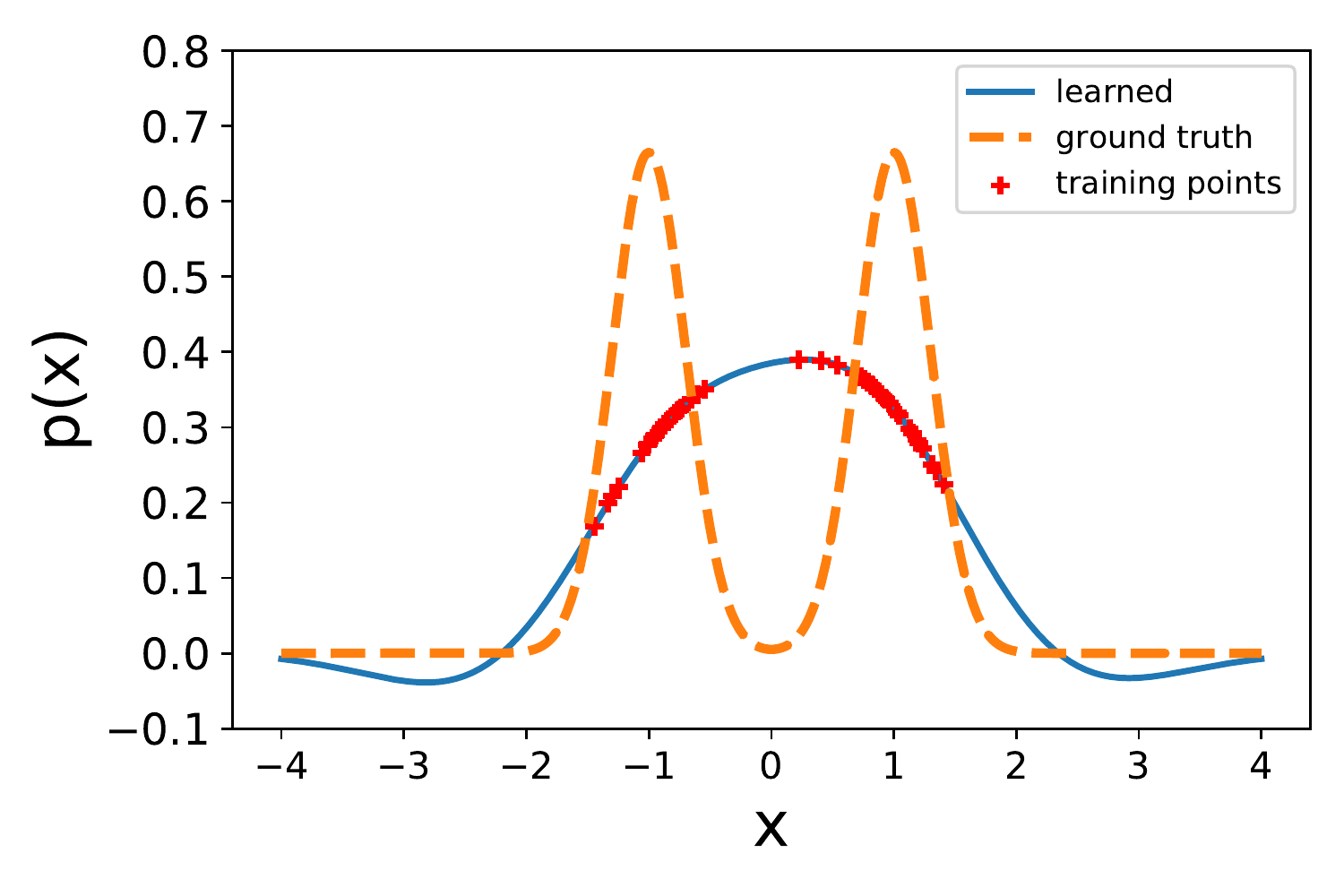}~
    \includegraphics[width=0.248\textwidth]{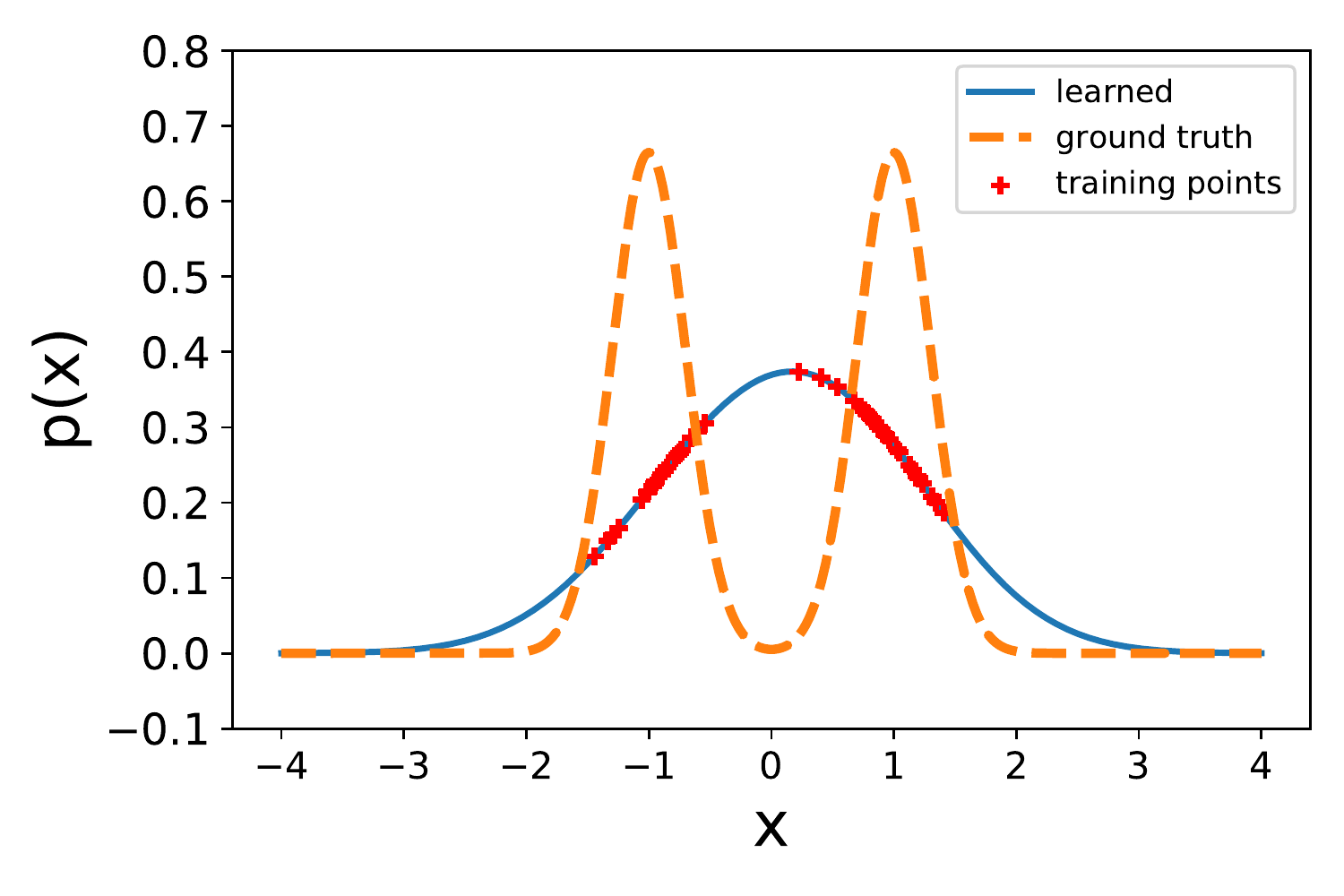}~
    \includegraphics[width=0.248\textwidth]{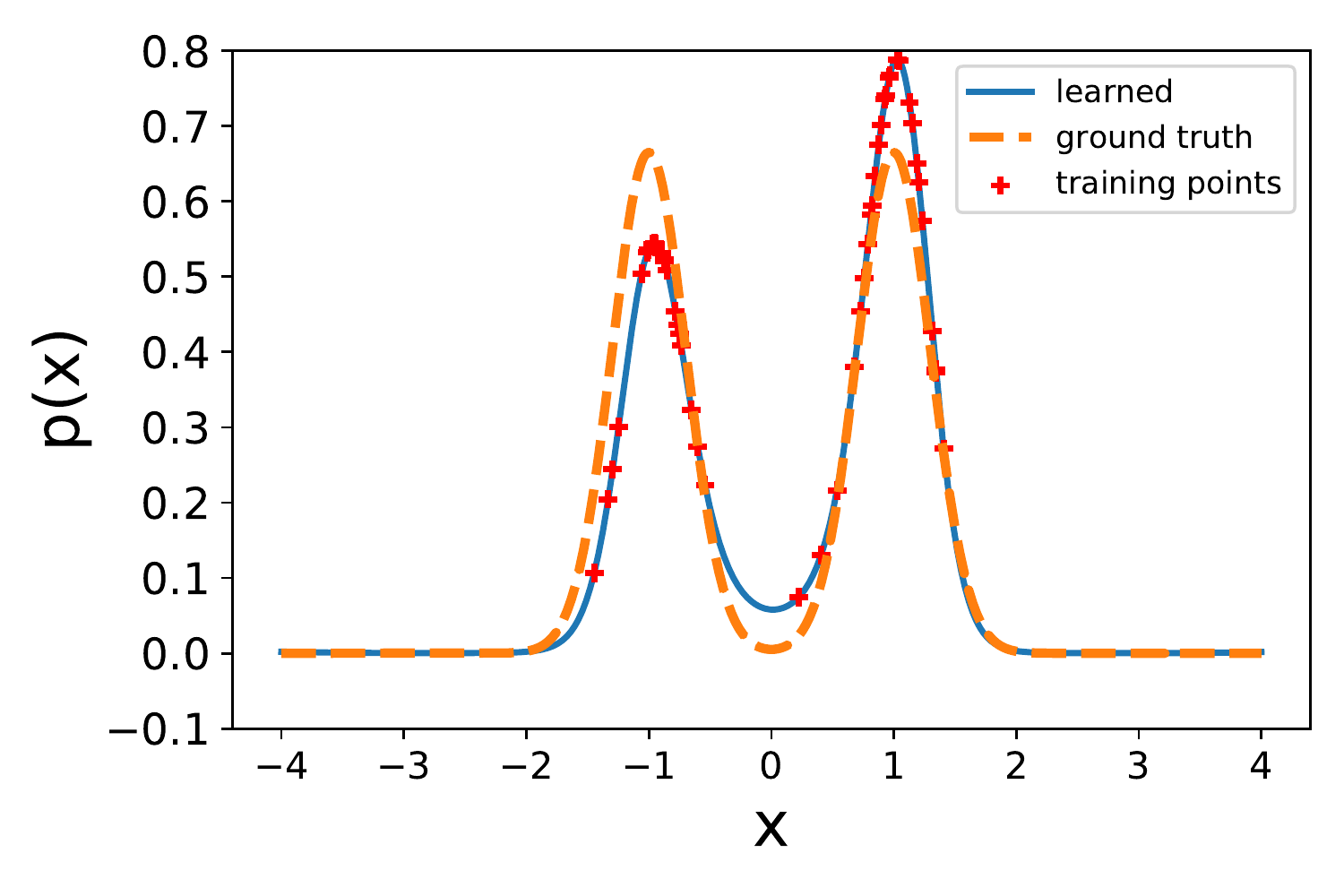}~
    \includegraphics[width=0.248\textwidth]{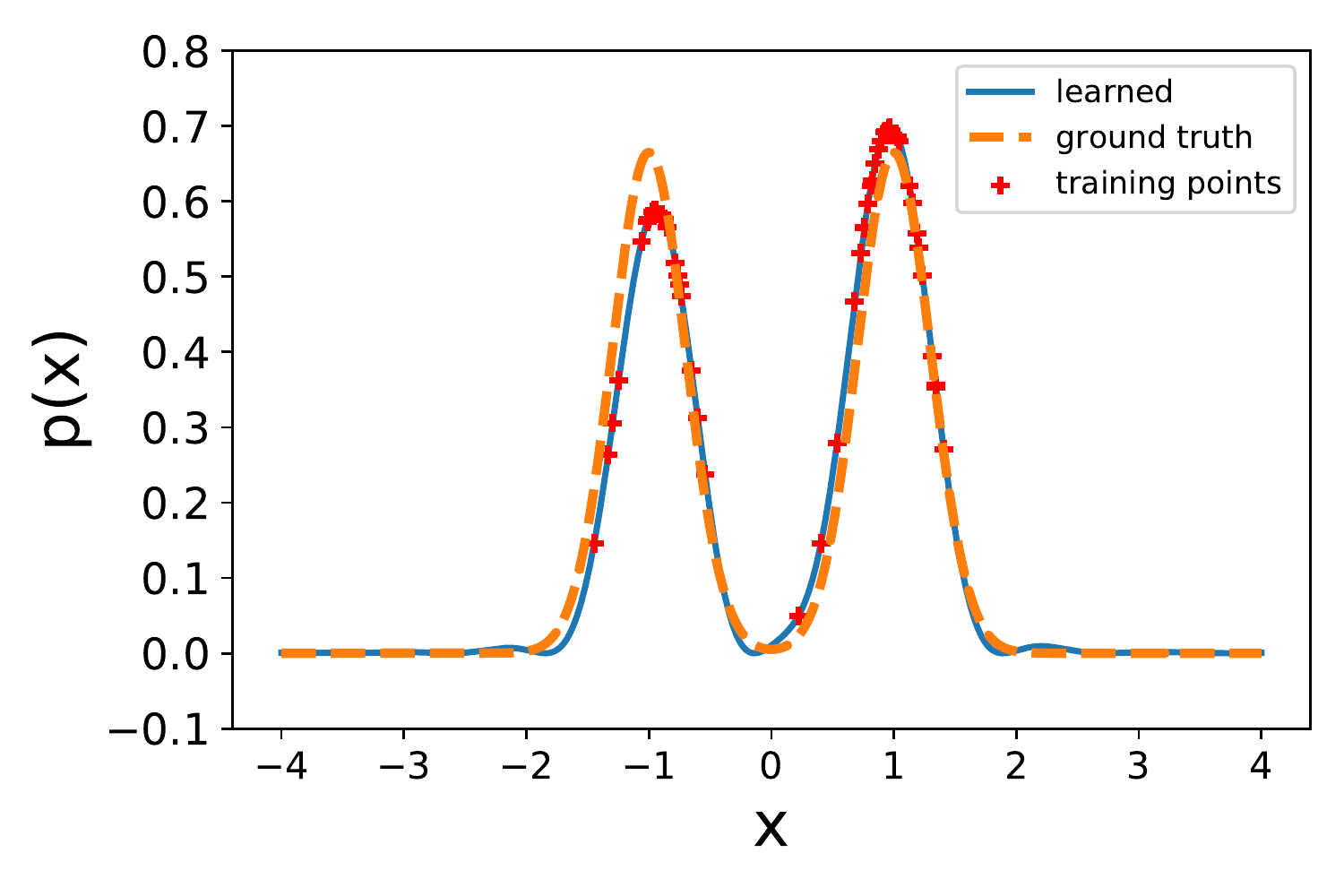}\\
    \hspace{-0.4cm}\vspace{-0.1cm}
    \includegraphics[width=0.248\textwidth]{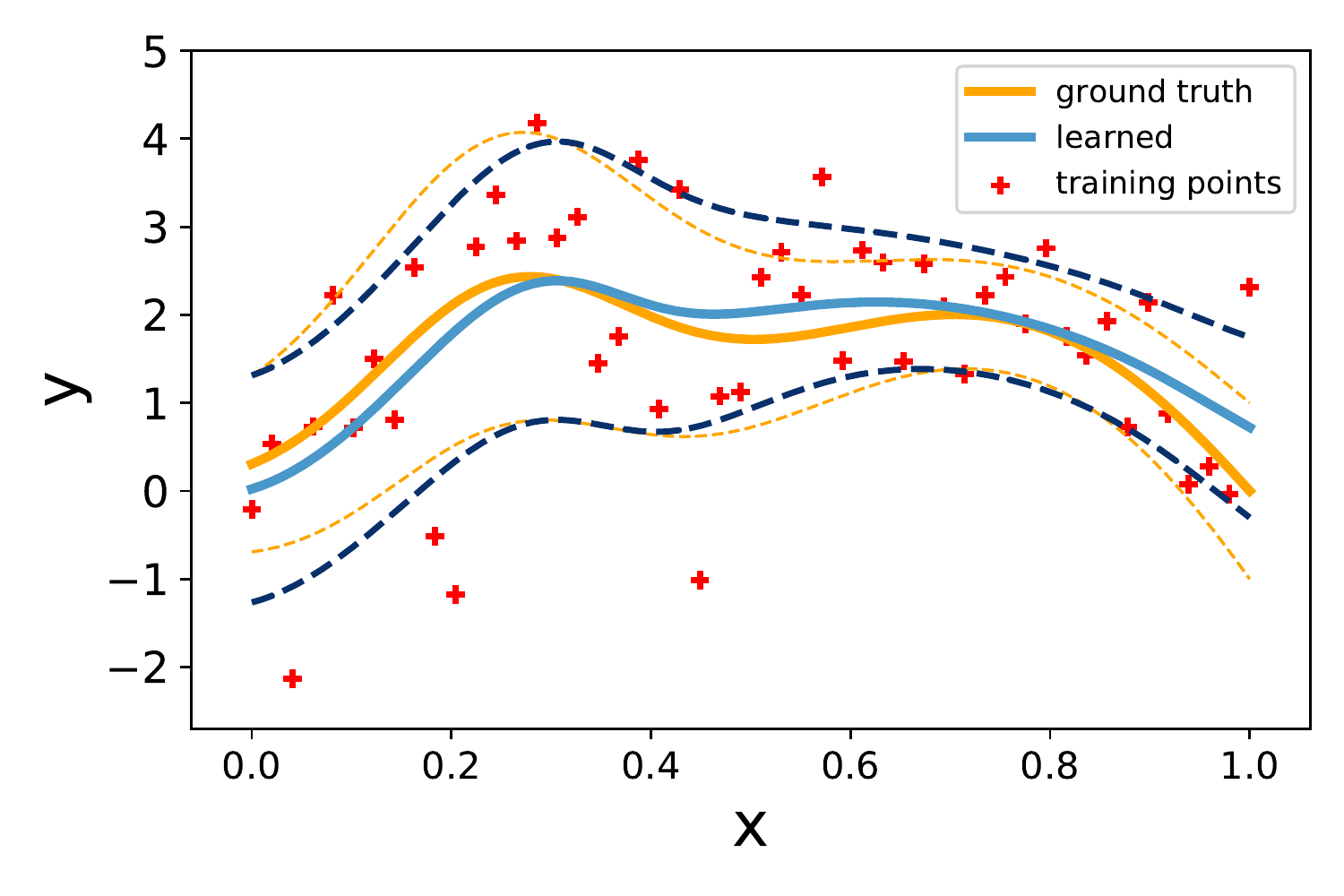}~
    \includegraphics[width=0.248\textwidth]{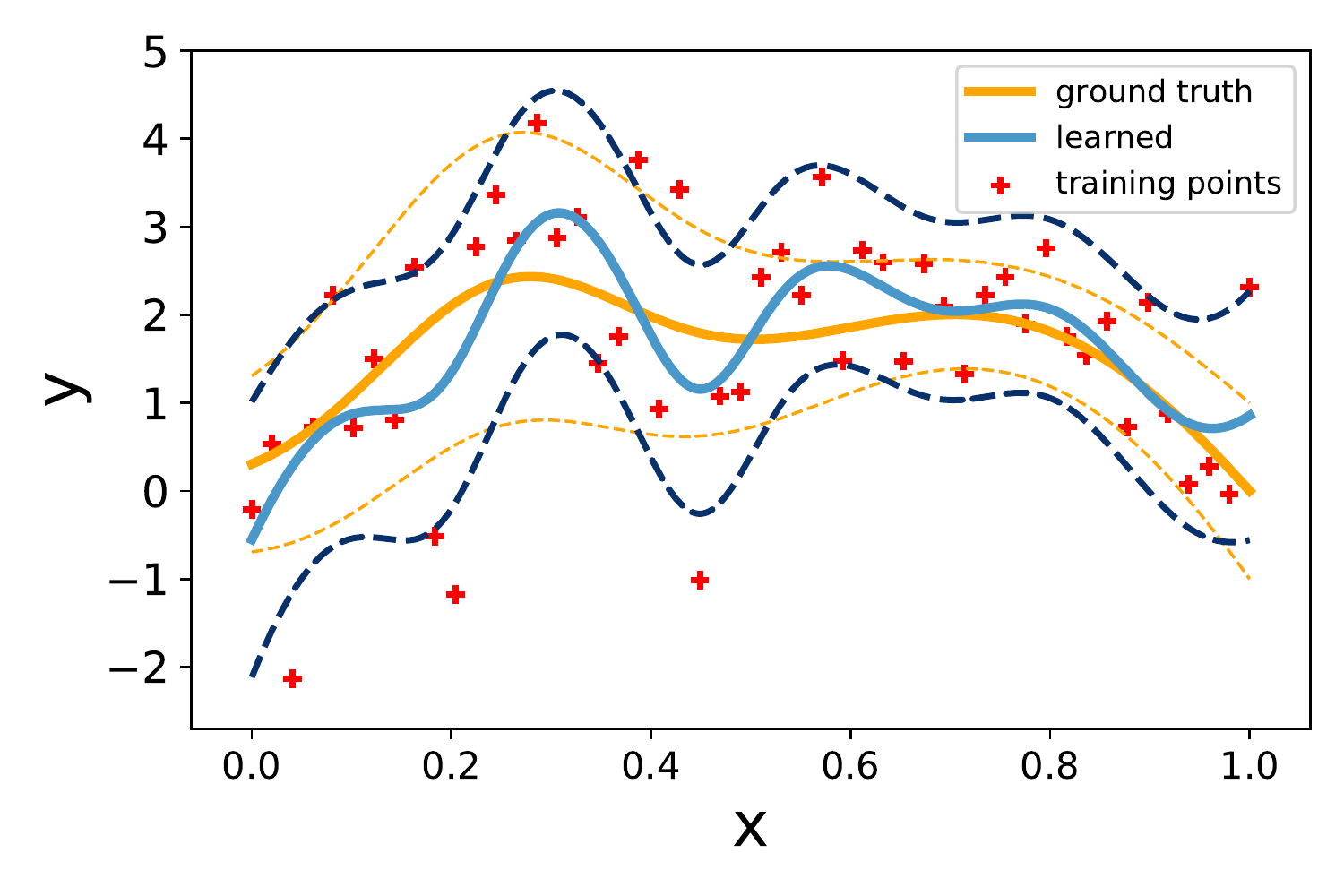}~
    \includegraphics[width=0.248\textwidth]{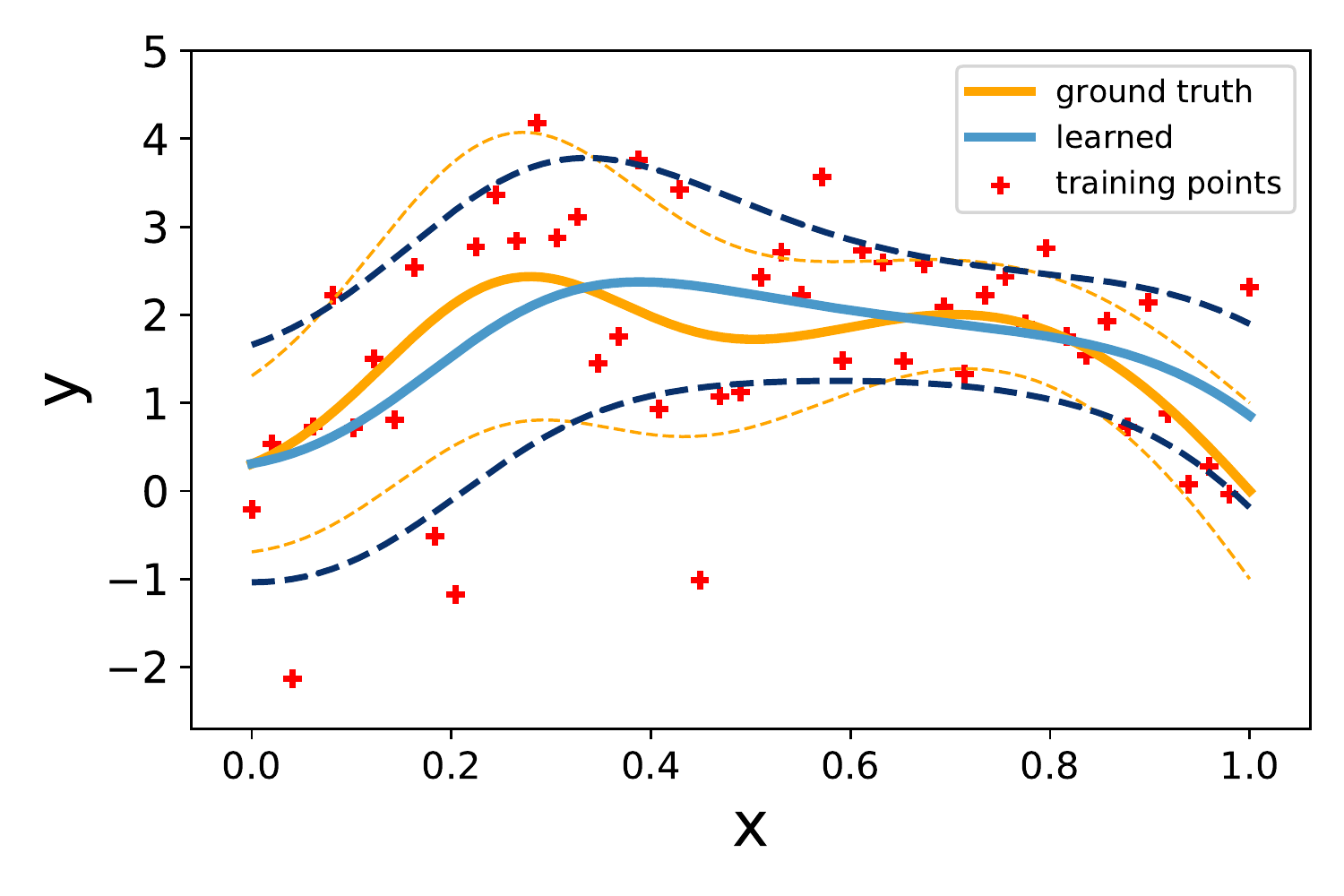}~
    \includegraphics[width=0.248\textwidth]{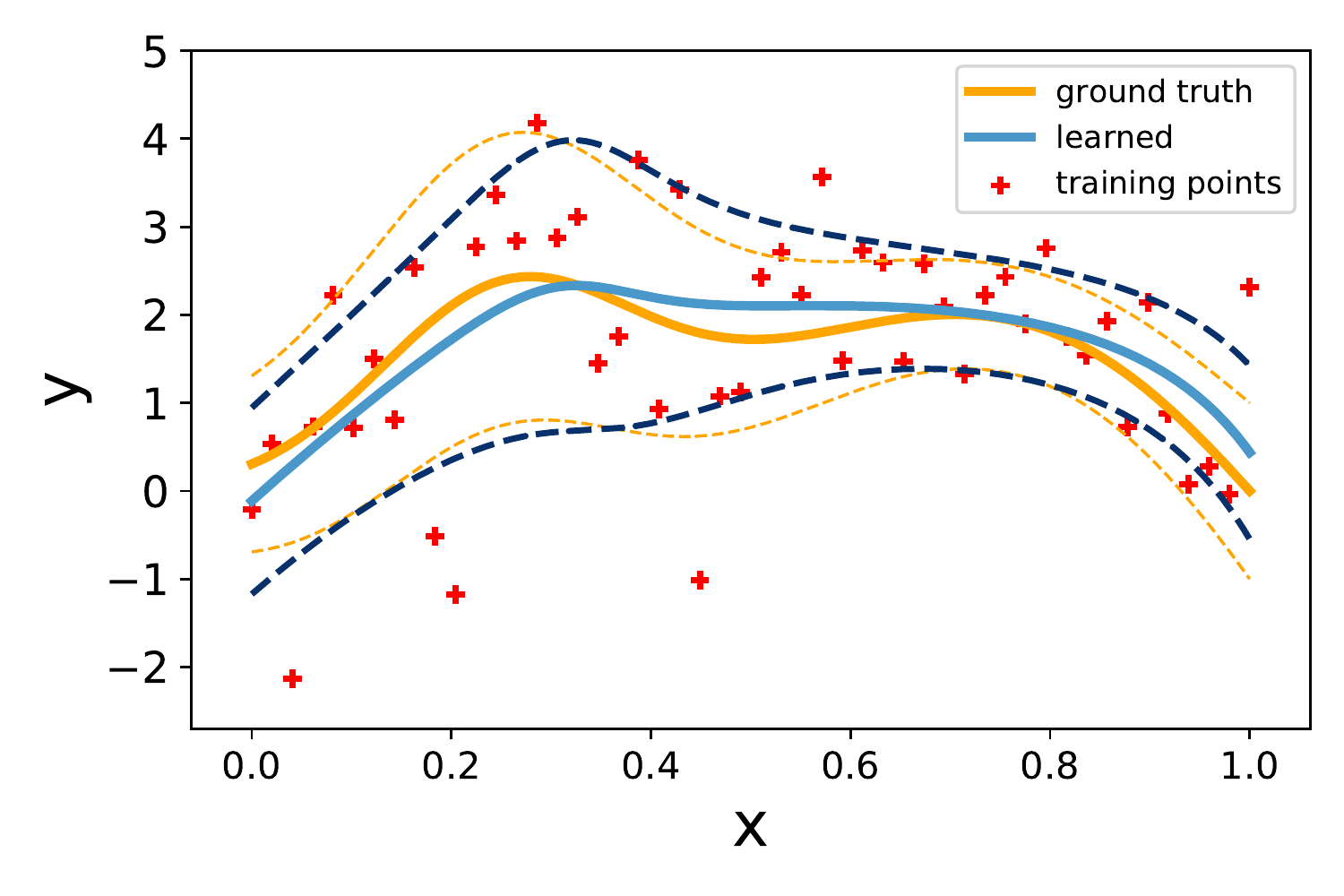}\\
    \hspace{-0.4cm}
    \includegraphics[width=0.248\textwidth]{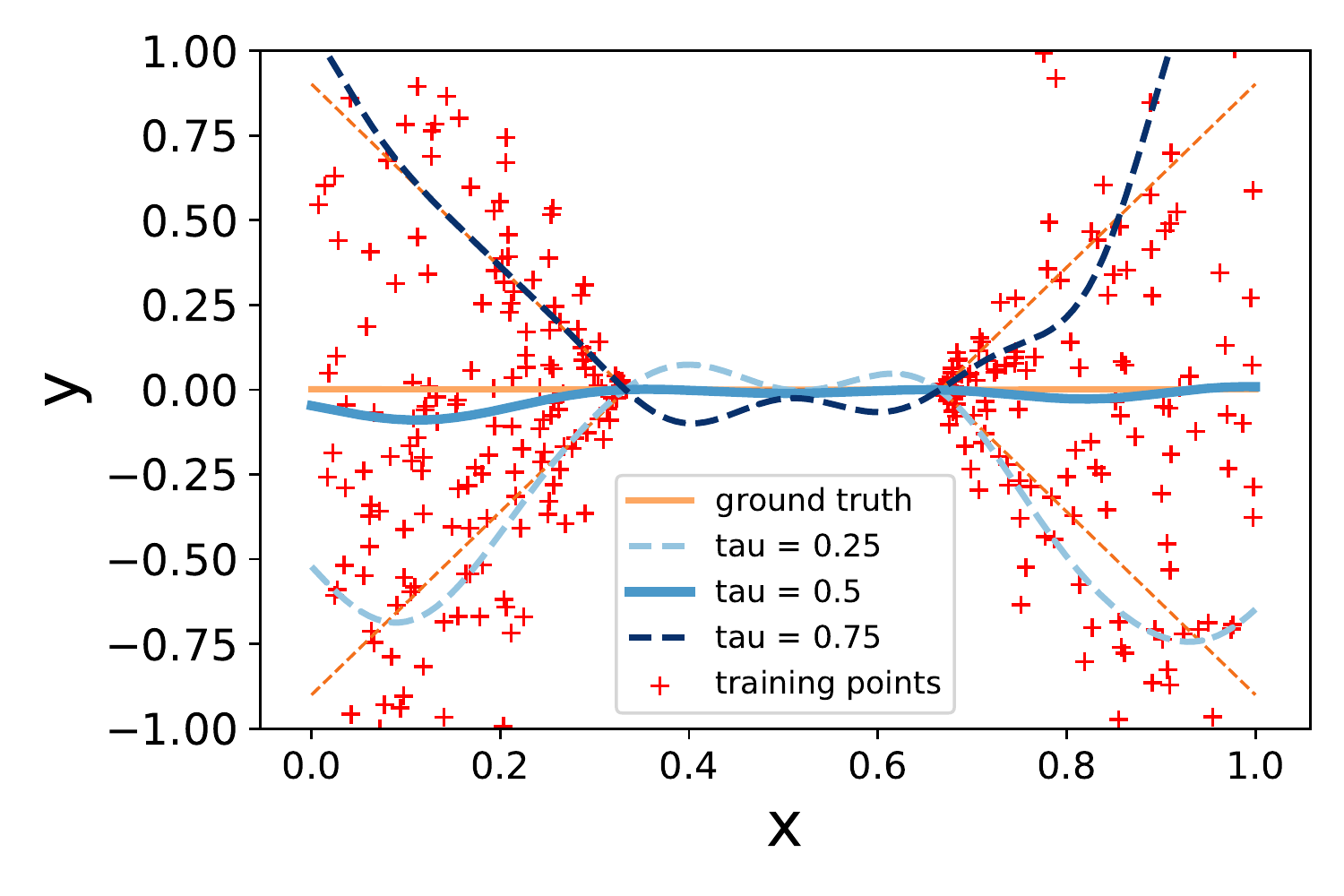}~
    \includegraphics[width=0.248\textwidth]{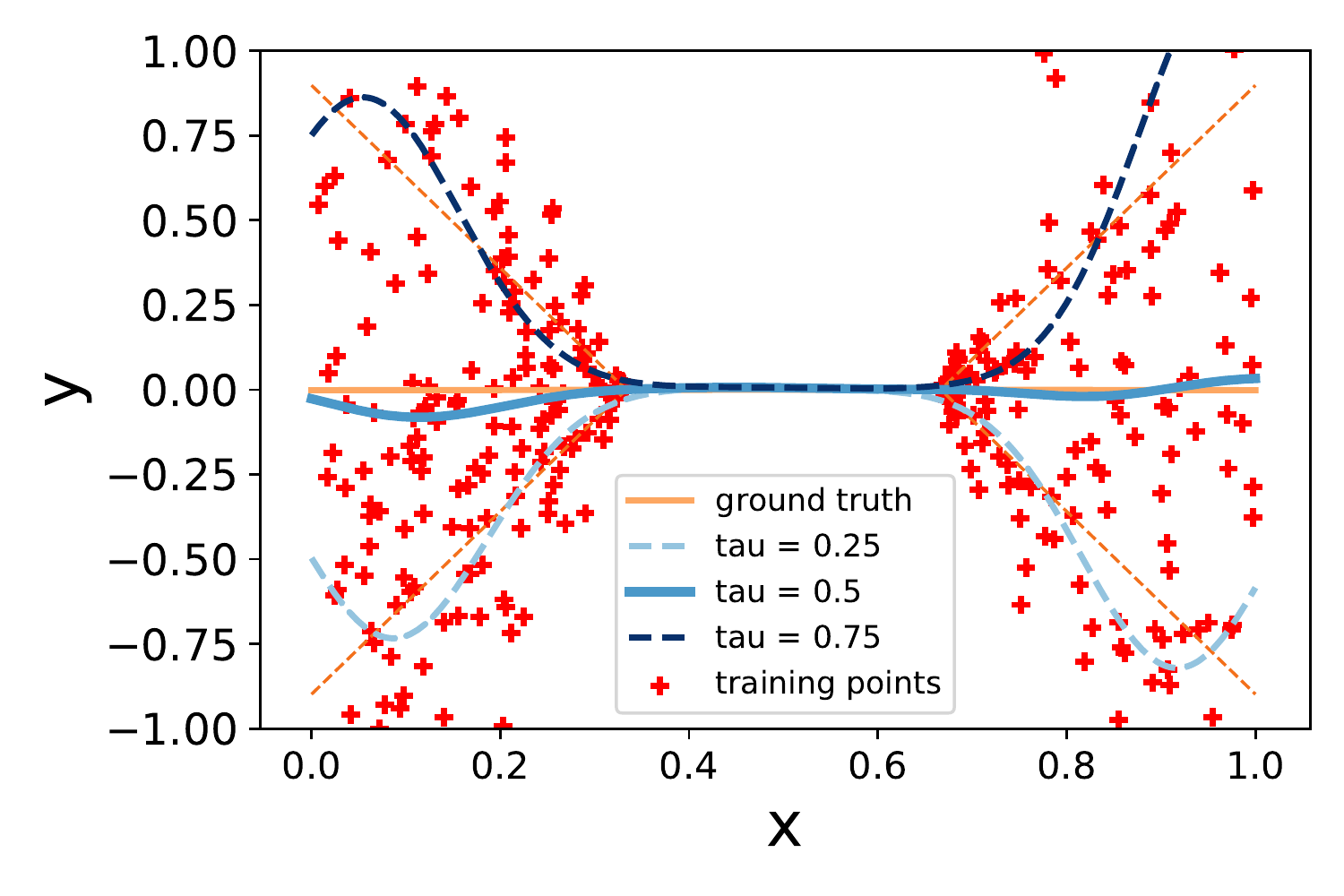}~
    \includegraphics[width=0.248\textwidth]{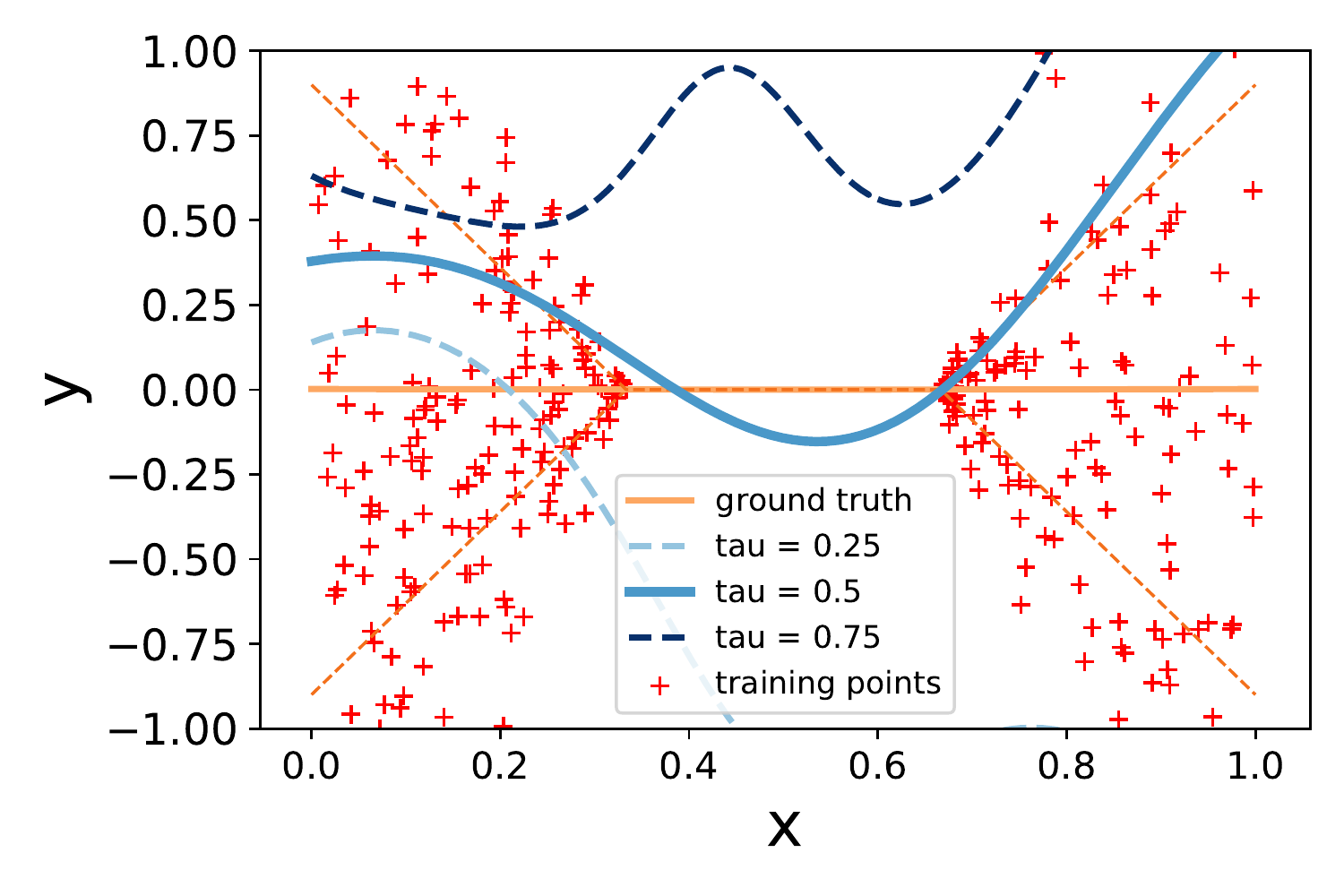}~
    \includegraphics[width=0.248\textwidth]{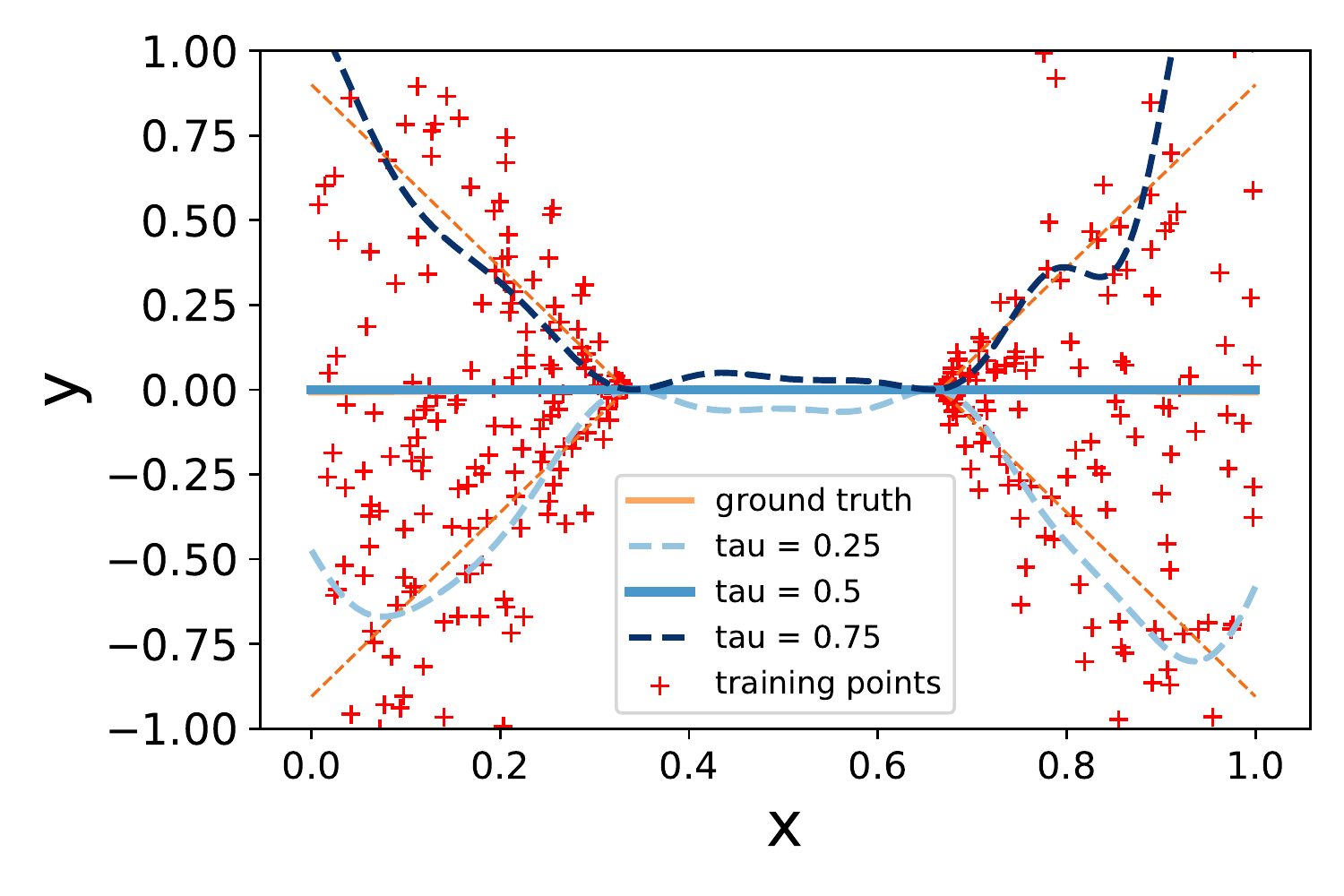}\\
    \vspace*{-.3cm}
    \caption{%
    Details in \cref{sec:experiments}. \emph{(top)} density estimation, \emph{(center)} regression with Gaussian heteroscedastic errors. \emph{(bottom)} multiple quantile regression.  Shades of blue: estimated curves. Orange: ground truth. Models: \emph{(left)} PNM, \emph{(center-left)} NCM, \emph{(center-right)}, GLM, \emph{(right)} Our model.}
    \label{fig:experiment}
\end{figure}

\paragraph{Density estimation.}
This problem is illustrated in \cref{ex:density-estimation}. Here we considered the {\em log-likelihood loss} as a measure of error, i.e., $L(z_1,\dots, z_n) = -\frac{1}{n} \sum_{i=1}^n \log(z_i)$, which is jointly convex and with an efficient proximal operator \cite{chaux2007variational}.    We recall that the problems are constrained to output a function whose integral is $1$. In \cref{fig:experiment}, we show the experiment on $n=50$ i.i.d.~points sampled from $\rho(x) = \tfrac{1}{2}{\cal N}(-1,0.3) + \tfrac{1}{2}{\cal N}(1,0.3)$ and where for all the models we used $\sigma = 1$, to illustrate pictorially the main interesting behaviors. Instead in \cref{app:details-experiments}, we perform a multivariate experiment in $d=10$ and $n=1000$, where we cross-validated $\sigma$ for each algorithm and show the same effects more quantitatively.  Note that PNM \emph{(left)} is non-negative on the training points, but it achieves negative values on the regions not covered by examples. This effect is worsened by the constraint on the integral that borrows areas from negative regions to reduce the log-likelihood on the dataset. NCM \emph{(center-left)} produces a function whose integral is one and that is non-negative everywhere, but the poor approximation properties of the model do not allow to fit the density of interest (see \cref{ex:NCM}). GLM \emph{(center-right)} produces a function that is non-negative and approximates quite well $\rho$, however, the obtained function does not sum to one, but to $0.987$, since the integral constraint can be enforced only approximately via Monte Carlo sampling (GLM does not satisfy {\bf P4}). Estimating the integral is easy in low dimensions but becomes soon impractical in higher dimensions \cite{robert2013monte}. Finally the proposed model \emph{(right)} leads to a convex problem and produces a non-negative function whose integral is $1$ and that fits the density~$\rho$ quite well.


\paragraph{\bf Heteroscedastic Gaussian process estimation.} The goal is to estimate $\mu: \R \to \R$ and $v: \R \to \R_+$ determining the conditional density $\rho$ of the form $\rho(y|x) = (2\pi v(x))^{-1/2} \exp(-(y - \mu(x))^2/(2v(x)))$ from which the data are sampled. 
The considered functional corresponds to the negative log-likelihood, i.e., $L =  \sum_{i=1}^n \frac{1}{2}\log v(x_i) + (y_i - \mu(x_i))^2/(2v(x_i))$ that becomes convex in $\eta,\theta$ via the so called {\em natural parametrization} $\eta(x) = \mu(x)/v(x)$ and $\theta(x) = 1/v(x)$ \cite{le2005heteroscedastic}. We used a linear model to parametrize $\eta$ and the non-negative models for $\theta$. The experiment 
on the same model of \cite{le2005heteroscedastic,YuanDoublyPL} is reported in \cref{fig:experiment}. Modeling $\theta$ via PNM \emph{(left)} leads to a convex problem and reasonable performance. In particular, the fact that $\theta =0$ corresponds to $v = +\infty$ prevents the model for $\theta$ from crossing zero. NCM \emph{(center-left)} leads to a convex problem, but very sensitive to the kernel width $\sigma$ and with poor approximation properties. GLM \emph{(center-right)} leads to a non-convex problem and we need to restart the method randomly to have a reasonable convergence. Our model {\em(right)} leads to a convex problem and produces a non-negative function for $\theta$, that fits well the observed data.

\paragraph{\bf Multiple quantile regression.} The goal here is to estimate multiple quantiles of a given conditional distribution $P(Y|x)$. Given $\tau \in (0,1)$, $q_{\tau}$ defined by $P(Y>q_{\tau}(x)|x) = \tau$ is the $\tau$-quantile of $\rho$. By construction $0<\tau_{-h}\leq \dots \leq \tau_h<1$ implies $q_{\tau_{-h}}(x) \leq \dots \leq q_{\tau_h}(x)$.  If we denote by ${\bf q}:\X \to \R^{2h+1}$ the list of quantiles, we have by construction ${\bf q}(x) \in \Y$ where $\Y$ is a convex cone $\Y = \{y \in \R^h~|~ y_{-h} \leq \dots \leq y_h \}$. To regress quantiles, we used the pinball loss $L_\tau$ (convex, non-smooth) considered in \cite{koenker_2005,steinwart2011estimating}, obtaining $L = \sum_{j=-h}^h\sum_{i=1}^n L_{\tau_j}(f(x_i),y_i)$. In \cref{fig:experiment}, we used $\tau_{-1} = \tfrac{1}{4}, \tau_0 = \tfrac{1}{2}, \tau_1 = \tfrac{3}{4}$. 
Using PNM, \emph{(left)} the ordering is enforced by explicit constraints on the observed dataset \cite{nonparamquantilereg,bondell2010noncrossing}. The resulting problem is convex. However, in regions with low density of points, PNM quantiles do not respect their natural order. To enforce the order constraint, a fine grid covering the space would be needed as in \cite{nonparamquantilereg}. For NCM, GLM and our model, we represented the quantiles as $q_{\tau_{\pm j}} = q_{\tau_0} \pm \sum_{i=1} v_{\pm i}$ where the $v$'s are non-negative functions and $q_{\tau_0}$, with $\tau_0 = \tfrac{1}{2}$, is the median and is modeled by a linear model. NCM \emph{(center-left)} leads to a convex problem and quantiles that respect the ordering, but the estimation is very sensitive to the chosen $\sigma$ and has poor approximation properties. GLM \emph{(center-right)} leads to a non-convex non-differentiable problem, with many local minima, which is difficult to optimize with standard techniques (see \cref{app:details-experiments}). GLM does not succeed in approximating the quantiles. Our model \emph{(right)} leads to a convex optimization problem that approximates the quantiles relatively well and preserves their natural order everywhere. 
%


\subsection*{Acknowledgments}
 This
work was funded in part by the French government under
management of Agence Nationale de la Recherche as part of
the ``Investissements d’avenir'' program, reference ANR-19-
P3IA-0001 (PRAIRIE 3IA Institute). We also acknowledge
support of the European Research Council (grant SEQUOIA
724063).

\putbib{biblio.bib}
\end{bibunit}

\newpage

\appendix

\begin{bibunit}[alpha]

\section{Notation and basic definitions}
\label{app:notation}

\begin{itemize}
    \item ${\cal H}$ is a separable Hilbert space.
    \item ${\cal X}$ is a Polish space (we will require explicitly compactness in some theorems).
    \item $\phi : \cx \rightarrow {\cal H}$ is a continous map. We also assume it to be uniformly bounded i.e. 
    $$\sup_{x \in \X} \|\phi(x)\| \leq c$$
     for some $c \in (0,\infty)$, if not differently stated.
    \item $k(x,x^{\prime}) := \phi(x)^\top \phi(x^{\prime})$ is the {\em kernel} function associated to the feature map $\phi$ \cite{scholkopf2002learning,berlinet2011reproducing}.

\end{itemize}

\section{Proofs and additional discussions}

\subsection{Proof of \cref{prop:model-non-negative-and-linear}} \label{app:proof-prop:model-non-negative-and-linear}

In this section, let us extend the definition in \cref{eq:model-non-negative} to any operator $A \in {\cal S}(\hh)$, without the implied positivity restriction (in \cref{eq:model-non-negative}, we ask that $A \succeq 0$) : 

\begin{equation} \label{eq:4bis}
    \tag{4bis}
    \forall A \in {\cal S}(\hh),~ \forall x \in \X,~ f_A(x) := \phi(x)^{\top} A \phi(x).
\end{equation}

\begin{proof}[Proof of \cref{prop:model-non-negative-and-linear}]

 To prove linearity, let $A, B \in {\cal S}(\hh)$ and $\alpha, \beta \in \R$. Since ${\cal S}(\hh)$ is a vector space, 
 $$\alpha A + \beta B \in {\cal S}(\hh).$$
  Let $x \in \cx$. By definition for the first equality and linearity for the second, 
    \[f_{\alpha A + \beta B }(x) = \phi(x)^{\top}(\alpha A + \beta B) \phi(x) = \alpha\phi(x)^{\top} A \phi(x) + \beta \phi(x)^{\top} B \phi(x).\]
    Finally, since by definition, $f_A(x) = \phi(x)^{\top} A \phi(x)$ and $f_B(x) = \phi(x)^{\top} B \phi(x)$, it holds :
    $$f_{\alpha A + \beta B }(x) = \alpha\phi(x)^{\top} A \phi(x) + \beta \phi(x)^{\top} B \phi(x) = \alpha f_A(x) + \beta f_B(x).$$
    Since this holds for all $x \in \cx$, this shows $f_{\alpha A + \beta B} = \alpha f_A + \beta f_B$. \\
    
    To prove the non-negativity, assume now that $A \succeq 0$. By definition of of positive semi-definiteness, 
    \[\forall h \in \hh,~ h^{\top} A h \geq 0.\]
    In particular, for any $x \in \cx$, the previous inequality applied to $h = \phi(x)$ yields 
    \[f_A(x) = \phi(x)^{\top}A \phi(x) \geq 0.\]
    Hence, $f_A \geq 0$. 
    
\end{proof}

\subsection{Proof of \cref{prop:model-satisfies-P1}
}\label{app:proof-prop:model-satisfies-P1}

Recall the definition of $f_A$ for any $A \in {\cal S}(\hh)$ in \cref{eq:4bis}. We have the lemma:

\blm[Linearity of evaluations]\label{lm:lineval}
Let $x_1,\dots, x_n \in \X$. Then the map 
\[A \in {\cal S}(\hh) \mapsto (f_A(x_i))_{1 \leq i \leq n} \in \R^n\]
is linear from ${\cal S}(\hh)$ to $\R^n$. 
\elm

\begin{proof}
This just follows from the fact that the definition of $f_A(x_i)$, $  f_A(x_i):= \phi(x_i)^{\top} A \phi(x_i),$
is linear in $A$.
\end{proof}

\begin{proof}[Proof of \cref{prop:model-satisfies-P1}]
Let $L:\R^n \to \R$ be a jointly convex function and $x_1,\dots, x_n \in \X$. 
The function $A \in {\cal S}(\hh) \mapsto L(f_A(x_1),\dots, f_A(x_n))$ can be written $L \circ R$, where
\[ R : A \in {\cal S}(\hh) \mapsto (f_A(x_i))_{1 \leq i \leq n} \in \R^n.\]
 Since $L$ is convex, and $R$ is linear by \cref{lm:lineval}, their composition is convex.\\
 
 Moreover, since ${\cal S}(\hh)_{+}$ is a convex subset of ${\cal S}(\hh)$, the restriction of $A \in {\cal S}(\hh) \mapsto L(f_A(x_1),\dots, f_A(x_n))$ on ${\cal S}(\hh)_{+}$ is also convex.
\end{proof}

\subsection{Proof of \cref{thm:representer-non-negative}}
\label{app:proof-thm:representer-non-negative}

In this section, we prove \cref{thm:representer-non-negative} for a more general class of spectral regularizers. 

\subsubsection{Compact operators and spectral functions}

In this section, we briefly introduce compact self-adjoint operators and the spectral theory of compact self-adjoint operators. For more details, see for instance \cite{2004linop}. We start by defining a compact self-adjoint operator (see Section2.16 of \cite{2004linop}) and stating its main properties:
\bd[compact operators]
Let $\hh$ be a separable Hilbert space. A bounded self-adjoint operator $A \in {\cal S}(\hh)$ is said to be compact if its range is included in a compact set. We denote with ${\cal S}_{\infty}(\hh)$ the set of compact self adjoint operators on $\hh$. It is a closed subspace of ${\cal S}(\hh)$ for the operator norm and the closure of the set of finite rank operators.
\ed


\bp[Spectral theorem \cite{2004linop}]
Let $\hh$ be a separable Hilbert space and let $A$ be a compact self adjoint operator on $\hh$. Then there exists a spectral decomposition of $A$, i.e., an orthonormal system $(u_k) \in \hh$ of eigenvectors of $A$ and corresponding eigenvalues $(\sigma_k)$ such that for all $h \in \hh$, it holds
\[Ah = \sum_{k}{\sigma_k u_k^\top h~u_k} =: \left(\sum_{k}{\sigma_k u_k u_k^{\top} }\right)h.
\]
Moreover, if $\sigma_k$ is an infinite sequence, it converges to zero. \\\
Furthermore, we say that the orthonormal system $(u_k)$ of eigenvectors of $A$ and the corresponding eigenvalues $(\sigma_k)$ is a \emph{basic system} of eigenvectors of $A$ if all the $\sigma_k$ are non zero. In this case, if $P_0$ denotes the orthogonal projection on $\noy(A)$, then it holds
\[\forall h \in \hh,~ h = \Pi_0~h + \sum_{k}{u_k u_k^{\top} ~h}\]
\ep

In what follows, to simplify notations, we will usually write $A = U \diag(\sigma) U^{\top}$ in order to denote a basic system of eigenvectors of $A$. Moreover, if $A$ is positive semi-definite, we will assume that the eigenvalues are sorted in decreasing order, i.e., $\sigma_{k+1} \leq \sigma_k$.

\bd[Spectral function on ${\cal S}_{\infty}(\hh)$ \cite{2004linop}]
Let $q : \R \rightarrow \R$ be a lower semi-continuous function such that $q(0) = 0$. Let $\hh$ be any separable Hilbert space.
For any $A \in {\cal S}_{\infty}(\hh)$ and any basic system $A = U \diag(\sigma) U^{\top}$, we define the spectral function $q$
\[q(A) = U \diag(q(\sigma))) U^T  = \sum_{k}{q(\sigma_k)u_k u_k^{\top}} .\]
\ed

\subsubsection{Classes of regularizers}

Let us now state our main assumption on regularizers. 

\ba[Assumption on regularizers]\label{ass:omega} For any $A \in {\cal S}(\hh)$, $\Omega(A)$ is of of the form
\[\Omega(A) = \begin{cases}
\tr(q(A)) = \sum_{k}{q(\sigma_k)} &\text{if } A = U \diag(\sigma) U^{\top} \in {\cal S}_{\infty}(\hh),~\sum_{k}{q(\sigma_k)} < \infty\\
+ \infty & \text{otherwise,}
\end{cases}\]
where $q : \R \rightarrow \R_+$ is:
\begin{itemize}
    \item non-decreasing on $\R_+$ with $q(0) = 0$;
    \item lower semi-continuous;
    \item $q(\sigma) \underset{|\sigma| \rightarrow + \infty}{\longrightarrow} +\infty$.
\end{itemize}
\ea

 Note that in this case, $\Omega$ is defined on ${\cal S}(\hh)$ for any Hilbert space $\hh$.

\br 
$\Omega(A) = \lambda_1 \|A\|_{\star} + \frac{\lambda_2}{2} \|A\|_{F}^2$ satisfies \cref{ass:omega}, with $q(\sigma) = \lambda_1 ~|\sigma| + \lambda_2~ \sigma^2$.
\er

\blm[Properties of $\Omega$] \label{lm:main_properties_ass}
Let $\Omega$ satisfying \cref{ass:omega}. Then the following properties hold.

\begin{enumerate}[label = (\roman*)]
    \item For any separable Hilbert spaces $\hh_1,\hh_2$ and any linear isometry $O : \hh_1 \rightarrow \hh_2$, i.e., such that $O^* O = I_{\hh_1}$, it holds 
    \[ \forall A   \in {\cal S}(\hh_1),~ \Omega(O A O^*) = \Omega(A).\]
    \item For any separable Hilbert space $\hh$ and any orthogonal projection $\Pi \in {\cal S}(\hh_1)$, i.e., satisfying $\Pi = \Pi^*,~ \Pi^2 = \Pi$, it holds
    \[ \forall A \succeq 0,~ \Omega(\Pi A \Pi) \leq \Omega(A).\]
    \item For any finite dimensional Hilbert space $\hh_n$, 
    \[\Omega \text{ is lower semi-continuous (l.s.c)},\qquad \Omega(A) \underset{\|A\|_{op} \rightarrow +\infty}{\longrightarrow} + \infty\]
    where we denoted by $\|\cdot\|_{op}$ the operator norm.
\end{enumerate}
\elm

\bpr 

\begin{enumerate}[label = (\roman*)]
    \item Write $A = \sum_{k}{\sigma_k u_k u_k^{\top}}$ where the $(u_k)$ form a basic system of eigen-vectors for $A$. The $(v_k) = (O u_k)$ form a basic system of eigen-vectors for $OAO^*$, as 
    \[OAO^* = \sum_{k}{\sigma_k v_k v_k^{\top}},\qquad \sigma_k \neq 0. \]
    Hence, by definition, $q(OAO^*) = \sum_{k}{q(\sigma_k) v_k v_k^{\top}}$. By definition of the trace, we have
    \[\Omega(OAO^*) =  \sum_{k}{q(\sigma_k)} = \Omega(A).\]
    \item Let $A$ be a compact self-adjoint semi-definite operator. Let $A = U \diag(\sigma) U^{\top}$ be a basic system of eigenvectors of $A$, where the $\sigma_k$ are positive and in decreasing order. Define $B = U\diag(\sqrt{\sigma}) U^{\top}$ and note that in this case, $A = B^2 = B^*B$. Using Exercise 23 of \cite{2004linop}, we have that for any orthogonal projection operator $\Pi$ and any index $k$, $\sigma_k(\Pi B^* B \Pi) \leq \sigma_k(B^*B)$ and hence $\sigma_k(\Pi A \Pi) \leq \sigma_k(A)$. Since $q$ is non decreasing, it holds $q(\sigma_k(\Pi A \Pi)) \leq q(\sigma_k(A))$ and hence 
    \[\Omega(\Pi A \Pi) = \sum_{k}{q(\sigma_k(\Pi A \Pi))} \leq \sum_{k}{q(\sigma_k(A))} = \Omega(A).\]
    \item Let $\hh_n$ be a finite dimensional Hilbert space and let $\| \cdot \|_{op}$ be the operator norm on ${\cal S}(\hh_n)$. If $q$ is continuous, then $A \in \hh_n \mapsto q(A)$ is continuous and hence $\Omega$ is continuous (since the trace is continuous in finite dimensions). Now assume $q$ is lower semi-continuous, and define for $n \in \N$, $q_n(t) := \inf_{s \in \R}{q(s) + n|t-s|}$. We have $q_n \geq 0$, $q_n(0) = 0$ $q_n$ is uniformly continuous and $q_n$ is an increasing sequence of functions such that $q_n \rightarrow q$ point-wise. Now it is easy to see that $\tr(q(A)) = \sup_{n} \tr(q_n(A))$ and hence $\Omega$ is lower semi-continuous as a supremum of continuous functions.\\
    The fact that $\Omega$ goes to infinity is a direct consequence of the fact that $q$ goes to infinity, by \cref{ass:omega}.
\end{enumerate}
\epr 

\br The three conditions of the previous lemma are in fact the only conditions needed in the proof. We could loosen \cref{ass:omega} to satisfy only these three properties.
\er

\subsubsection{Finite-dimensional representation and existence of a solution}

Fix $n \in \N$, a loss function $L : \R^n \rightarrow \R\cup \left\{+\infty\right\}$, a separable Hilbert space $\hh$, a regularizer $\Omega$ on ${\cal S}(\hh)$ a feature map $\phi : \X \rightarrow \hh$ and points $(x_1,...,x_n) \in \X^n$. 

Recall the problem in \cref{eq:problem-for-representer}:
\eqal{
\tag{\ref{eq:problem-for-representer}}
\textstyle \inf_{A \succeq 0} L(f_A(x_1),\dots,f_A(x_n)) + \Omega(A).
}

Define $\hh_n$ to be the finite-dimensional subset of $\hh$ spanned by the $\phi(x_i)$, i.e.,
\[\hh_n := \lspan{(\phi(x_i))_{1 \leq i\leq n}} = \left\{\sum_{i=1}^n{\alpha_i} \phi(x_i)~:~ \alpha \in \R^n\right\}.\]
Define $\Pi_n$ is the orthogonal projection on $\hh_n$, i.e.,
    \[\Pi_n \in {\cal S}(\hh),~\Pi_n^2 = \Pi_n,~ \range(\Pi_n) =\hh_n.\]
 Define ${\cal S}_n(\hh)_+$ to be the following subspace of ${\cal S}(\hh)_{+}$ : 
 \[{\cal S}_n(\hh)_+ := \Pi_n  {\cal S}(\hh)_{+} \Pi_n = \left\{\Pi_n A \Pi_n~:~ A  \in  {\cal S}(\hh)_{+}\right\}.\]

\bp \label{prp:prelim_thm1}
Let $L$ be a lower semi-continuous function which is bounded below, and assume $\Omega$ satisfies \cref{ass:omega}. Then \cref{eq:problem-for-representer} has a solution $A^*$ which is in ${\cal S}_n(\hh)_+$. 
\ep

\bpr In this proof, denote by $J$ the function defined by 
\[\forall A \in {\cal S}(\hh),~ J(A) := L(f_A(x_1),...,f_A(x_n)) + \Omega(A).\]
Our goal is to prove that the problem $\inf_{A \in {\cal S}(\hh)_+}{J(A)}$ has a solution which is in ${\cal S}_n(\hh)_+$, i.e., of the form  $\Pi_n A \Pi_n$ for some $A \in {\cal S}(\hh)_+$.\\

\paragraph{1.}

Let us start by fixing $A \in {\cal S}(\hh)_+$.\\
First note that since $\Pi_n$ is the orthogonal projection on $\lspan(\phi(x_i))_{1\leq i \leq n}$, in particular $\Pi_n \phi(x_i) = \phi(x_i)$ for all $1 \leq i \leq n$. Thus, for any $1 \leq i \leq n$, 
    \[f_A(x_i) = \phi(x_i)^{\top} A \phi(x_i) = \phi(x_i)^{\top}\Pi_n A \Pi_n \phi(x_i) = f_{\Pi_n A \Pi_n}(x_i).\]
    Here, the first and last equalities come from the definition of $f_A$ and $f_{\Pi_n A \Pi_n}$. Thus, 
    \[J(A) =L(f_{\Pi_n A \Pi_n}(x_1),...,f_{\Pi_n A \Pi_n}(x_n)) + \Omega(A).\]
Now since $\Omega$ satisfies \cref{ass:omega}, by the second point of \cref{lm:main_properties_ass}, it holds $\Omega(\Pi_n A \Pi_n) \leq \Omega(A)$, hence 
    \[J(\Pi_n A \Pi_n) \leq J(A).\]

This last inequality combined with the fact that ${\cal S}_n(\hh)_+ = \Pi_n {\cal S}(\hh)_+ \Pi_n \subset {\cal S}(\hh)_+$ show that 
\eqal{\label{eq:crucial}
\textstyle  \inf_{A \in {\cal S}_n(\hh)_+} J(A) =  \inf_{A \succeq 0} J(A).
}
\paragraph{2.}
Let us now show that $\inf_{A \in {\cal S}_n(\hh)_+} J(A) $ has a solution. Let us exclude the case where $J  = +\infty$, in which case $A = 0$ can be taken to be a solution.\\\

Let $V_n$ be the injection $V_n : \hh_n \hookrightarrow \hh$. Note that $V_n V_n^* = \Pi_n$ and $V_n^* V_n = I_{\hh_n}$. These simple facts easily show that 
\[{\cal S}_n(\hh)_+ = V_n {\cal S}(\hh_n)_+ V_n^* = \left\{V_n \tilde{A} V_n^*~:~ \tilde{A} \in {\cal S}(\hh_n)_+\right\}.\]

Thus, our goal is to show that
$\inf_{\tilde{A} \in {\cal S}(\hh_n)_+} J(V_n A V_n^*)$ has a solution.\\

By the first point of \cref{lm:main_properties_ass}, since $V_n^* V_n = I_{\hh_n}$, it holds 
\[\forall \tilde{A} \in {\cal S}(\hh_n),~\Omega(V_n \tilde{A} V^*_n) = \Omega(\tilde{A}) \implies J(V_n \tilde{A} V_n^*) = L(f_{V_n \tilde{A} V_n^*}(x_1),...,f_{V_n \tilde{A} V_n^*}(x_n)) + \Omega(\tilde{A}).\]

Let $\tilde{A}_0 \in {\cal S}(\hh_n)_{+}$ be a point such that $J_0 :=J(V_n\tilde{A}_0 V_n^*) < \infty$. Let $c_0$ be a lower bound for $L$. By the third point of \cref{lm:main_properties_ass}, there exists a radius $R_0$ such that for all $\tilde{A} \in {\cal S}(\hh_n)$,
\[ \|\tilde{A}\|_F > R_0  \implies \Omega(\tilde{A}) > J_0 - c_0.\]
Since $c_0$ is a lower bound for $L$, this implies
\eqals{
\textstyle  \inf_{\tilde{A} \in {\cal S}(\hh_n)_+}J(V_n \tilde{A} V^*_n) =  \inf_{\tilde{A} \in {\cal S}(\hh_n)_+,~\|\tilde{A}\|_F \leq R_0}J(V_n \tilde{A} V^*_n).
}
Now since $L$ is lower semi-continuous, $\Omega$ is lower semi-continuous by the last point of \cref{lm:main_properties_ass}, and $\tilde{A} \mapsto (f_{V_n \tilde{A} V_n^*}(x_i))_{1 \leq i \leq n}$ is linear hence continuous, the mapping $A \mapsto  J(V_n \tilde{A} V_n^*)$ is lower semi-continuous. Hence, it reaches its minimum on any non empty compact set. Since $\hh_n$ is finite dimensional, the set $\left\{\tilde{A} \in {\cal S}(\hh_n)_+~:~\|\tilde{A}\|_F \leq R_0\right\}$ is compact (closed and bounded) and non empty since it contains $\tilde{A}_0$, and hence there exists $\tilde{A}_* \in {\cal S}(\hh_n)_+$ such that $J(V_n \tilde{A}_* V_n^*) =   \inf_{\tilde{A} \in {\cal S}(\hh_n)_+,~\|\tilde{A}\|_F \leq R_0}J(V_n \tilde{A} V^*_n)$. 
Going back up the previous equalities, this shows that $A_* = V_n \tilde{A}_* V_n^* \in {\cal S}_n(\hh)_+$ and $J(A_*) = \inf_{A \succeq 0}J(A)$. 
\epr 

\subsubsection{Proof of \cref{thm:representer-non-negative}}

We will prove the following \cref{thm:1bis} whose statement is that of  \cref{thm:representer-non-negative} with more general assumptions.


\bt \label{thm:1bis} Let $L$ be lower semi-continuous and bounded below, and $\Omega$ satisfying \cref{ass:omega}. Then \cref{eq:problem-for-representer} has a solution $A_*$ which can be written in the form
\eqals{
 \sum_{i,j=1}^n \mathbf{B}_{ij} \phi(x_i)\phi(x_j)^\top, \qquad \textrm{for some matrix} ~ \mathbf{B} \in \R^{n\times n}, ~ \mathbf{B} \succeq 0.
}
Moreover, if $L$ is convex, and $\Omega$ is of the form \cref{eq:base-regularizer} with $\lambda_2 > 0$, this solution is unique. By \cref{eq:model-non-negative}, $A_*$ corresponds to a function of the form
$$f_*(x)= \sum_{i,j=1}^n \mathbf{B}_{ij} k(x,x_i)k(x,x_j).$$
\et

\blm \label{lm:thm1}
The set ${\cal S}_n(\hh)_{+}$ can be represented in the following way
\[{\cal S}_n(\hh)_{+} = \left\{\sum_{1 \leq i,j \leq n}{\mathbf{B}_{i,j} \phi(x_i)\phi(x_j)^{\top}},~:~ \mathbf{B} \in \R^{n \times n},~ \mathbf{B} \succeq 0\right\}.\]
In particular, for any $A \in {\cal S}_n(\hh)_{+}$, there exists a matrix $\mathbf{B} \in \R^{n \times n}$, $\mathbf{B} \succeq 0$ such that 
\[A =\sum_{1 \leq i,j \leq n}{\mathbf{B}_{i,j} \phi(x_i)\phi(x_j)^{\top}} \implies  \forall x \in \X,~ f_A(x) = \sum_{1 \leq i,j \leq n}{\mathbf{B}_{i,j} k(x_i,x)k(x_j,x)}.\]
\elm 

\bpr Define $S_n : \hh \rightarrow \R^n$ to be the operator such that 
\[\forall h,~ S_n(h) = \left(h^{\top} \phi(x_i)\right)_{1 \leq i \leq n},\]
with adjoint $S_n^* :  \R^n \rightarrow \hh$ such that 
\[ \forall \alpha \in \R^n,~ S_n^* \alpha = \sum_{i=1}^n{\alpha_i \phi(x_i)}.\]
Note that for any ${\bf B} \in \R^{n \times n},~ S_n^* {\bf B} S_n = \sum_{i,j}{{\bf B}_{i,j} \phi(x_i)\phi(x_j)^{\top}}$. 
\paragraph{1. Proving ${\cal S}_n(\hh)_{+} \subset \left\{\sum_{1 \leq i,j \leq n}{\mathbf{B}_{i,j} \phi(x_i)\phi(x_j)^{\top}},~:~ \mathbf{B} \in \R^{n \times n},~ \mathbf{B} \succeq 0\right\}$.}

Let $\Pi_n A \Pi_n$ be in ${\cal S}_n(\hh)_{+}$. Using the previous equality, we want to show there exists ${\bf B} \in \R^{n\times n},~ {\bf B} \succeq 0$ such that $\Pi_n A \Pi_n = S_n^* {\bf B} S_n$.  Using \cref{lm:important}, we see that $\Pi_n$ can be written in the form $S_n^* T_n$ where $T_n : \hh \rightarrow \R^n$ (write $\Pi_n = O_n O_n^*$ and note that $O_n$ is of the form $S_n^* \tilde{O}_n$). Hence, defining ${\bf B}$ to be the matrix associated to the operator $T_n A T_n^* : \R^n \rightarrow \R^n$, it holds $\Pi_n A \Pi_n = S_n^* {\bf B} S_n$. Moreover, $A \succeq 0$ implies ${\bf B} = T_n A T_n^* \succeq 0$.
\paragraph{2. Proving $ \left\{\sum_{1 \leq i,j \leq n}{\mathbf{B}_{i,j} \phi(x_i)\phi(x_j)^{\top}},~:~ \mathbf{B} \in \R^{n \times n},~ \mathbf{B} \succeq 0\right\} \subset {\cal S}_n(\hh)_{+}$.} 

Let ${\bf B} \in \R^{n \times n}$  and assume ${\bf B} \succeq 0$. Since ${\bf B} \succeq 0$, $A := S_n^* {\bf B} S_n \succeq 0$.  Since $S_n^*$ has its range included in $\hh_n$, $\Pi_n S_n^* = S_n^*$. Thus, $ \Pi_n A \Pi_n = A$ and hence $A \in {\cal S}_n(\hh)_+$.\\\

The second statement comes from the definition of $f_A(x)$. Indeed assume $A \in {\cal S}_n(\hh)_+$. By definition, $f_A(x) = \phi(x)^{\top} A \phi(x)$. Moreover, by the previous point, there exists ${\bf B} \in \R^{n \times n},~ {\bf B} \succeq 0$ such that $A = \sum_{1 \leq i,j \leq n}{{\bf B}_{i,j} \phi(x_i) \phi(x_j)^{\top}}$. Combining these two facts yields:
\[\forall x \in \X,~ f_A(x) = \sum_{1 \leq i,j \leq n}{{\bf B}_{i,j} \phi(x)^\top \phi(x_i)~ \phi(x_j)^{\top} \phi(x)} = \sum_{1 \leq i,j \leq n}{{\bf B}_{i,j} k(x,x_i)~k(x,x_j)}.\]
The last equality comes from the definition $k(x,\tilde{x}) = \phi(x)^{\top} \phi(\tilde{x})$. 
\epr 

\begin{proof}[Proof of \cref{thm:1bis}]
Under the assumptions of \cref{thm:1bis}, one satisfies the assumptions of \cref{prp:prelim_thm1}. Thus, \cref{eq:problem-for-representer} has a solution $A_*$ which is in ${\cal S}_n(\hh)_+$. Now applying \cref{lm:thm1}, $A_*$ can be written in the form $A_* = \sum_{i,j}{{\bf{B}}_{i,j}\phi(x_i)\phi(x_j)^{\top}}$ for ${\bf B} \in \R^{n \times n}$, ${\bf B} \succeq 0$, and hence
\[\forall x \in \X,~ f_{A_*}(x) = \sum_{i,j}{{\bf{B}}_{i,j}k(x,x_i)k(x,x_j)}.\]
Uniqueness in the case where $\Omega$ is of the form \cref{eq:base-regularizer} with $\lambda_2>0$ comes from the fact that the loss function is strongly convex in this case, and thus the minimizer is unique. 
\end{proof}

\subsection{Proof of \cref{prop:finite-dimensional-primal-characterization}}
\label{app:proof-prop:finite-dimensional-primal-characterization}

Recall the definitions of $S_n : \hh \rightarrow \R^n$ and its adjoint $S_n^* : \R^n \rightarrow \hh$ : 
\[\forall h,~ S_n(h) = \left(h^{\top} \phi(x_i)\right)_{1 \leq i \leq n},~\forall \alpha \in \R^n,~ S_n^* \alpha = \sum_{i=1}^n{\alpha_i \phi(x_i)}.\]
Note that the kernel matrix $\mathbf{K} = \left(k(x_i,x_j)\right)_{1 \leq i,j \leq n}$ can also be written as $\mathbf{K} = S_n S_n^*$. \\
Let $r$ be the rank of ${\bf K}$ and ${\bf V} \in \R^{r \times n}$ be a matrix such that 
\[{\bf V}^{\top} {\bf V} = {\bf K}.\]
Note that ${\bf V}$ is of rank $r$ and hence $\bf VV^{\top}$ is invertible, making the following definition of $O_n : \R^r \rightarrow \hh$ valid:
 \[ O_n = S_n^* {\bf V}^{\top} ({\bf V \bf V^{\top}})^{-1}.\]
 The following result holds : 
\blm\label{lm:important}
$O_n O_n^* = \Pi_n$ and $O_n^* O_n = I_r$.  
\elm 

\bpr 
Using the fact that ${\bf V^{\top}V} = {\bf K} = S_n S_n^*$, we have 
\[O_n^* O_n = ({\bf V \bf V^{\top}})^{-1} {\bf V} S_nS_n^* {\bf V}^{\top} ({\bf V \bf V^{\top}})^{-1} = ({\bf V \bf V^{\top}})^{-1} {\bf V} {\bf V^{\top}V} {\bf V}^{\top} ({\bf V \bf V^{\top}})^{-1}  = I_r.\]
Now let us show that $O_n O_n^* = \Pi_n$. First of all, $\tilde{\Pi}_n := O_n O_n^*$ is self adjoint and is a projection operator since $\tilde{\Pi}_n^2 = O_n (O_n^* O_n) O_n^* = O_n O_n^* = \tilde{\Pi}_n$ by the previous point. Moreover, its range is included in $\lspan(\phi(x_i))_{1 \leq i \leq n}$ since $O_n = S_n^* \tilde{O}_n$ for a certain $\tilde{O}_n$ and the range of $S_n^*$ is $\lspan(\phi(x_i))_{1 \leq i \leq n}$. 
Finally since the rank of $S_n^*$ is also the rank of $S_n S_n^*$ which is $r$, we deduce that the range of $\lspan(\phi(x_i))_{1 \leq i \leq n}$ is of dimension $r$ and hence, since $O_n^* O_n = I_r$ implies that $O_n O_n^*$ is of rank $r$, putting things together, $\tilde{\Pi}_n = \Pi_n$. 
\epr 

\br[Constructing ${\bf V}$] In the case where the kernel matrix ${\bf K}$ is full rank, ${\bf V} \in \R^{n \times n}$ and is invertible, and $O_n$ can be simply written $S_n^* {\bf V}^{-1}$.\\
 In the case where the kernel matrix ${\bf K}$  is not full-rank, we build ${\mathbf V}$ as ${\mathbf V} = {\bf \Sigma}^{1/2} {\bf U}^\top$, where ${\bf \Sigma} \in \R^{r \times r}$  is diagonal and ${\bf U} \in \R^{n \times r}$ is unitary and correspond to the {\em economy eigendecomposition} of ${\bf K}$ where $r$ is the rank of ${\bf K}$, i.e., ${\bf K} = {\bf U} {\bf \Sigma} {\bf U}^\top$.
\er 
Consider the following generalization of the finite dimensional model proposed in \cref{eq:emp-feature-map-solution} in the case where ${\bf K}$ is not necessarily full rank : 
\eqal{
\tag{\ref{eq:emp-feature-map-solution}}
\tilde{f}_{\mathbf{A}}(x) = \Phi(x)^\top \mathbf{A} \Phi(x), \qquad {\mathbf A} \in \R^{r\times r},~ \mathbf{A} \succeq 0,
}
where $\Phi : \X \mapsto \R^r$ is defined as $\Phi(x) = O_n^* \phi(x) = ({\bf VV^{\top}})^{-1}{\bf V} v(x)$, where $v(x) = (k(x_i,x))_{1 \leq i \leq n} \in \R^n$. 

We are now ready to prove \cref{prop:finite-dimensional-primal-characterization}.

\begin{proof}[Proof of \cref{prop:finite-dimensional-primal-characterization}] Recall 
\eqal{
\tag{\ref{eq:finite-dimensional-primal-characterization}}
\textstyle \min_{\mathbf{A} \succeq 0} L(\tilde{f}_\mathbf{A}(x_1), \dots, \tilde{f}_\mathbf{A}(x_n)) + \Omega(\mathbf{A}).
}

The fact that \cref{eq:finite-dimensional-primal-characterization} has a solution, and that this solution is unique if $\lambda_2 >0$ and $L$ is convex can be seen as a simple consequence of \cref{thm:1bis} in the case where the model considered is the finite dimensional model defined in \cref{eq:emp-feature-map-solution}. Let us now prove the other part of the proposition.

Start by noting that with our definition of $O_n$, for all ${\bf A} \in \R^{r \times r},~ {\bf A} \succeq 0$,~
\eqal{\tag{a}
\label{(a)}
f_{O_n {\bf A} O_n^*} = \tilde{f}_{\bf A}.}

Moreover, 
\eqal{
\tag{b}\label{(b)}
\left\{O_n {\bf A} O_n^* ~:~ {\bf A } \in \R^{r\times r},~ {\bf A } \succeq 0\right\} = {\cal S}_n(\hh)_+.}
Finally, since $O_n$ is an isometry which implies $\Omega(O_n {\bf A} O_n^* ) = \Omega({\bf A})$ and by \cref{(a)}, for any $ {\bf A} \in {\cal S}(\R^n)_{+}$, it holds :
\eqal{
\tag{c}\label{(c)}
L(f_{O_n {\bf A} O_n^*}(x_1),...,f_{O_n {\bf A} O_n^*}(x_n)) + \Omega(O_n {\bf A} O_n^*) = L(\tilde{f}_{ {\bf A} }(x_1),...,\tilde{f}_{{\bf A} }(x_n)) + \Omega({\bf A}).}

Now combining \cref{(c)} and \cref{(b)}, any solution ${\bf A}_*$ to \cref{eq:finite-dimensional-primal-characterization} corresponds to a solution $A_* \in \argmin_{A \in {\cal S}_n(\hh)_+}{L(f_A(x_1),...,f_A(x_n))+ \Omega(A)}$, where $A_* = O_n {\bf A}_* O_n^*$. Now using \cref{eq:crucial} in the proof of \cref{prp:prelim_thm1}, we see that $A_*$ is also a minimizer of \cref{eq:problem-for-representer} hence the result. 

Note that the fact that the condition number of the problem, if it exists, is preserved because $O_n$ is an isometry.
\end{proof}

\subsection{Proof of \cref{thm:dual} and algorithmic consequence.\label{app:proof-thm:dual}}

In this section, we prove \cref{thm:dual} and explain how to derive an efficient algorithm to solve it in certain cases.\\\

Let us start by proving the following lemma. 

\blm \label{lm:omega_grad}
Let $\lambda_1,\lambda_2 \geq 0$ and assume $\lambda_2 > 0$. Let $\Omega_+$ be defined on ${\cal S}(\R^r)$ as follows : 
\[\Omega_+(A) = \begin{cases} 
\lambda_1 \|A\|_\star + \frac{\lambda_2}{2}\|A\|_F^2 & \text{ if } A \succeq 0 ;\\
 + \infty & \text{ otherwise .}
\end{cases}
\]
Then $\Omega_+$ is a closed convex function, and its Fenchel conjugate is given for any $B \in {\cal S}(\R^r)$ by the formula:
\[\Omega_+^*(B) = \frac{1}{2\lambda_2} \left\|\left[B - \lambda_1 I\right]_{+}\right\|^2_{F}.\]
Moreover, $\Omega_+$ is differentiable at every point, and is $1/\lambda_2$ smooth. Its gradient is given by:
\[\nabla \Omega_+^*(B) = \frac{1}{\lambda_2}\left[B-\lambda_1 I\right]_{+}.\]
\elm

\bpr 
Write 
\[\Omega_+(A) = \iota_{{\cal S}(\R^r)_+} + \lambda_1 \|A\|_\star + \frac{\lambda_2}{2} \|A\|_F^2.\]
Here, $\iota_C$ stands for the characteristic function of the convex set $C$, i.e. $\iota_C(x) = 0$ if $x\in C$ and $+\infty$ otherwise. Since $\|\cdot\|^2_{F}$ and $\|\cdot\|_\star$ are both convex, continuous, and real valued, and since $\iota_{{\cal S}(\R^r)_+}$ is closed since ${\cal S}(\R^r)_+$ is a closed non-empty convex subset of ${\cal S}(\R^r)$, this shows that $\Omega_+$ is indeed convex and closed. Note that it is continuous on its domain ${\cal S}(\R^r)_+$. Moreover, it is strongly convex since $\lambda_2 > 0$.
Fix $B \in {\cal S}(\R^r)$ and consider the problem
\[\sup_{A \in {\cal S}(\R^r)}{\tr(AB) - \Omega_+(A)} = \sup_{A \succeq 0} \tr(A(B - \lambda_1 I)) - \frac{\lambda_2}{2}\|A\|^2_F \]
Since $\Omega_+$ is strongly convex, we know there exists a unique solution to this problem.

Note that $A_* = \argmax \tr(AB) - \Omega_+(A)$ if and only if 
\[A_* = \argmin_{A \in {\cal S}(\R^r)_+} \frac{1}{2}\left\|\left(A  - \frac{1}{\lambda_2}(B-\lambda_1 I)\right)\right\|^2.  \]
That is $A_*$ is the orthogonal projection of $\frac{B - \lambda_1 I }{\lambda_2}$ on ${\cal S}(\R^r)_+$ for the Frobenius scalar product. Hence, $A_* = \left[\frac{B - \lambda_1 I}{\lambda_2}\right]_{+}$.

Here, for any symetric matrix $C$, we denote with $[C]_+$ resp $[C]_{-}$ its positive resp negative part. Given an eigendecomposition $C = U \Sigma U^T$ with $\Sigma$ diagonal, they are defined by $[C]_+ = U \max(0,\Sigma) U^T$ and $[C]_- = U \max(0,-\Sigma) U^T$.
Hence, the Fenchel conjugate of $\Omega_+$ is given by 
\[\Omega_+^*(B) = \frac{1}{2\lambda_2}\left\|\left[B - \lambda_1 I\right]_{+}\right\|_{F}^2.\]
Consider $\omega_+^* : \sigma \in \R \mapsto \max(0,\sigma^2) \in \R$. $\omega_+^*$ is $1$-smooth and differentiable, and $(\omega_+^*)^{\prime}(\sigma) = \max(0,\sigma)$. Hence, the function 
\[B \mapsto \tr(\omega_+^*(B)) = \left\|[B]_{+}\right\|_{F}^2\]
is differentiable and $1$-smooth, with differential given by the spectral function $(\omega_+^*)^{\prime}(B) = [B]_+$. 
Hence, $\Omega_+$ is differentiable and $\nabla \Omega_+^*(B)  = \frac{1}{\lambda_2}[B -\lambda_1 I]_{+}$, and is $1/\lambda_2$ smooth.
\epr 

\bt[Convex dual problem]\label{thm:dualbis}
Let $L : \R^n \rightarrow \R \cup \set{+\infty}$  be convex closed function and  $L^*$ be the Fenchel conjugate of $L$ (see \cite{boyd2004convex} for the definition of closed and of the dual conjugate). Assume $\Omega$ is of the form \cref{eq:base-regularizer}. Assume there exists ${\bf A} \in \R^{r \times r}$, $\bf{A} \succeq 0$ such that L is continuous in $(\tilde{f}_{\bf A}(x_i))_{ 1\leq i\leq n} $.

Then the problem in \cref{eq:finite-dimensional-primal-characterization} has the following dual formulation, 
\begin{align}
\tag{\ref{dualfinited}}
\sup_{\alpha \in \R^n} -L^*(\alpha) - \tfrac{1}{2\la_2}\|[{\bf V}\diag(\alpha){\bf V}^\top+\la_1 {\bf I}]_{-}\|^2_{F},
\end{align}
and this supremum is atteined. Let $\alpha^* \in \R^n$ be a solution of \eqref{dualfinited}. Then, the solution of  \eqref{eq:problem-for-representer} is obtained via \eqref{eq:char-optimal-solution-representer}, with ${\mathbf B} \in \R^{r\times r}, {\bf{B}} \succeq 0$ as 
\begin{align}
    \tag{\ref{expression_from_dual}}
    {\bf B} = {\bf V}^{\top}{(\bf V V^{\top})^{-1}}  \left(\frac{1 }{\lambda_2}\left[{\bf V} \diag(\alpha_*) {\bf V}^{\top} + \lambda_1 I\right]_-\right){(\bf V V^{\top})^{-1}} {\bf V}.
\end{align} 
\et
\begin{proof}[Proof of \cref{thm:dualbis}]

We apply theorem 3.3.1 of \cite{Borwein:1616007} with the following parameters (on le left, the ones in theorem 3.3.1 of \cite{Borwein:1616007} and on the right the ones by which we replace them).

\begin{center}
\begin{tabular}{c|c}
     ${\bf E}$ & ${\cal S}(\R^r) $ \\
    ${\bf Y}$ & $\R^n$  \\
    $A : {\bf E} \rightarrow {\bf Y}$ & $R :{\bf A} \in {\cal S}(\R^r) \mapsto (\tilde{f}_{\bf A}(x_1),...,\tilde{f}_{\bf A}(x_n))  \in \R^n$\\
    $f : {\bf E} \rightarrow ]-\infty,+\infty]$ & $\Omega_+  : {\cal S}(\R^r) \rightarrow ]-\infty,+\infty] $\\
    $g : {\bf Y} \rightarrow ]-\infty,+\infty] $ & $L : \R^n \rightarrow ]-\infty,+\infty]$\\
    $p = \inf_{x \in {\bf E}}{g(Ax) + f(x)}$ & $ p = \inf_{{\bf A} \in {\cal S}(\R^r)}{L(\tilde{f}_{\bf A}(x_1),...,\tilde{f}_{\bf A}(x_n)) + \Omega_+({\bf A})}$\\
    $d = \sup_{\phi \in {\bf Y}}{- g^*(\phi) - f^*(-A^*\phi)}$ & $ d =  \sup_{\alpha \in \R^n}{- L^*(\alpha) - \Omega_+^*(-R^*(\alpha))}$
\end{tabular}
\end{center}

Indeed, for all $1 \leq i \leq n$, if $\Phi$ is defined in \cref{eq:emp-feature-map-solution}, $\Phi(x_i) =  {\bf V}e_i$ and thus $\tilde{f}_{\bf A}(x_i) =\Phi(x_i)^{\top} {\bf A} \Phi(x_i)= e_i^{\top} ({\bf V}^{\top} {\bf A} {\bf V})e_i.$ Thus, for any ${\bf A} \in {\cal S}(\R^r),~ R({\bf A}) :=  \left(\tilde{f}_{\bf A}(x_i)\right)_{1\leq i \leq n} =  \diag({\bf V}^{\top} {\bf A} {\bf V})$. The following properties are satisfied : 

\begin{itemize}
    \item $L$ is lower semi-continuous, convex and bounded below hence closed (see \cite{Borwein:1616007});
    \item similarly, $\Omega_+$ is a non negative closed convex function, with dual $\Omega_+^*$ given in \cref{lm:omega_grad} which is differentiable and smooth; 
    \item $\dom(\Omega_+) = {\cal S}(\R^n)_+$ ;
    \item $R$ is linear, and for any $\alpha \in \R^n$, it holds $R^* \alpha = {\bf V} \diag(\alpha) {\bf V}^{\top}$;
    \item The dual $d $ can therefore be re-expressed as \cref{dualfinited}, using the expressions for $\Omega_+^*$ and $R^*$ : 
    \eqal{\tag{\ref{dualfinited}}
    \  \sup_{\alpha \in \R^n}{-L^*(\alpha) - \frac{1}{2\lambda_2}\left\|\left[{\bf V} \diag( \alpha){\bf V}^{\top} + \lambda_1 I \right]_{-}\right\|^2_F}  
    }
    \item Assume there exists ${\bf A} \in \R^{r \times r}$, $\bf{A} \succeq 0$ such that $L$ is continuous in $(\tilde{f}_{\bf A}(x_i))_{ 1\leq i\leq n}$. Then there exists a point of continuity of $g$ such which is also in $R \dom f$, hence the assumption of theorem 3.3.1 of \cite{Borwein:1616007} is satisfied. 
\end{itemize}

Applying theorem 3.3.1 of \cite{Borwein:1616007}, the following properties hold:
\begin{itemize}
    \item $d = p$, 
    \item $d$ is atteined for a certain $\alpha_* \in \R^n$. Indeed,  there exists ${\bf A } \in \dom \Omega_+$ such that $R({\bf A}) \in \dom(L)$. Thus , $L(R({\bf A})) + \Omega_+({\bf A}) < + \infty$ and hence $d < +\infty$. Moreover, since $L$ and $\Omega_+$ are lower bounded, this shows that $d$ is lower bounded and hence $d > -\infty$. Hence $d$ is finite and thus is atteined by theorem 3.3.1.
\end{itemize}

Now using Exercise 4.2.17 of \cite{Borwein:1616007} since $L$ and $\Omega_+$ are closed convex and since $\Omega_+^*$ is differentiable, we see that the optimal solution of the primal problem ${\bf A}_*$ is given by the following formula:
\[{\bf A}_* = \nabla \Omega_+^*(-R^*\alpha^*) = \frac{1 }{\lambda_2}\left[{\bf V} \diag(\alpha_*) {\bf V}^{\top} + \lambda_1 I\right]_- .\]

Thus, for any $x \in \X$, using the definition of $\Phi(x)$, it holds 

\begin{align*}
\tilde{f}_{\bf A}(x) &= \Phi(x)^{\top}{\bf A}_* \Phi(x) \\
&= v(x)^{\top}{\bf V}^{\top}{(\bf V V^{\top})^{-1}}  \left(\frac{1 }{\lambda_2}\left[{\bf V} \diag(\alpha_*) {\bf V}^{\top} + \lambda_1 I\right]_-\right){(\bf V V^{\top})^{-1}} {\bf V} v(x).
\end{align*}
Thus, setting 
\[{\bf B} = {\bf V}^{\top}{({\bf V V}^{\top})^{-1}}  \left(\frac{1 }{\lambda_2}\left[{\bf V} \diag(\alpha_*) {\bf V}^{\top} + \lambda_1 I\right]_-\right){({\bf V V}^{\top})^{-1}} {\bf V},\]
it holds $\tilde{f}_{\bf A}(x) = v(x)^{\top} {\bf B} v(x)$. Since $v(x) = (k(x,x_i))_{1 \leq i \leq n} \in \R^n$, this shows the result. In particular, note that when $\bf V$ is invertible (i.e. when $\bf K$ is full rank) then the equation above is exactly \cref{expression_from_dual}, since ${\bf V}^{\top}{({\bf V V}^{\top})^{-1}} = {\bf V}^{-1}$.

\end{proof}

\begin{proof}[Proof of \cref{thm:dual}]
It is a direct consequence of the previous theorem.

\end{proof}

Note that the conditions of theorem \cref{thm:dual} are satisfied in many interesting cases, such as the ones described in the following proposition.
\bp
Assume one of the following conditions is satisfied : 
\begin{enumerate}[label = (\roman*)]
    \item $\dom(L) = \R^n$;
    \item $ \R^{n}_{++} \subset \dom(L)$ and $k(x_i,x_i) > 0$ for all $1 \leq i \leq n$
    \item $\bf K$ is full rank and there exists a continuity point $\alpha_0$ of $L$ such that $\alpha_0 \in \R^{n}_{+}$.
\end{enumerate}
Then there exists $\bf{A} \in {\cal S}(\R^n)_+$ such that $L$ is continuous in $(\tilde{f}_{\bf{A}}(x_1),...,\tilde{f}_{\bf{A}}(x_n))$.
\ep 

\bpr 
Let us prove these points.
\begin{itemize}
\item if $\dom(L) = \R^n$, since $L$ is convex, $L$ is continuous everywhere. Taking $\bf{A} = 0$, the result holds.
\item if $k(x_i,x_i) >0$ for all $i > 0$, then taking ${\bf A}= I_r$, we have $(\tilde{f}_{\bf A}(x_i))_{1 \leq i\leq n} = (k(x_i,x_i))_{1 \leq i\leq n}$ which is in  $\R^n_{++}$. Since $\R^n_{++} \subset \dom(L)$ and $\R^n_{++}$ is open, $L$ is continuous on $\R^n_{++}$ and hence, $\bf{A}$ satisfies the desired property.
\item Let $\alpha_0$ be a continuity point of $L$ in $\R^n_+$. If we assume ${\bf K}$ is full rank, then in particular, ${\bf V}\in \R^{n \times n}$ is of rank $n$ and invertible. Thus, there exists ${\bf A}\in {\cal S}(\R^r)_+$ such that 
\[{\bf V}^{\top} {\bf A}{\bf V} = \diag(\alpha_0) \implies (\tilde{f}_{\bf A}(x_i))_{1 \leq i\leq n}  = \alpha_0.\]
\end{itemize}
\epr

\paragraph{Discussion on how to solve  \cref{dualfinited}}

Proximal splitting methods can be applied to solve \cref{dualfinited} such as FISTA \cite{beck2009fast}, provided the proximal operator of $L^*$ can be computed (see \cite{boyd14} for the definition of the proximal operator). Indeed, \cref{dualfinited} can be written as 
\[\min_{\alpha \in \R^n}{F(\alpha) = f(\alpha) + g(\alpha)},\qquad f(\alpha) = \Omega_+^*(-{\bf V} \diag(\alpha){\bf V}^{\top}) ,~ g(\alpha) = L^*(\alpha).\]

where $\Omega_+^*$ has been defined in  \cref{lm:omega_grad} and has been shown to be smooth and differentiable. Thus, since $\alpha \mapsto {\bf V}\diag(\alpha) {\bf V}^\top$ is linear, $f$ is smooth and differentiable. Moreover, one can have access to the gradient of $f$ by performing an eigenvalue decomposition of ${\bf V}\diag(\alpha) {\bf V}^\top$ whose complexity is bounded above by ${\cal O}(r^3)$. Thus, one can apply one of the algorithms in section 4 of \cite{beck2009fast} in order to compute an optimal solution to \cref{dualfinited}. Moreover, a bound on the performance of the algorithm is given in theorem 4.4 of this same work. 
Note that if $L$ is of the form $L(\alpha) = \sum_{i=1}^n{\ell_i(\alpha_i)}$, it suffices to be able to compute the proximal operator of the $\ell_i$ to get a proximal operator for $L^*$ (see \cite{boyd14}).

\subsection{Proof and additional discussion of \cref{thm:universality}}
\label{app:proof-thm:universality}
We recall the notion of universality~\cite{micchelli2006universal}, in particular {\em cc-universality} \cite{sriperumbudur2011universality}, here explicited in the context of non-negative functions. A set ${\cal F}$ is a {\em universal approximator} for non-negative functions on $\X$ if, for any compact subset ${\cal Z}$ of ${\cal X}$, we have that the set ${\cal F}|_{\cal Z}$ of restrictions on ${\cal Z}$, defined as ${\cal F}|_{\cal Z} = \{f|_{\cal Z}~|~ f \in {\cal F}\}$, is dense in the set $C^+({\cal Z})$ of non-negative continuous functions over ${\cal Z}$ in the maximum norm. In the following theorem we prove the cc-universality of the proposed model
\begin{theorem}
Let $\X$ be a locally compact Hausdorff space, $\hh$ a separable Hilbert space and $\phi:\X \to \hh$ a $cc$-universal feature map. Let $\|\cdot\|_\circ$ be a norm for ${\cal S}(\hh)$ such that $\|\cdot\|_{\star} \trianglerighteq \|\cdot\|_\circ$. Then $\cal{F}^\circ_{\phi}$ is a $cc$-universal approximator for the non-negative functions on $\X$.
\end{theorem}
\begin{proof}
Proving that the proposed model is a cc-universal approximator for non-negative functions, is equivalent to require that given a compact set ${\cal Z} \subseteq {\cal X}$, a non-negative function $g:{\cal Z} \to \R_+$ and $\epsilon > 0$, there exists $f_{A_{g,{\cal Z},\epsilon}} \in \cal{F}^\circ_{\phi}$ such that $\|g - f_{A_{g,{\cal Z},\epsilon}}\|_{C(Z)} \leq \epsilon$.
In particular, let $Q = 2\|g\|^{1/2}_{C({\cal Z})} + \epsilon^{1/2}$, since $\phi$ is $cc$-universal, given ${\cal Z}, g, \epsilon$, there exists $w_{\sqrt{g},{\cal Z},\tfrac{\epsilon}{Q}}$ such that $\|\sqrt{g} - \phi(\cdot)^\top w_{\sqrt{g},{\cal Z},\tfrac{\epsilon}{Q}}\|_{C(Z)} \leq \tfrac{\epsilon}{Q}$.
Define $A_{g,{\cal Z},\epsilon} = w_{\sqrt{g},{\cal Z},\tfrac{\epsilon}{Q}} \otimes w_{\sqrt{g},{\cal Z},\tfrac{\epsilon}{Q}}$. Note that for any $x \in \X$,
\eqal{
f_{A_{g,{\cal Z},\epsilon}}(x) = \phi(x)^\top A_{g,{\cal Z},\epsilon} \phi(x) = \phi(x)^\top \left(w_{\sqrt{g},{\cal Z},\tfrac{\epsilon}{Q}} \otimes w_{\sqrt{g},{\cal Z},\tfrac{\epsilon}{Q}}\right) \phi(x) = (\phi(x)^\top w_{\sqrt{g},{\cal Z},\tfrac{\epsilon}{Q}})^2.
}
Then, by denoting with $h(x) = \sqrt{g(x)} - \phi(x)^\top w_{\sqrt{g},{\cal Z},\tfrac{\epsilon}{Q}}$, we have 
\eqal{
\|g - f_{A_{g,{\cal Z},\epsilon}}\|_{C(Z)} & = \sup_{x \in {\cal Z}} |g(x) - (\phi(x)^\top w_{\sqrt{g},{\cal Z},\tfrac{\epsilon}{Q}})^2|  \\
& =\sup_{x \in {\cal Z}} \left|\left(\sqrt{g(x)} - \phi(x)^\top w_{\sqrt{g},{\cal Z},\tfrac{\epsilon}{Q}}\right)\left(\sqrt{g(x)} + \phi(x)^\top w_{\sqrt{g},{\cal Z},\tfrac{\epsilon}{Q}}\right)\right| \\
& =\sup_{x \in {\cal Z}} |h(x)(2\sqrt{g(x)} - h(x))| \\
& \leq \|h\|_{C({\cal Z})}(2\|\sqrt{g}\|_{C({\cal Z})} + \|h\|_{C({\cal Z})})\\
& \leq \frac{\epsilon}{Q}\left(2\|g\|^{1/2}_{C({\cal Z})} + \frac{\epsilon}{Q}\right) \leq \epsilon.
}
The last step is due to the fact that $\epsilon/Q \leq \sqrt{\epsilon}$, then $2\|g\|^{1/2}_{C({\cal Z})} + \frac{\epsilon}{Q} \leq Q$.
\end{proof}

\subsection{Proof and additional discussion of \cref{thm:model-richer-than-exponential}}
\label{app:proof-thm:model-richer-than-exponential}

In \cref{thm:general-version-banach-algebra}, stated below, we prove that ${\cal E}_\phi \subseteq {\cal F}_\phi$ under the very general assumption that ${\cal G}_\phi$ is a multiplication algebra, i.e.. if ${\cal G}_\phi$ is closed under pointwise product of the functions. In \cref{thm:general-model-richer-than-exponential} we specify this result when ${\cal G}_\phi$ is a Sobolev space, proving that ${\cal E}_\phi \subsetneq {\cal F}^\circ_\phi$. \cref{thm:model-richer-than-exponential} is a direct consequence of the latter theorem.

\paragraph{General result when ${\cal G}_\phi$ is a multiplication algebra.}
First we endow ${\cal G}_\phi$ with a Hilbertian norm. Define $\|\cdot\|_{{\cal G}_\phi}$ as $\|f_w\|_{{\cal G}_\phi} = \|w\|_\hh,$ for any $w \in \hh$. 

\begin{definition}\label{def:mult-algebra}
${\cal G}_\phi$ is a multiplication algebra, when there exists a constant $C$ such that the unit function $u:\X \to \R$ that maps $x \mapsto 1$ for any $x \in \X$ is in ${\cal G}_\phi$ and
\eqal{
\label{eq:banach-algebra-ineq}
\|f \cdot g\|_{{\cal G}_\phi} \leq C\|f\|_{{\cal G}_\phi}\|g\|_{{\cal G}_\phi}, \qquad \forall~ f,g \in {\cal G}_\phi,
}
where we denote by $f\cdot g$ the pointwise multiplication, i.e.,  $(f\cdot g)(x) = f(x)g(x)$ for all $x \in \X$.
\end{definition}

\begin{remark}[Renormalizing the constant]
Note that when ${\cal G}_\phi$ is a multiplication algebra for a constant $C$, it is always possible to define an equivalent norm $\|\cdot\|'_{{\cal G}_\phi}$ as $\|\cdot\|'_{{\cal G}_\phi} = C \|\cdot\|_{{\cal G}_\phi}$ for which ${\cal G}_\phi$ is a multiplication algebra with constant $1$.
\end{remark}

\begin{theorem}[General version when ${\cal G}_\phi$ is an algebra]\label{thm:general-version-banach-algebra}
 Let $\|\cdot\|_{\star} \trianglerighteq \|\cdot\|_\circ$. Let $\X$ be a compact space and $\phi$ be a bounded continuous map such that ${\cal G}_\phi$ is a multiplication algebra, then ${\cal E}_\phi \subseteq {\cal F}^\circ_{\phi}$.
\end{theorem}
\begin{proof}
Let $g \in {\cal E}_\phi$ and take $f \in {\cal G}_\phi$ such that $g(x) = e^{f(x)}$ for all $x \in \X$. 
First we prove that ${\cal E}_\phi \subseteq {\cal F}^\circ_\phi$. With this goal, first we prove that $\sqrt{g} \in {\cal G}_\phi$ and then we construct a rank one positive operator such that $f_{A_g}(x) = g(x)$ for every $x \in \X$. We start noting that, given $f \in {\cal G}_\phi$ and $t \in \N$, $f^t$ defined by $f \cdot f^{t-1}$ for $t \in \N$ satisfies $f^t \in {\cal G}_\phi$, with $\|f^t\|_{{\cal G}_\phi} \leq C^t \|f\|^t_{{\cal G}_\phi}$, by repeated application of the \cref{eq:banach-algebra-ineq}. Moreover note that the function $s = \sum_{t \in \N} \frac{1}{2^t t!} f^t,$
satisfies $s \in {\cal G}_\phi$, indeed
\[\|s\|_{{\cal G}_\phi} \leq \sum_{t \in \N} \frac{1}{2^t t!} \|f^t\|_{{\cal G}_\phi} \leq \sum_{t \in \N} \frac{1}{2^t t!} C^t\|f\|^t_{{\cal G}_\phi} \leq  e^{C \|f\|_{{\cal G}_\phi}/2}.\]
Moreover $s$ satisfies $s(x) = \sqrt{g(x)}$ for all $x \in \X$, indeed for $x \in \X$ we have 
\[s(x) = \phi(x)^\top s = \sum_{t \in \N} \frac{1}{2^t t!} \phi(x)^\top f^t = \sum_{t \in \N} \frac{1}{2^t t!} f^t(x) = e^{f(x)/2} = \sqrt{g(x)}.\]
Now let $A_g = s \otimes s$, we have that $\|A_g\|_\circ \leq \|A_g\|_\star$ by assumption, and $\|A_g\|_\star = \|s\|^2_{{\cal G}_\phi} < \infty$, so the function $f_{A_g} \in {\cal F}^\circ_\phi$ and for any $x \in \X$
\[f_{A_g}(x) = \phi(x)^\top A_g \phi(x) = \phi(x)^\top (s \otimes s) \phi(x) = (\phi(x)^\top s)^2 = g(x).\]
Since for any $g \in {\cal E}_\phi$ there exists $f_{A_g} \in {\cal F}^\circ_\phi$ that is equal to $g$ on their domain of definition, we have that ${\cal E}_\phi \subseteq {\cal F}^\circ_\phi$.
\end{proof}

Now we are going to specialize the result above for Sobolev spaces.

\paragraph{Result for Sobolev spaces}

The result below is based on the general result in \cref{thm:general-version-banach-algebra}, however it is possible to do a proof based only on norm inequalities for compositions of functions in Sobolev space (see for example \cite{brezis2001gagliardo}). While more technical, this second approach would allow to derive also a more quantitative analysis on the norms of the functions in ${\cal G}_\phi$ and ${\cal F}^\circ_\phi$. We will leave this for a longer version of this work. 

\begin{theorem}\label{thm:general-model-richer-than-exponential}
 Let $\|\cdot\|_{\star} \trianglerighteq \|\cdot\|_\circ$. Let $\X \subseteq \R^d$ and $\X$ compact with locally Lipschitz boundary and let ${\cal G}_\phi = W^m_2(\X)$. Let $x_0 \in \X$. Then the following holds: 

(a)  ${\cal E}_\phi \subsetneq {\cal F}^\circ_{\phi}.$ ~~~ (b) The function $f_{x_0}(x) = e^{-\|x- x_0\|^{-2}} \in C^\infty(\X)$ satisfies $f_{x_0} \in {\cal F}^\circ_{\phi}$ and $f_{x_0} \notin {\cal E}_\phi$.
\end{theorem}
\begin{proof}
First we prove that ${\cal E}_\phi \subseteq {\cal F}^\circ_\phi$, via \cref{thm:general-version-banach-algebra}, then we.  To apply this result we need first to prove that ${\cal G}_\phi = W^m_2(\X)$ is a multiplication algebra when $W^m_2(\X)$ is a RKHS as in our case.

\noindent{\bf Step 1, $m > d/2$.} First note that ${\cal G}_\phi$ satisfies $m > d/2$ since $W^m_2(\X)$ admits a representation in terms of a separable Hilbert space $\hh$ and a feature map $\phi:\X  \to \hh$, i.e., it is a {\em reproducing Kernel Hilbert space} and for the same reason $\|\cdot\|_{{\cal G}_\phi}$ is equivalent to $\|\cdot\|_{W^m_2(\X)}$ \cite{wendland2004scattered}.

\noindent{\bf Step 2. ${\cal G}_\phi$ is a multiplication algebra. Applying \cref{thm:general-version-banach-algebra}.} Since ${\cal G}_\phi = W^m_2(\X)$ with $m > d/2$, then it is a multiplication algebra. This result is standard (e.g. see pag. 106 of \cite{adams2003sobolev} for $m \in \N$ and $\X = \R^d$) and we report it in \cref{lm:mult-algebra} in \cref{app:additional-proofs}.  Then we apply \cref{thm:general-version-banach-algebra} obtaining ${\cal E}_\phi \subseteq {\cal F}^\circ_\phi$.

\noindent{\bf Step 3. Proving that $f_{x_0} \in {\cal F}^\circ_\phi$ and not in ${\cal E}_\phi$.} 
By construction the function $v(x) = e^{-1/(2\|x - x_0\|^2)}$ is in $C^\infty(\X)$ and so in $W^m_2(\X)$ for any $m \geq 0$. Since ${\cal G}_\phi = W^m_2(\X)$, then $v \in {\cal G}_\phi$, i.e., there exists $w \in \hh$ such that $w^\top \phi(\cdot) = v(\cdot)$. Define $A_v = w \otimes w$, then 
\[f_{A_v}(x) = \phi(x)^\top A_v \phi(x) = (w^\top \phi(x))^2 = v^2(x) = f_{x_0}(x), \quad \forall x \in \X.\]
Then $f_{x_0} = f_{A_v}$ on $\X$, i.e., $f_{x_0} \in {\cal F}^\circ_\phi$. To conclude note that, $f_{x_0}$ does not belong to ${\cal E}_\phi$, since $x_0 \in \X$ and $f_{x_0}(x_0) = 0$, while for any $g \in {\cal E}_\phi$ we have $\inf_{x \in X} g(x) > 0$. Indeed, we have that for any $f \in {\cal G}_\phi$, $\|f\|_{C(\X)} = \sup_{x \in\X} |f(x)| < \infty$, since ${\cal G}_\phi = W^m_2(\X) \subset C(\X)$. Moreover, given $g \in {\cal G}_\phi$, and denoting by $f \in {\cal G}_\phi$ the function such that $g = e^f$, we have that $\inf_{x \in \X} g(x) \geq e^{-\|f\|_{C(\X)}} > 0$. Finally, since ${\cal E}_\phi \subseteq {\cal F}^\circ_\phi$, but there exists $f_{x_0} \in {\cal F}^\circ_\phi$ and not in ${\cal E}_\phi$, then ${\cal E}_\phi \subsetneq {\cal F}^\circ_\phi$.
\end{proof}

\paragraph{Proof of \cref{thm:model-richer-than-exponential}.} This result is a direct application of \cref{thm:general-model-richer-than-exponential}, since $\X = [-R,R]^d$, with $R \in (0,\infty)$ is a compact set with Lipschitz boundary.


\subsection{Proof of \cref{thm:rademacher}}
\label{app:proof-thm:rademacher}
We recall here the Rademacher complexity and prove \cref{thm:rademacher}. 
This latter theorem is obtained from the following \cref{thm:emp-rademacher} that bounds the {\em empirical Rademacher complexity} introduced below. First we recall that the function class ${\cal F}^\circ_{\phi, L}$ is defined as 
\[{\cal F}^\circ_{\phi, L} = \{f_A ~|~ A \succeq 0, \|A\|_{\circ} \leq L \},\]
for a given norm $\|\cdot\|_\circ$ on operators, a feature map $\phi:\X \to \hh$ and $L > 0$.
Now we define the empirical Rademacher complexity and the Rademacher complexity \cite{bartlett2002rademacher}.
Given $x_1,\dots, x_n \in \X$, the empirical Rademacher complexity for a class ${\cal F}$ of functions mapping $\X$ to $\R$, is defined as
\[\widehat{R}_n({\cal F}) = 2 \mathbb{E} \sup_{f \in {\cal F}} \left|\frac{1}{n} \sum_{i=1}^n \sigma_i f(x_i) \right|,\]
where $\sigma_i$ independent Rademacher random variables, i.e., $\sigma_i = -1$ with probability $1/2$ and $+1$ with probability $1/2$ and the expectation is on $\sigma_1, \dots, \sigma_n$.  Let $\rho$ be a probability distribution on $\X$ and $x_1, \dots, x_n$ sampled independently according to $\rho$. The {\em Rademacher complexity} $R_n({\cal F})$ is defined as
\[ R_n({\cal F}) = \mathbb{E} \widehat{R}_n({\cal F}),\]
where the last expectation is on $x_1,\dots,x_n$. In the following theorem we bound $\widehat{R}_n$.
\begin{theorem}\label{thm:emp-rademacher}
Let $\|\cdot\|_\circ \trianglerighteq \|\cdot\|_{F}$. 
Let $x_1,\dots,x_n \in \X$, $L \geq 0$. \[ \widehat{R}_n({\cal F}^\circ_{\phi,L}) \leq \frac{2L}{n} \sqrt{\sum_{i=1}^n \|\phi(x_i)\|^4}.\]
\end{theorem}
\begin{proof}
Given $f_A \in {\cal F}^\circ_{\phi, L}$, since $\|\cdot\|_\circ$ is stronger or equivalent to Hilbert-Schmidt norm, we have that $\|A\|_{F} \leq \|A\|_\circ \leq L$.  Since $A$ is bounded and $\phi(\cdot) \in \hh$, by linearity of the trace we have $f_A(x) = \phi(x)^\top A \phi(x) = \tr(A ~ \phi(x) \otimes \phi(x))$ for any $x \in \X$. Then, by linearity of the trace 
\eqal{
{\hat R}_n({\cal F}^\circ_{\phi, L}) & = 2 \mathbb{E} \sup_{f \in {\cal F}^\circ_{\phi,L}} \left|\frac{1}{n} \sum_{i=1}^n \sigma_i f(x_i) \right|  = 2 \mathbb{E} \sup_{A \succeq 0, \|A\|_\circ \leq L} \left|\frac{1}{n} \sum_{i=1}^n \sigma_i \phi(x_i)^\top A \phi(x_i) \right| \\
& = 2 \mathbb{E} \sup_{A \succeq 0, \|A\|_\circ \leq L} \left|\frac{1}{n} \sum_{i=1}^n \sigma_i \tr(A ~ (\phi(x_i) \otimes \phi(x_i))) \right| \\
& = 2 \mathbb{E} \sup_{A \succeq 0, \|A\|_\circ \leq L} \left|\tr\left(A~ \left(\frac{1}{n} \sum_{i=1}^n \sigma_i \phi(x_i) \otimes \phi(x_i) \right)\right) \right|
}
Now since $\|\cdot\|_\circ$ is stronger or equivalent to $\|\cdot\|_{F}$ this means that 
$\{A \in {\cal S}(\hh)~|~\|A\|_{\circ} \leq L\} \subseteq \{A \in {\cal S}(\hh)~|~\|A\|_{F} \leq L\}$, then 
\eqal{
2 \mathbb{E} \sup_{A \succeq 0, \|A\|_\circ \leq L} & \left|\tr\left(A~ \left(\frac{1}{n} \sum_{i=1}^n \sigma_i \phi(x_i) \otimes \phi(x_i) \right)\right) \right|   \\
& \leq 2 \mathbb{E} \sup_{A \succeq 0, \|A\|_{F} \leq L} \left|\tr\left(A~ \left(\frac{1}{n} \sum_{i=1}^n \sigma_i \phi(x_i) \otimes \phi(x_i) \right)\right) \right|\\
& \leq 2 \mathbb{E} \sup_{A \succeq 0, \|A\|_F \leq L} \|A\|_{F}~ \Big\|\frac{1}{n} \sum_{i=1}^n \sigma_i \phi(x_i) \otimes \phi(x_i) \Big\|_F\\
& \leq 2 L ~ \mathbb{E} ~ \Big\|\frac{1}{n} \sum_{i=1}^n \sigma_i \phi(x_i) \otimes \phi(x_i) \Big\|_F.
}
To conclude denote by $\zeta_i$ the random variable $\sigma_i \phi(x_i) \otimes \phi(x_i)$. 
Then
\eqals{
\mathbb{E} ~ \Big\|\tfrac{1}{n} \sum_{i=1}^n \sigma_i \phi(x_i) \otimes \phi(x_i) \Big\|^2_F & = \mathbb{E} ~ \|\tfrac{1}{n} \sum_{i=1}^n \zeta_i \|_F \\
& = \mathbb{E} ~ \sqrt{\tr\left(\Big(\tfrac{1}{n} \sum_{i=1}^n \zeta_i\Big)^*\Big(\tfrac{1}{n} \sum_{i=1}^n \zeta_i\Big)\right)} = \mathbb{E} ~ \sqrt{\tr\Big(\tfrac{1}{n^2} \sum_{i,j=1}^n \zeta_i \zeta_j \Big)}.
}
By Jensen inequality, the concavity of the square root, and the linearity of the trace
\eqals{
\mathbb{E} ~ \sqrt{\tr\Big(\tfrac{1}{n^2} \sum_{i,j=1}^n \zeta_i \zeta_j \Big)} \leq
 \sqrt{\mathbb{E} ~\tr\Big(\tfrac{1}{n^2} \sum_{i,j=1}^n \zeta_i \zeta_j \Big)} =  \sqrt{\tfrac{1}{n^2} \sum_{i,j=1}^n ~\tr(\mathbb{E} \zeta_i \zeta_j)}.
}
Now note that for $i \in \{1,\dots,n\}$, we have $\mathbb{E}_{\sigma_i}\zeta_i = 0$, moreover $\mathbb{E} \sigma_i^2 = \|\phi(x_i)\|^2 \phi(x_i) \otimes \phi(x_i)$. Finally, given $x_1,\dots, x_n$, we have that $\zeta_i$ is independent from $\zeta_j$, when $i \neq j$. Then when $i \neq j$ we have
$\tr(\mathbb{E} \zeta_i \zeta_j) = \tr((\mathbb{E}_{\sigma_i}\zeta_i)(\mathbb{E}_{\sigma_j} \zeta_j)) = 0$. When $i=j$ we have $\tr(\mathbb{E} \zeta_i \zeta_j) = \tr(\mathbb{E} \zeta_i^2) = \|\phi(x_i)\|^4$. So
\eqals{
\frac{1}{n^2} \sum_{i,j=1}^n ~\tr(\mathbb{E} \zeta_i \zeta_j) & = \frac{1}{n^2} \sum_{i=1}^n \|\phi(x_i)\|^4.
}
From which we obtain the desired result.
\end{proof}

Now we are ready to bound $R_n$ as follows
\paragraph{Proof \cref{thm:rademacher}.}
The proof is obtained by applying \cref{thm:emp-rademacher} and considering that $\|\phi(x)\|$ is uniformly bounded by $c$ on $\X$. \qed{}


\subsection{Proof of \cref{prop:model-satisfies-P4}}
\label{app:proof-prop:model-satisfies-P4}

See \cref{app:notation} for the basic technical assumptions on $\X$, $\hh$ and $\phi$. In particular $\X$ is Polish and $\phi$ is continuous and uniformly bounded by a constant $c$. 

\bpr[Proof of \cref{prop:model-satisfies-P4}]
In the following we will consider integrability and measurability with respect to a measure $dx$ on $\X$. In particular $p : \cx \rightarrow \R$ is an integrable function on $\cx$ with respect to the measure $dx$. Now define $\Psi(x) = p(x) \phi(x) \phi(x)^{\top}$. We have that $\Psi$ is measurable, since $\phi$ and $p$ are measurable.
Since $p$ is integrable, $p$ is finite almost everywhere, and hence $\Psi(x) = p(x) \phi(x) \phi(x)^{\top}$ is  defined and trace class almost everywhere, and satisfies 
\[\|\Psi(x)\|_\star = |p(x)|~\|\phi(x)\|^2_{\hh} \leq |p(x)| c^2 \text{ almost everywhere}.\]
Since the space of trace class operators is separable, this shows that $\Psi$ is Bochner integrable and thus that the operator $W_p = \int_{x \in \X}{\phi(x) \phi(x)^\top p(x) dx}$ is well defined and trace class, with trace norm bounded by $\kappa^2 \|p\|_{L^1(\X)}$. Moreover, by linearity of the integral, for any $A \in {\cal S}(\hh)$, 
\[\tr(A W_p) = \int_{\cx}{\tr(A \phi(x)\phi(x)^{\top}) p(x)dx} = \int_{\cx}{f_A(x)p(x)dx},\]
where the last equality follows from the definition of $f_A$ and the fact that 
\[\tr(A \phi(x)\phi(x)^\top) = \tr(\phi(x)^{\top} A \phi(x)) = \phi(x)^{\top} A \phi(x) = f_A(x).\]
\epr

\begin{remark}[Extension to more general linear functionals.]
Note that the linearity of the model in $A$ allows to generalize very easily the construction above to any linear functional that we want to apply to the model. This is especially true when the model has a finite dimensional representation as \cref{eq:char-optimal-solution-representer}, i.e. $f_{\bf B} = \sum_{ij=1}^n {\bf B}_{i,j} k(x,x_i) k(x,x_j)$ with ${\bf B} \succeq 0$. In this case, given a linear functional ${\cal L}: C(\X) \to \R$, we have
\[{\cal L}(f_{\bf B}) = \sum_{i,j=1}^n {\bf B}_{i,j} {\cal L}(k(x,x_i) k(x,x_j)) = \tr({\bf B} {\bf W}_{\cal L}),\]
where $({\bf W}_{\cal L})_{i,j} = {\cal L}(k(x,x_i) k(x,x_j))$ for $i,j = 1,\dots,n$.
\end{remark}

\subsection{Proof of \cref{prop:convex-cone}}
\label{app:proof-prop:convex-cone}
In \cref{app:proof-prop:convex-cone} and  \cref{app:proof-thm:representer-convex-cone}, we will use the following notations. 

Let $h,p \in \N$ and  $\hh,\hh_1,\hh_2$ be separable Hilbert spaces.
\begin{itemize}
\item $A = (A_s)_{1 \leq s \leq p} \in {\cal S}(\hh)^p$ will denote a family of self-adjoint operators;
\item Given a feature map $\phi : \X \rightarrow \hh$ and $A = (A_s)_{ 1 \leq s \leq p} \in {\cal S}(\hh)^p$ we will define the function $f_A$ as follows 
\[\forall x \in \X,~ f_A(x) = (f_{A_s}(x))_{ 1 \leq s \leq p}  = \left(\phi(x)^{\top} A_s \phi(x)\right)_{1 \leq s \leq p} \in \R^p,\qquad f_A : \X \rightarrow \R^p\]
\item Given a matrix $C \in \R^{p\times h}$ which corresponds to a list of column vectors $(c^t)_{1 \leq t \leq h} \in (\R^p)^h$, we define 
\[K^C(\hh) := \left\{A = (A_s)_{1 \leq s \leq p} \in {\cal S}(\hh)^p~:~\sum_{s=1}^p{c^{t}_s A_s} \succeq 0,~1 \leq t \leq h\right\}\]
\item For any $A = (A_s)_{1 \leq s \leq p} \in {\cal S}(\hh_1)^p$ and any bounded linear operator $L : \hh_1 \rightarrow \hh_2$, $LAL^*$ will be a slight abuse of notation to denote the family $(L A_s L^*)_{1 \leq s \leq p} \in {\cal S}(\hh_2)^p$. 
\end{itemize}

\begin{proof}[Proof of \cref{prop:convex-cone}] Let $p,h \in \N$ and 
let $C \in \R^{p \times h}$ be a matrix representing the column vectors $c^1...c^h$.\\
Let $\cy$ be the polyhedral cone defined by $C$, i.e. $\cy = \left\{ y \in \R^p~:~C^{\top} y \geq 0\right\}$. \\
Let $\hh$ be a separable Hilbert space and $\phi :\X \rightarrow \hh$ be a fixed feature map. \\
With our previous notations, our goal is to prove that for any $A = (A_s)_{1 \leq s \leq p} \in {\cal S}(\hh)^p$, 
\[A \in K^C(\hh) \implies \forall x \in \X,~ f_A(x) \in \cy.\]
Assume $A \in K^C(\hh)$ and let $x \in \X$. By definition, $f_A(x) = (\phi(x)^{\top}A_s \phi(x))_{1 \leq s \leq p} \in \R^p$. Hence, 
\[C^{\top} f_A(x) = \left(\sum_{s = 1}^p{c^t_s \phi(x)^{\top}A_s \phi(x)}\right)_{1 \leq t \leq h} = \left(\phi(x)^{\top}\left(\sum_{s = 1}^p{c^t_s A_s }\right)\phi(x)\right)_{1 \leq t \leq h} .\]
Since $A \in K^C(\hh)$, for all $1 \leq t \leq h$, it holds $\sum_{s = 1}^p{c^t_s A_s } \succeq 0$. In particular, this implies $\phi(x)^{\top}\sum_{s = 1}^p{c^t_s A_s }  \phi(x) \geq 0$ for all $1 \leq t \leq h$. Hence 
\[C^{\top} f_A(x)  \geq 0 \implies f_A(x) \in \cy.\]
\end{proof}

\subsection{Proof of \cref{thm:representer-convex-cone}}
\label{app:proof-thm:representer-convex-cone}

Using the notations of the previous section, the goal of this section is to solve a problem of the form 

\eqal{\tag{\ref{eq:prototypical-problem-convex-cones}}
\inf_{ A \in K^C(\hh)}{L(f_A(x_1),...,f_A(x_n)) + \Omega(A)},}

for given $p,h \in \N$,~ $C \in \R^{p \times h}$, separable Hilbert space $\hh$, feature map $\phi : \X \rightarrow \hh$, regularizer $\Omega$, loss function $L:\R^n \rightarrow \R \cup {+\infty}$ and $x_1,...,x_n \in \X$.

We start by stating the form of the regularizers we will be using.

\ba \label{ass:omegap}
Let $p \in \N$. For any separable Hilbert space $\hh$ and any $A = (A_s)_{1 \leq s \leq p} \in {\cal S}(\hh)^p$, $\Omega$ is of the form
\[\Omega(A) = \sum_{s = 1}^p{\Omega_s(A_s)},\qquad\Omega_s(A_s) = \lambda_{s,1} \|A_s\|_\star + \frac{\lambda_{s,2}}{2}\|A_s\|_{F}^2,\]
where $\la_{s,1},\la_{s,2} \geq 0$ and $\la_{s,1} + \la_{s,2} > 0$.
\ea

\blm[Properties of $\Omega$]\label{lm:prop_omegap} Let $\Omega$ be a regularizer such that $\Omega$ satisfies \cref{ass:omegap}. Then $\Omega$ satisfies the following properties. 
\begin{enumerate}[label = (\roman*)]
    \item For any separable Hilbert spaces $\hh_1,\hh_2$ and any linear isometry $O : \hh_1 \rightarrow \hh_2$, i.e., such that $O^* O = I_{\hh_1}$, it holds 
    \[ \forall A   \in {\cal S}(\hh_1)^p,~ \Omega(O A O^*) = \Omega(A).\]
    \item For any separable Hilbert space $\hh$ and any orthogonal projection $\Pi \in {\cal S}(\hh_1)$, i.e. satisfying $\Pi = \Pi^*,~ \Pi^2 = \Pi$, it holds
    \[ \forall A \in {\cal S}(\hh)^p,~ \Omega(\Pi A \Pi) \leq \Omega(A).\]
    \item For any finite dimensional Hilbert space $\hh_n$, taking $||A_s||_{op}$ to be the operator norm on $\hh_n$,
    \[\Omega \text{ is continuous},\qquad \Omega(A) \underset{\sup_{s}||A_s||_{op} \rightarrow +\infty}{\longrightarrow} + \infty\]
\end{enumerate}
\elm 

\bpr Note that since
\[\Omega(A) = \sum_{s = 1}^p{\Omega_s(A_s)},\qquad \Omega_s(A_s) = \lambda_{s,1} \|A_s\|_\star + \frac{\lambda_{s,2}}{2}\|A_s\|_{F}^2,\]
where $\la_{s,1},\la_{s,2} \geq 0$ and $\la_{s,1} + \la_{s,2} > 0$,
it is actually sufficient to prove the following result.\\\
Let $\lambda_1,\lambda_2 \geq 0$ and assume $\lambda_1+\lambda_2 > 0$. Let for any $A \in {\cal S}(\hh)$, $\Omega(A) = \lambda_1 \|A\|_\star + \frac{\lambda_2}{2}\|A\|_F^2$. Then the following hold:
\begin{enumerate}[label = (\roman*)]
    \item For any separable Hilbert spaces $\hh_1,\hh_2$ and any linear isometry $O : \hh_1 \rightarrow \hh_2$, i.e., such that $O^* O = I_{\hh_1}$, it holds 
    \[ \forall A   \in {\cal S}(\hh_1)^p,~ \Omega(O A O^*) = \Omega(A).\]
    \item For any separable Hilbert space $\hh$ and any orthogonal projection $\Pi \in {\cal S}(\hh_1)$, i.e. satisfying $\Pi = \Pi^*,~ \Pi^2 = \Pi$, it holds
    \[ \forall A \in {\cal S}(\hh)^p,~ \Omega(\Pi A \Pi) \leq \Omega(A).\]
    \item For any finite dimensional Hilbert space $\hh_n$, 
    \[\Omega \text{ is continuous},\qquad \Omega(A) \underset{\|A\|_{op} \rightarrow +\infty}{\longrightarrow} + \infty,\]
    where we denote by $\|\cdot\|_{op}$ the operatorial norm.
\end{enumerate}
\paragraph{1.} (i)
has already been proven in \cref{lm:main_properties_ass}. 
\paragraph{2.} Let us prove (ii).
Let $\hh$ be a separable Hilbert space, $\Pi$ an orthogonal projection on $\hh$ and $A \in {\cal S}(\hh)$.\\

Using the fact that $\|B\|_{\star} = \sup_{\|C\|_{op} \leq 1}{\tr(BC)}$, where $\|C\|_{op}$ denotes the operator norm on ${\cal S}(\hh)$, we have  by property of the trace
\[\|\Pi A \Pi\|_{\star}  = \sup_{\|C\|_{op} \leq 1}{\tr( \Pi A\Pi C )} = \sup_{\|C\|_{op} \leq 1}{\tr(  A(\Pi C  \Pi))}.\]
Now since $\|\Pi C \Pi\|_{op} \leq \|C\|_{op} \leq 1$, it holds $\sup_{\|C\|_{op} \leq 1}{\tr(  A(\Pi C  \Pi))} \leq \sup_{\|C\|_{op} \leq 1}{\tr(  A  C  )} = \|A\|_{\star}$. Thus:
\[\|\Pi A \Pi\|_{\star} \leq \|A\|_\star.\]

Moreover, since $\Pi \preceq I$, it holds $ \Pi A\Pi A \Pi \preceq \Pi A^2 \Pi$. Hence,
\eqals{
\| \Pi A \Pi \|_F^2 = \tr(\Pi A \Pi \Pi A \Pi) \leq \tr(\Pi A^2 \Pi)  
}
Now using the fact that $\tr(\Pi A^2 \Pi) = \tr(A \Pi A)$, we can once again use the fact that $\Pi \preceq I$ to show that $A \Pi A \preceq A^2$ and hence $\tr(A \Pi A) \leq \tr(A^2)$. Putting things together, we have shown 
\[\tr(\Pi A \Pi \Pi A \Pi) \leq \tr(A^2) \implies \|\Pi A \Pi\|_F^2 \leq \|A\|_F^2.\]
Thus, by summing the inequalities, $\Omega(\Pi A \Pi) \leq \Omega(A)$. 

\paragraph{3.} The proof of (iii) is straightforward. The continuity of $\Omega$ comes from the fact that it is a norm on any finite dimensional Hilbert space. Moreover, since $\lambda_1 > 0$ or  $\lambda_2 > 0$ , $\Omega$ goes to infinity. 
\epr

\br 
As in the previous sections, the fact that $\Omega$ satisfies these three properties is actually sufficient to complete the proof.
\er

Recall that $\hh_n$ is the finite dimensional subset of $\hh$ spanned by the $\phi(x_i)$. Recall that $\Pi_n$ is the orthogonal projection on $\hh_n$, i.e.
    \[\Pi_n \in {\cal S}(\hh),~\Pi_n^2 = \Pi_n,~ \range(\Pi_n) =\hh_n.\]

 Define $K^C_{n}(\hh)$ to be the following subspace of $K^C(\hh)$ : 
 \[K^C_{n}(\hh) :=  \left\{\Pi_n A \Pi_n~:~ A  \in  K^C(\hh)\right\}.\]
 It is straightforward to show that $K^C_{n}(\hh) \subset K^C(\hh)$ since projecting left and right preserves the linear inequalities.

\bp \label{prp:prelim_thm6}
Let $L$ be a lower semi-continuous function which is bounded below, and assume $\Omega$ satisfies \cref{ass:omegap}. Then \cref{eq:prototypical-problem-convex-cones}  has a solution $A^*$ which is in $K^C_n(\hh)$. 
\ep

\bpr In this proof, denote with $J$ the function defined by 
\[\forall A \in {\cal S}(\hh)^p,~ J(A) := L(f_A(x_1),...,f_A(x_n)) + \Omega(A).\]
Our goal is to prove that the problem $\inf_{A \in K^C(\hh)}{J(A)}$ has a solution which is in $K_n^C(\hh)$, i.e. of the form  $\Pi_n A \Pi_n$ for some $A \in K_n^C(\hh)$.\\

\paragraph{1.}

Let us start by fixing $A \in K^C(\hh)$.\\
First note that since $\Pi_n$ is the orthogonal projection on $\lspan(\phi(x_i))_{1\leq i \leq n}$, in particular $\Pi_n \phi(x_i) = \phi(x_i)$ for all $1 \leq i \leq n$. Thus, for any $1 \leq i \leq n$, 
    \[f_A(x_i) = (\phi(x_i)^{\top} A_s \phi(x_i))_{ 1 \leq s \leq p} = (\phi(x_i)^{\top}\Pi_n A_s \Pi_n \phi(x_i))_{1\leq s \leq p} = f_{\Pi_n A \Pi_n}(x_i).\]
    Here, the first and last equalities come from the definition of $f_A$ and $f_{\Pi_n A \Pi_n}$. Thus, 
    \[J(A) =L(f_{\Pi_n A \Pi_n}(x_1),...,f_{\Pi_n A \Pi_n}(x_n)) + \Omega(A).\]
Now since $\Omega$ satisfies \cref{ass:omegap}, by the second point of \cref{lm:prop_omegap}, it holds $\Omega(\Pi_n A \Pi_n) \leq \Omega(A)$, hence 
    \[J(\Pi_n A \Pi_n) \leq J(A).\]

This last inequality combined with the fact that $K_n^C(\hh) = \left\{\Pi_n A \Pi_n~:~ A  \in  K^C(\hh)\right\} \subset K^C(\hh)$ show that 
\eqals{
\textstyle  \inf_{A \in K_n^C(\hh)} J(A) =  \inf_{K^C(\hh)} J(A).
}
\paragraph{2.}
Let us now show that $\inf_{A \in K^C_n(\hh)} J(A) $ has a solution. Let us exclude the case where $J  = +\infty$, in which case $A = 0$ can be taken to be a solution.\\\

Let $V_n$ be the injection $V_n : \hh_n \hookrightarrow \hh$. Note that $V_n V_n^* = \Pi_n$ and $V_n^* V_n = I_{\hh_n}$. These simple facts easily show that 

\[K_n^C(\hh) = V_n K^C(\hh_n) V_n^* = \left\{V_n \tilde{A} V_n^*~:~ \tilde{A} \in K_n^C(\hh_n)\right\}.\]

Thus, our goal is to show that
$\inf_{\tilde{A} \in K^C_n(\hh_n)} J(V_n A V_n^*)$ has a solution.\\

By the first point of \cref{lm:prop_omegap}, since $V_n^* V_n = I_{\hh_n}$, it holds 
\[\forall \tilde{A} \in {\cal S}(\hh_n),~\Omega(V_n \tilde{A} V^*_n) = \Omega(\tilde{A}) \implies J(V_n \tilde{A} V_n^*) = L(f_{V_n \tilde{A} V_n^*}(x_1),...,f_{V_n \tilde{A} V_n^*}(x_n)) + \Omega(\tilde{A}).\]

Let $\tilde{A}_0 \in K^C(\hh_n)$ be a point such that $J_0 :=J(V_n\tilde{A}_0 V_n^*) < \infty$. Let $c_0$ be a lower bound for $L$. By the third point of \cref{lm:prop_omegap}, there exists a radius $R_0$ such that for all $\tilde{A} \in {\cal S}(\hh_n)$,
\[ \|\tilde{A}\|_F > R_0  \implies \Omega(\tilde{A}) > J_0 - c_0.\]
Since $c_0$ is a lower bound for $L$, this implies
\eqals{
\textstyle  \inf_{\tilde{A} \in K^C(\hh_n)}J(V_n \tilde{A} V^*_n) =  \inf_{\tilde{A} \in K^C(\hh_n),~\|\tilde{A}\|_F \leq R_0}J(V_n \tilde{A} V^*_n).
}
Now since $L$ is lower semi-continuous, $\Omega$ is continuous by the last point of \cref{lm:prop_omegap}, and $\tilde{A} \mapsto (f_{V_n \tilde{A} V_n^*}(x_i))_{1 \leq i \leq n}$ is linear hence continuous, the mapping $A \mapsto  J(V_n \tilde{A} V_n^*)$ is lower semi-continuous. Hence, it reaches its minimum on any non empty compact set. Since $\hh_n$ is finite dimensional, the set $\left\{\tilde{A} \in K^C(\hh_n)~:~\|\tilde{A}\|_F \leq R_0\right\}$ is compact (closed and bounded) and non empty (it contains $\tilde{A}_0$), and hence there exists $\tilde{A}_* \in K^C_n(\hh)$ such that $J(V_n \tilde{A}_* V_n^*) =   \inf_{\tilde{A} \in K^C(\hh_n),~\|\tilde{A}\|_F \leq R_0}J(V_n \tilde{A} V^*_n)$. 
Going back up the previous equalities, this shows that $A_* := V_n \tilde{A}_* V_n^* \in K_n^C(\hh)$ and $J(A_*) = \inf_{A \succeq 0}J(A)$.

\epr

\blm \label{lm:thm6}
The set $K^C_n(\hh)$ can be represented in the following way
\[K^C_n(\hh) = \left\{\left(S_n^* {\bf B}_s S_n\right)_{1 \leq s \leq p} \in {\cal S}(\hh)^p~:~ {\bf B} = ({\bf B}_s)_{1 \leq s \leq p} \in K^C(\R^n) \right\}\]
In particular, for any $A \in K^C_n(\hh)$, there exists $p$ symmetric matrices $\mathbf{B} = ({\bf B}_s)_{1 \leq s \leq p} \in K^C(\R^n)$ such that 
\[\forall x \in \X,~ f_A(x) = \left(\sum_{1 \leq i,j \leq n}{[\mathbf{B}_s]_{i,j} k(x_i,x)k(x_j,x)}\right)_{1 \leq s \leq p}.\]
\elm 

\bpr The proof is exactly analoguous to the proof of \cref{lm:thm1}.
\epr 

We will prove the following \cref{thm:6bis} which statement is that of  \cref{thm:representer-convex-cone} with more precise assumptions.

\bt \label{thm:6bis} Let $L$ be lower semi-continuous and bounded below, and $\Omega$ satisfying \cref{ass:omegap}. Then \cref{eq:problem-for-representer} has a solution of the form
\eqals{
\textstyle f_{*}(x) = \left(\sum_{i,j=1}^n [\mathbf{B}_s]_{i,j} k(x,x_i)k(x,x_j)\right)_{1 \leq s \leq p}, 
}
for some family  $\mathbf{B} = ({\bf B}_s)_{1 \leq s \leq p} \in K^C(\R^n)$. Moreover, if $L$ is convex, this solution is unique. 
\et 

\begin{proof}[Proof of \cref{thm:6bis}]
The proof is completely analoguous to that of \cref{thm:1bis}, combining \cref{lm:thm6} and \cref{prp:prelim_thm6}.
\end{proof}

\section{Additional proofs}
\label{app:additional-proofs}

\begin{lemma}
\label{lm:mult-algebra}
Let $\X \subset \R^d$, $d \in \N$, be a compact set with Lipschitz boundary. Let $m > d/2$. Then $W^m_2(\X)$ is a multiplication algebra (see \cref{def:mult-algebra}).
\end{lemma}
\begin{proof}
When $m \in \N$ and $m > d/2$, then $W^m_2(\R^d)$ is a multiplication algebra \cite{adams2003sobolev}. When $m \notin \N$, by Eq.~2.69 pag. 138 of \cite{triebel2006function} we have that $F^m_{2,2}(\R^d)$ is a multiplication algebra when $m > d/2$, where $F^m_{2,2}$ is the Triebel-Lizorkin space of smoothness $m$ and order $2,2$ and corresponds to $W^m_2(\R^d)$, i.e., $F^m_{2,2}(\R^d) = W^m_2(\R^d)$ \cite{triebel2006function}. 

So far we have that $m > d/2$ implies that $W^m_2(\R^d)$ is a multiplication algebra, now we extend this result to $W^m_2(\X)$. Note that since $\X$ is compact and with Lipschitz boundary, for any $f \in W^m_2(\X)$ there exists an {\em extension} $\tilde{f} \in W^m_2(\R^d)$ such that $\tilde{f}|_\X = f$ and $\|\tilde{f}\|_{W^m_2(\R^d)} \leq C_1 \|f\|_{W^m_2(\X)}$ with $C_1$ depending only on $m,d,\X$ (see Thm.~5.24 pag. 154 for $m \in \N$ and 7.69 when $m \notin \N$ pag 256 \cite{adams2003sobolev}). Then, since for any $f:\R^d \to \R$, by construction we have $\|f|_X\|_{W^m(\X)} \leq \|f\|_{W^m(\R^d)}$ \cite{adams2003sobolev}. Then, for any $f, g \in W^m_2(\X)$, denoting by $\tilde{f}, \tilde{g}$ the extensions of $f, g$, we have
\eqal{
\|f \cdot g\|_{W^m_2(\X)} &= \|\tilde{f}|_\X \cdot \tilde{g}|_\X\|_{W^m_2(\X)} \leq \|\tilde{f} \cdot \tilde{g}\|_{W^m_2(\R^d)} \\
& \leq  C\|\tilde{f}\|_{W^m_2(\R^d)} \|\tilde{g}\|_{W^m_2(\R^d)} \leq C C_1^2 \|f\|_{W^m_2(\X)} \|g\|_{W^m_2(\X)}.
}
To conclude $u:\X \to \R$ that maps $x \mapsto 1$ has bounded norm corresponding to $\|u\|^2_{W^m_2(\X)} = \int_\X dx$.
So $W^m_2(\X)$ when $m > d/2$ and $\X$ is compact with Lipschitz boundary is a multiplication algebra.
\end{proof}

\section{Additional details on the other models}
\label{app:details-other-models}

Recall that the goal is to solve a problem of the form \cref{eq:prototypical-problem}, i.e. 
\[\min_{f \in {\cal F}} L(f(x_1),...,f(x_n)) + \Omega(f).\]

In this section, $\phi : \X \rightarrow \hh$ will always denote a feature map, $k : \X \times \X \rightarrow \R$ a positive semi definite kernel on $\X$ ($k(x,x^{\prime}) = \phi(x)^{\top} \phi(x^{\prime})$ if $k$ is the positive semi-definite kernel associated to $\phi$). Given a kernel $k$, ${\bf K} \in \R^{n \times n}$ will always denote the positive semi-definite kernel matrix with coefficients ${\bf K}_{i,j} = k(x_i,x_j),~1 \leq i,j \leq n$.

\noindent{\bf Generalized linear models (GLM).} Consider generalized linear models of the form, 
$ f_w(x) = \psi(w^\top \phi(x))$. Assume the regularizer is of the form $\Omega(f_w) = \frac{\lambda}{2} \|w\|^2$. Using the representer theorem \cite{cheney2009course}, any solution to \cref{eq:prototypical-problem} is of the form $w = \sum_{i=1}^n{\alpha_i \phi(x_i)}$ and thus \cref{eq:prototypical-problem} becomes the following finite dimensional problem in $\alpha$:

\eqal{\label{protoglm}
\min_{\alpha \in \R^n}{L(\psi({\bf K}\alpha))}
 + \frac{\lambda}{2}\alpha^\top {\bf K}\alpha.
}

 In the case where one wishes to learn a density function with respect to a basis measure $\nu$, a common choice of model is functions of the form 
 \[
 p_{\alpha}(x) = \frac{\exp(g(x))}{\int_{\tilde{x} \in \X}{\exp(g(\tilde{x}))d\nu(\tilde{x})}},\qquad g(x) = \sum_{i=1}^n{\alpha_i k(x_i,x)}.
 \]

 where $k$ is a positive semi-definite kernel on $\X$. The prototypical problem one solves to find the best $p_\alpha$ is 
 
 \eqal{\label{protodensity}
 \min_{\alpha \in \R^n}{L(p_\alpha(x_1),...,p_{\alpha}(x_n))} + \frac{\lambda}{2}\alpha^\top {\bf K}\alpha.
 }

 In the specific case where the loss function is the negative log likelihood $L(z_1,...,z_n) = \frac{1}{n}\sum_{i=1}^n{-\log(z_i)}$, it can be shown that \cref{protodensity} is convex in $\alpha$.\\\

 In practice, we solve \cref{protoglm} by applying standard gradient descent with restarts, as the problem is non convex.
 
 To solve \cref{protodensity}, since the problem is convex, the algorithm is guaranteed to converge. However, since we can only estimate the quantity $\int_{\tilde{x} \in \X}{\exp(g(\tilde{x}))d\nu(\tilde{x})}$; we do so by taking a measure $\nu$ from which we can sample. However, this becomes intractable as the dimension grows, as the experiments on density estimation will put into light.

\noindent{\bf Non-negative coefficients models (NCM).} Recall the definition of an NCM. It represent non-negative functions as
$f_{\alpha}(x) = \sum_{i=1}^n \alpha_i k(x,x_i)$, with $\alpha_1,\dots \alpha_n \geq 0$,
given a kernel $k(x,x') \geq 0$ for any $x,x' \in \X$. In this case, the prototypical problem is of the form : 

\eqal{\label{proto_NW}
\min_{\alpha \geq 0}{L({\bf K}\alpha) + \frac{\lambda}{2} \alpha^{\top} {\bf K} \alpha.
}}

If we are performing density estimation with respect to the measure $\nu$, one wishes to impose  $\int_{\X}{f_{\alpha}(x)d\nu(x)} = 1$, which can be seen as an affine constraint over $\alpha$, since

\[\int_{\X}{f_{\alpha}(x)d\nu(x)} = {\bf u}^\top \alpha,\qquad {\bf u} = \left(\int_{\X}{k(x,x_i)d\nu(x)}\right)_{1 \leq i \leq n} \in \R^n.\]
In this case, the prototypical problem will be of the form

\eqal{\label{proto_NWDE}
\min_{\substack{\alpha \geq 0 \\ {\bf u}^{\top} \alpha = 1}}{L({\bf K}\alpha) + \frac{\lambda}{2} \alpha^{\top} {\bf K} \alpha.
}}

If $L$ is a convex smooth function, both problems \cref{proto_NW} and \cref{proto_NWDE} can be solved using projected gradient descent, since the projections on the set $\alpha \geq 0$ and the simplex
 $$\set{\alpha \in \R^n~:~ \alpha \geq 0,~  {\bf u}^{\top} \alpha = 1}$$ can be computed in closed form.

In the main paper, we mention that NCM models do not satisfy {\bf P2} i.e. that they cannot approximate any function arbitrarily well. We implement  \cref{ex:NCM} in the following way. Let $g (x) = e^{-\|x\|^2/2}$.  Take $k(x,x^{\prime}) = e^{-\|x-x^{\prime}\|^2}$, $n$ points $(x_1,...,x_n)$ taken uniformly in the interval $[-5,5]$. To find the function $f_\alpha$ which best approximates $g$, we perform least squares regression, i.e. solve the prototypical problem \cref{proto_NW} with the square loss function 
\[L(y) = \frac{1}{2n}\sum_{i=1}^n{|y_i - g(x_i)|^2}.\] 
We perform cross validation to select the value of $\lambda$ for each value of $n$. In \cref{fig:NCM}, we show the obtained function $f_{\alpha}$ for $n = 100,1000,10000$. This clearly illustrates that with this model, we cannot approximate $g$ in a good way, no matter how many points $n$ we have.

\begin{figure}[t!]
    \centering
    \includegraphics[width=0.33\textwidth]{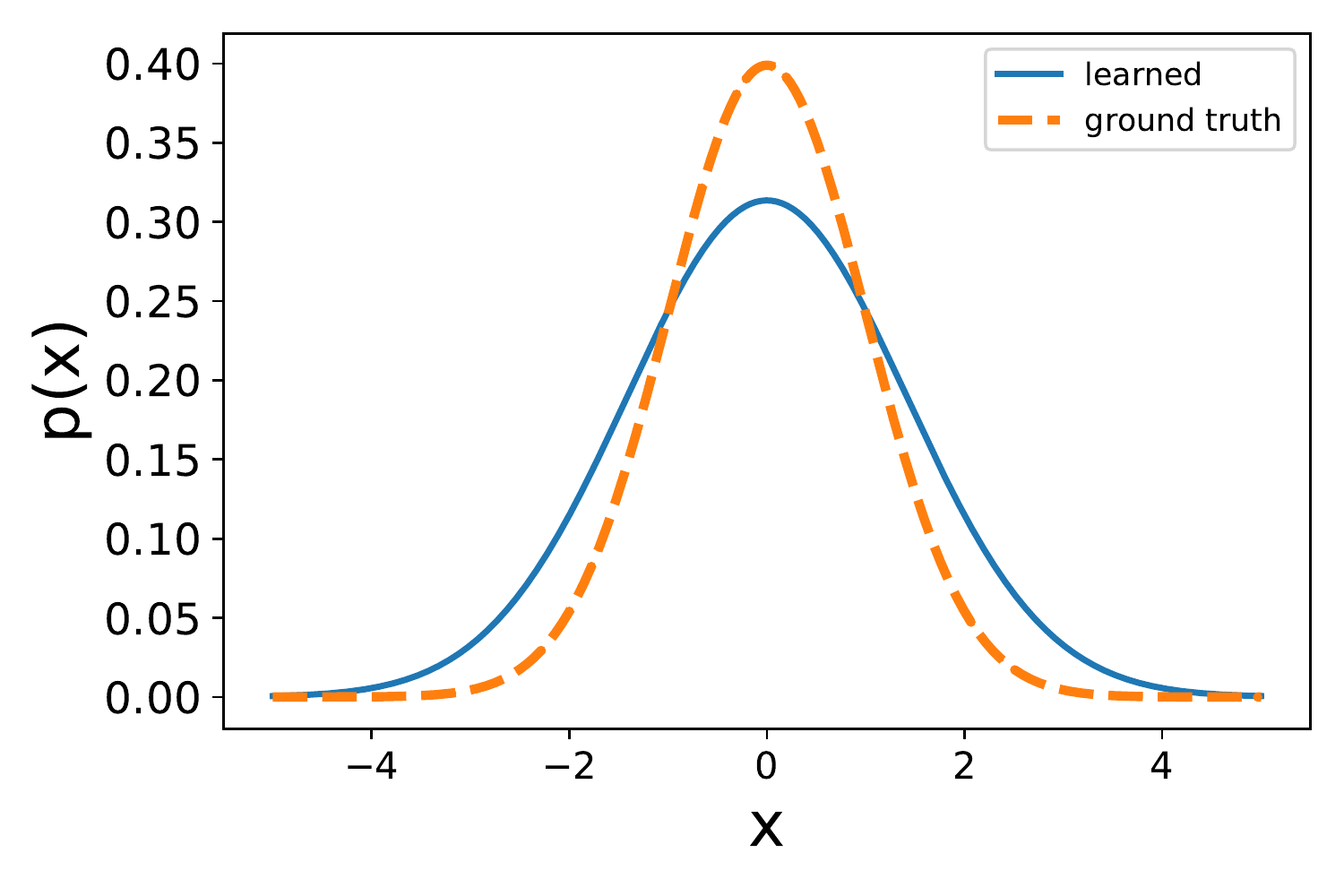}~
    \includegraphics[width=0.33\textwidth]{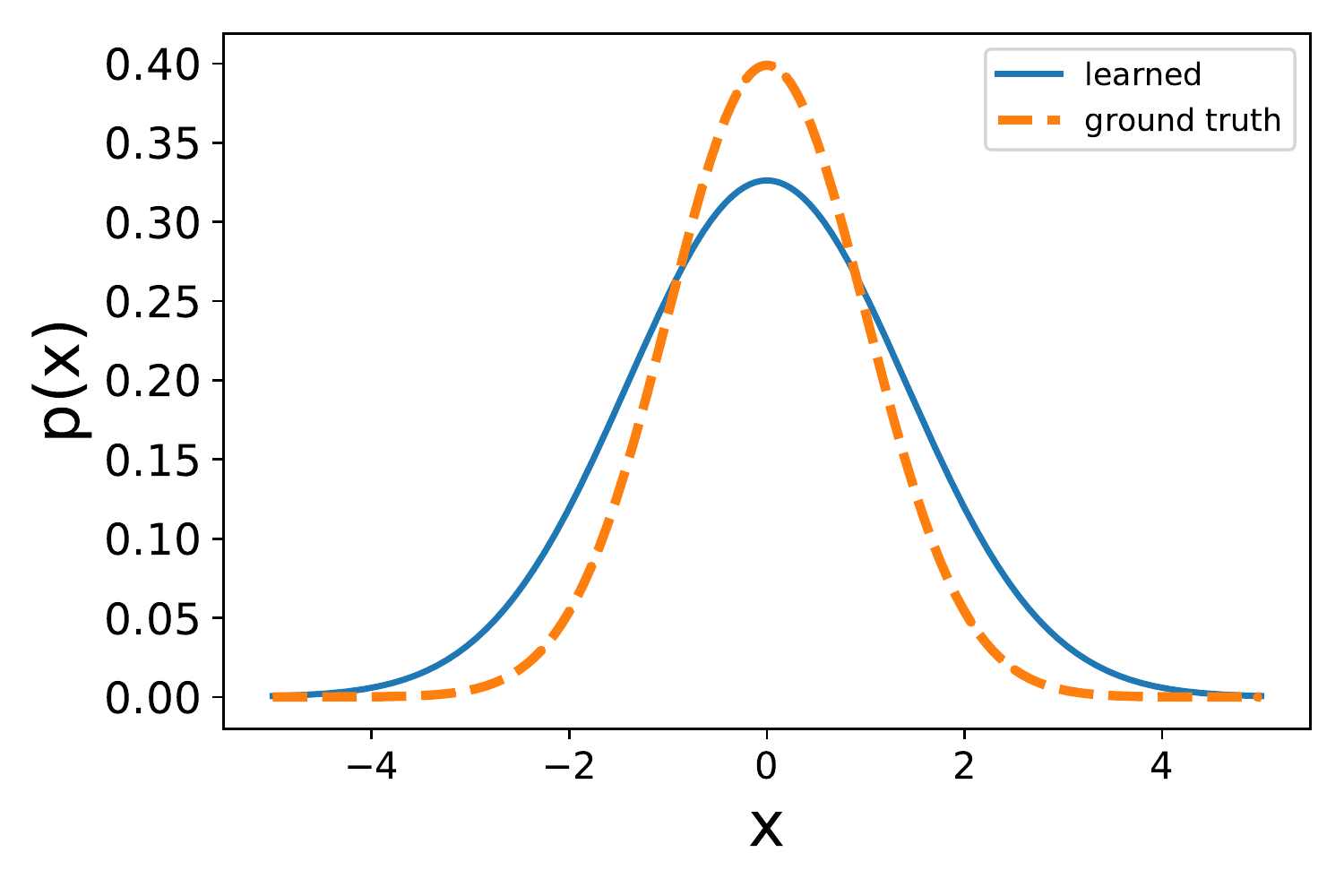}~
    \includegraphics[width=0.33\textwidth]{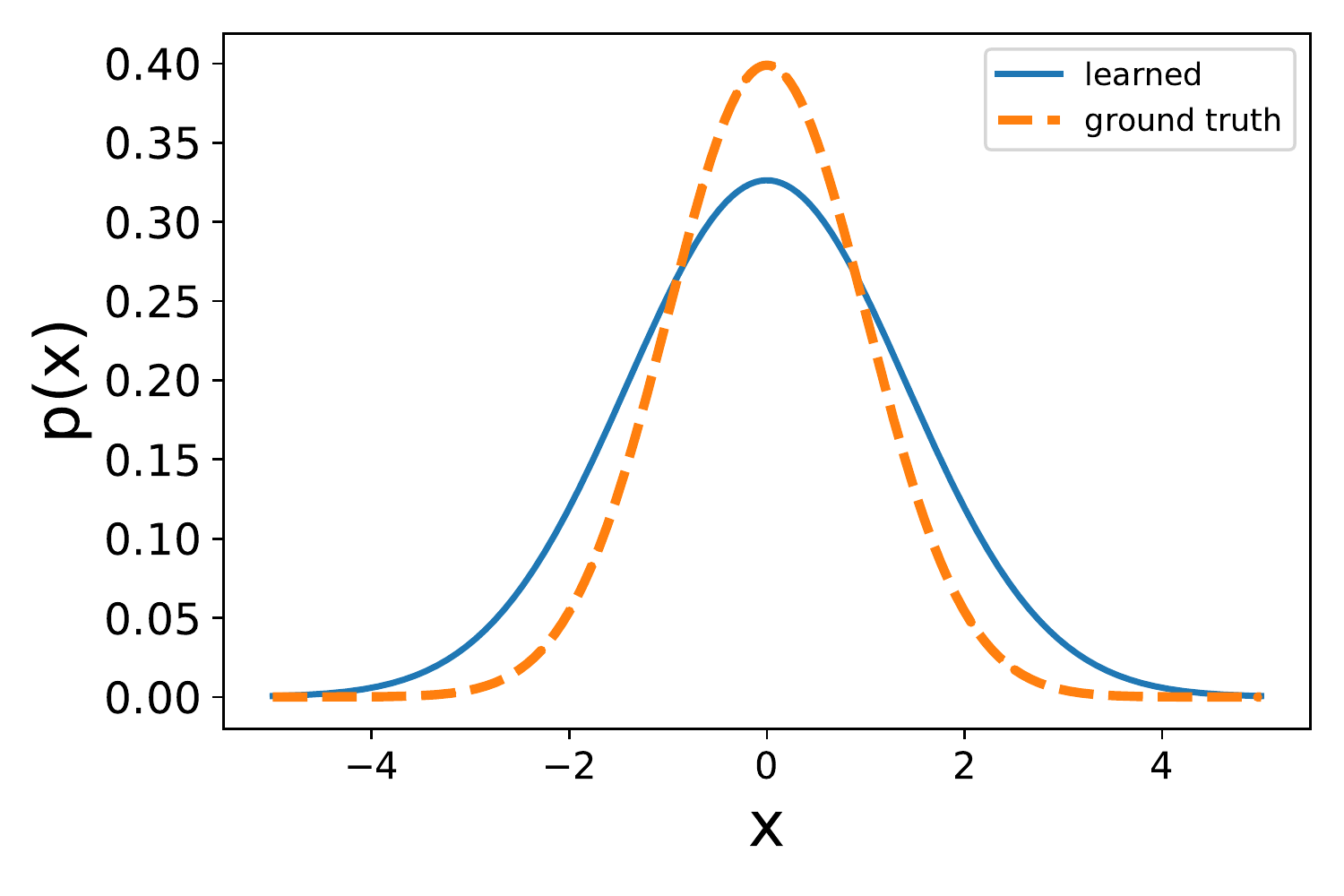}
    \caption{%
    Best approximation of $g$ using NCM with \emph{(left)} $n = 100$  \emph{(center)} $n = 1000$  \emph{(right)} $n = 10000$ points.}
    \label{fig:NCM}
\end{figure}

\noindent{\bf Partially non-negative linear models (PNM).} Consider partially non negative models of the form $f_w(x) = w^\top \phi(x)$, with $w \in \{ w \in \hh ~|~ w^\top \phi(x_1) \geq 0,\dots,w^\top \phi(x_n) \geq 0 \}$ (that is we impose $f_w(x_i) \geq 0$). Take $\Omega$ to be of the form $\frac{\lambda}{2}\|w\|^2$ in \cref{eq:prototypical-problem}. Using the representer theorem in \cite{cheney2009course}, we can show that there is a solution of this problem of the form $f_\alpha = \sum_{i=1}^n{\alpha_i k(x,x_i)}$, leading to the following optimization problem in $\alpha$ to recover the optimal solution: 
\eqal{
\label{proto_PNM}
\min_{{\bf K}\alpha \geq 0}{L({\bf K}\alpha) + \frac{\lambda}{2}\alpha^\top {\bf K} \alpha}
}
If we want to impose that the resulting $f_{\alpha}$ sums to one for a given measure $\nu$ on $\X$, we proceed as in \cref{proto_NWDE} and solve
\eqal{\label{proto_PNMDE}
\min_{\substack{{\bf K}\alpha \geq 0 \\ {\bf u}^{\top} \alpha = 1}}{L({\bf K}\alpha) + \frac{\lambda}{2} \alpha^{\top} {\bf K} \alpha.
}}
However, there is no guarantee that the resulting $f_\alpha$ will be a density, as will be made clear in the next section on density estimation.

In the experiments, we solve \cref{proto_PNM} and \cref{proto_PNMDE} in the following way. We first compute a cholesky factor of ${\bf K}$ : $  {\bf K}= {\bf V}^{\top} {\bf V}$. Changing variables by setting ${\bf V} \beta = \alpha$, the objective functions become strongly convex in $\beta$. We then compute the dual of these problems and apply a proximal algorithm like FISTA, since the proximal operator of $L$ is always known in our experiments.

\section{Additional details on the experiments}
\label{app:details-experiments}

In this section, we provide additional details on the experiments. The code will be available online. Recall that we consider four different models for functions with non-negative outputs : GLM, PNM, NCM and our model.

\paragraph{Kernels.}

All the models we consider depend on certain positive semi definite kernels $k$. In all the experiments, we have taken the kernels to be Gaussian kernels with width $\sigma$:
\[\forall x,x^{\prime} \in \R^d,~ k(x,x^{\prime}) = \exp\left(-\frac{\|x-x^{\prime}\|^2}{2\sigma^2}\right).\]

\paragraph{Regularizers.}

For GLM, PNM and NCM, the regularizer for the underlying linear models are always of the form $\frac{\lambda}{2}\|w\|^2$ where $w$ is the parameter of the linear model, which translates to $\frac{\lambda}{2}\alpha^\top {\bf K} \alpha$  where the $\alpha$ are the coefficients of the finite dimensional representation. 
For our model, we always take the regularizer to be of the form $\lambda \left(\|A\|_{\star} + 0.05 \|A\|_F^2\right)$.

\paragraph{Parameter selection.}

In all experiments except for the one on density estimation in the main paper (in which we fix $\sigma = 1$ and select $\lambda$), we select the parameters $\sigma$ of the kernels involved as well as the parameters $\lambda$ for the regularizers using $K$ fold cross validation with $K = 7$. This means that once the data set has been generated, we randomly divide it into two sets : the training set containing $70\%$ of the data and the test set containing $30\%$ of the data. We then train our model for the given $\sigma,\lambda$ and report the performance on the test set. We repeat this operation $K = 7$ times and consider the mean performance on the test set to be a good indicator of the performance of our model for a given set of parameters. We then select the best parameters by doing a grid search. The code for this cross-validation will be available online.

\paragraph{Formulations and algorithms.}

The formulations of our three problems : density estimation, regression with Gaussian heteroscedastic errors, and multiple quantile regression, have been expressed in the main paper in a generic way involving functions with unconstrained outputs, and functions with outputs constrained to be non negative and sometimes summing to one. We always model functions with unconstrained outputs with a linear model with gaussian kernel, and model the functions with constrained outputs with the four models for non-negative functions we consider: ours, PNM, GLM and NCM.

In practice, we implement the methods PNM, GLM and NCM as explained in 
\cref{app:details-other-models}. In particular, we use FISTA for PNM, and our model, dualizing the equality constraints for density estimation.
This relies on the fact that the proximal operators of the log likelihood, the objective function for heteroscedastic regression as well as the pinball loss can be computed in closed form, and that the regularization is smooth in the right coordinates.

\paragraph{Details on the experiments of the main text.}

Here, we add a few precisions on the toy distributions we have used to sample data and the number of sampled used when not specified in the main text.

\begin{itemize}
    \item For heteroscedastic regression, the data was generated as the toy data in section 5 of \cite{le2005heteroscedastic}, with $n = 80$ points.
    \item For quantile regression, the data points $(x_i,y_i)$ were generated according to the following distribution for $(X,Y)$ : $X \sim \frac{1}{2}U(0,1/3) + \frac{1}{2}U(2/3,1) $ and $Y|x \sim {\cal N}(0,\sigma(x))$ where 
    \[\sigma(x) = \begin{cases}-x + 1/3 & \text{ for } 0 \leq x \leq 1/3 \\
    x-2/3 &\text{ for } 2/3 \leq x \leq 1\\
    0 & \text{ otherwise }.
    \end{cases}.\]
    Here, $U$ stands for the uniform distribution.  Moreover, in order to perform the experiments in the main paper, we have used $500$ sample points.
\end{itemize}

\paragraph{Density estimation in dimension $10$ with $n = 1000$.}

In this paragraph, we consider the following experiment. Let $d = 10$, $X \in \R^d$ be a random variable distributed as a mixture of Gaussians :
\[X \sim \frac{1}{2}{\cal N}(-2e_1,1/\sqrt{2\pi} I_d) + \frac{1}{2}{\cal N}(2e_1,1/\sqrt{2\pi} I_d)\]
where $e_1$ is the first vector of the canonical basis of $\R^d$.

Let $n = 1000$ and let $(x_1,...,x_n)$ be $n$ iid samples of $X$.  We perform the four different methods, cross validating both the regularization parameter $\la$ and the kernel parameter $\sigma$ at each time. We learn the density in the form 
$$p(x) = f(x)\nu(x),\qquad \nu \text{ is the density associated with } {\cal{N}}(0,5 I_d).$$

We then use our models for densities to compute the best $f$ in its class using the negative log-likelihood as a loss function. It is crucial that we can sample from $\nu$ in order to approximate the integral in the case of GLMs.

In order to visualize the results of the different algorithms in \cref{fig:experimentdensity}, we compute the learnt distribution $p$, and then sample randomly $n_0 = 500$ points from a uniform distribution on the box centered at $0$ and of width $5$ in order to explore regions where the density is close to zero, $n_0$ points sampled from the true distribution of the data, in order to explore points where the density is representative, and $n_0$ points on the line $[-4,4] \times \set{0}^{d-1}$ where the density is at its highest. We then project onto the first coordinate, i.e. given a point $x = (x_i)_{1 \leq i \leq d}$ and the associated predicted density $p(x)$, we plot the point $(x_1, p(x))$. Note that for readability, we have used the same scale for our model and the PNM, and a smaller scale for  the two others since the learnt density is much flatter.  

\begin{figure}[t!]
    \centering
    {\tiny ${}~~~\quad{}$ PNM ${}~\qquad\qquad\qquad\qquad\qquad\qquad\qquad{}$ Our model ${}\qquad\qquad\qquad\qquad\qquad\qquad\quad~~~{}$ NCM ${}\qquad\qquad\qquad\qquad\qquad\qquad\quad~~~{}$ GLM}\\
    \hspace{-0.4cm}\vspace{-0.1cm}
    \includegraphics[width=0.248\textwidth]{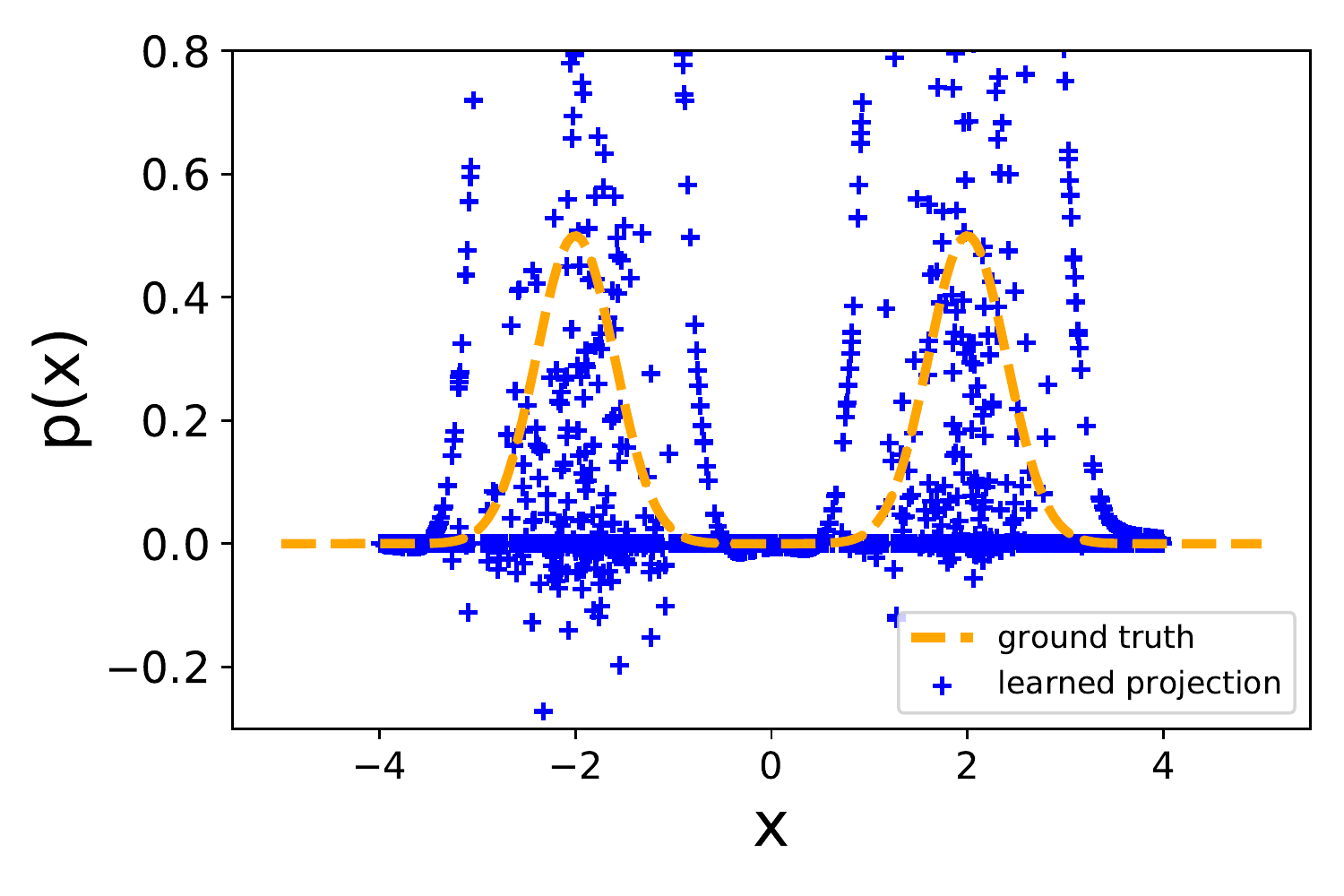}~
    \includegraphics[width=0.248\textwidth]{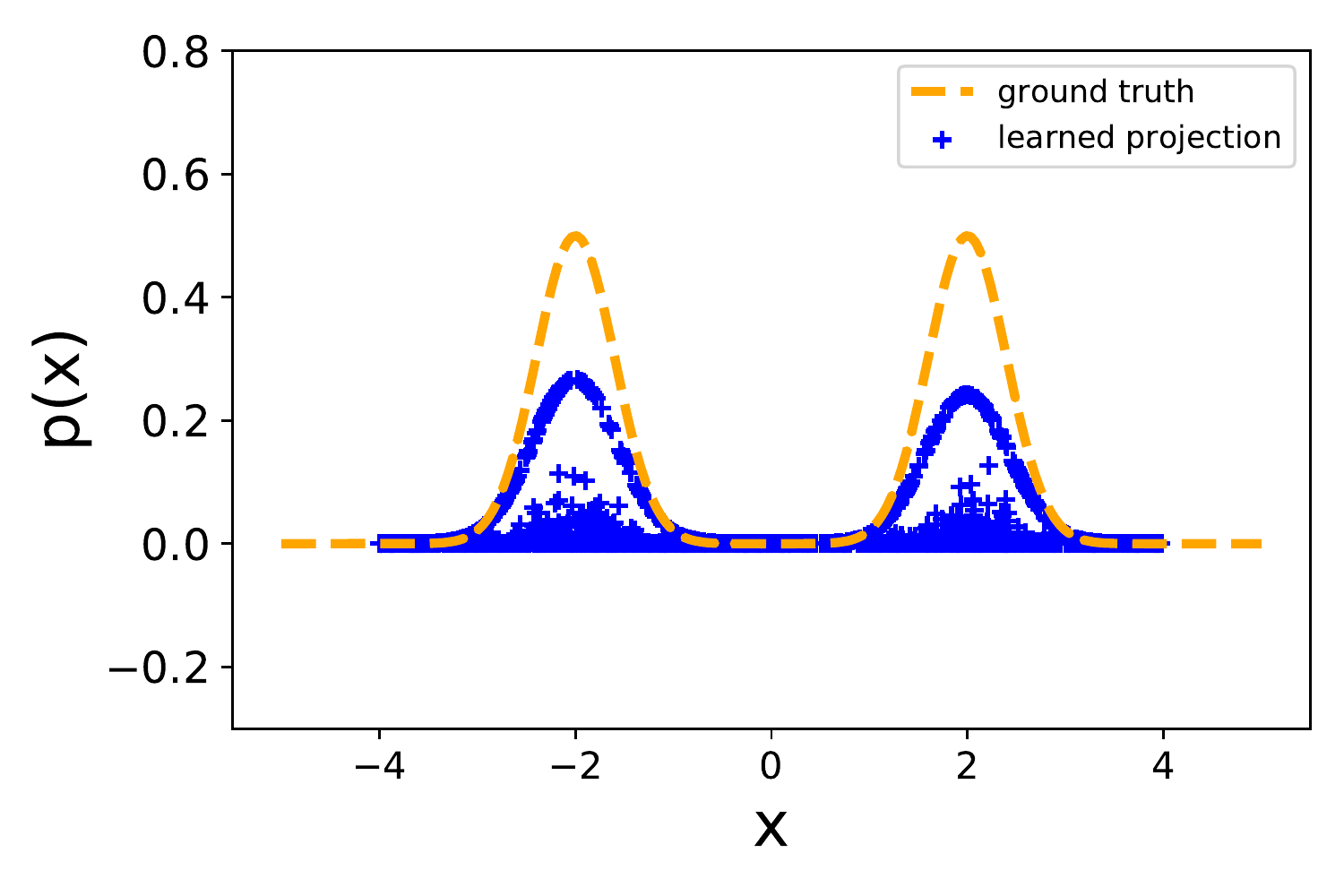}~
    \includegraphics[width=0.248\textwidth]{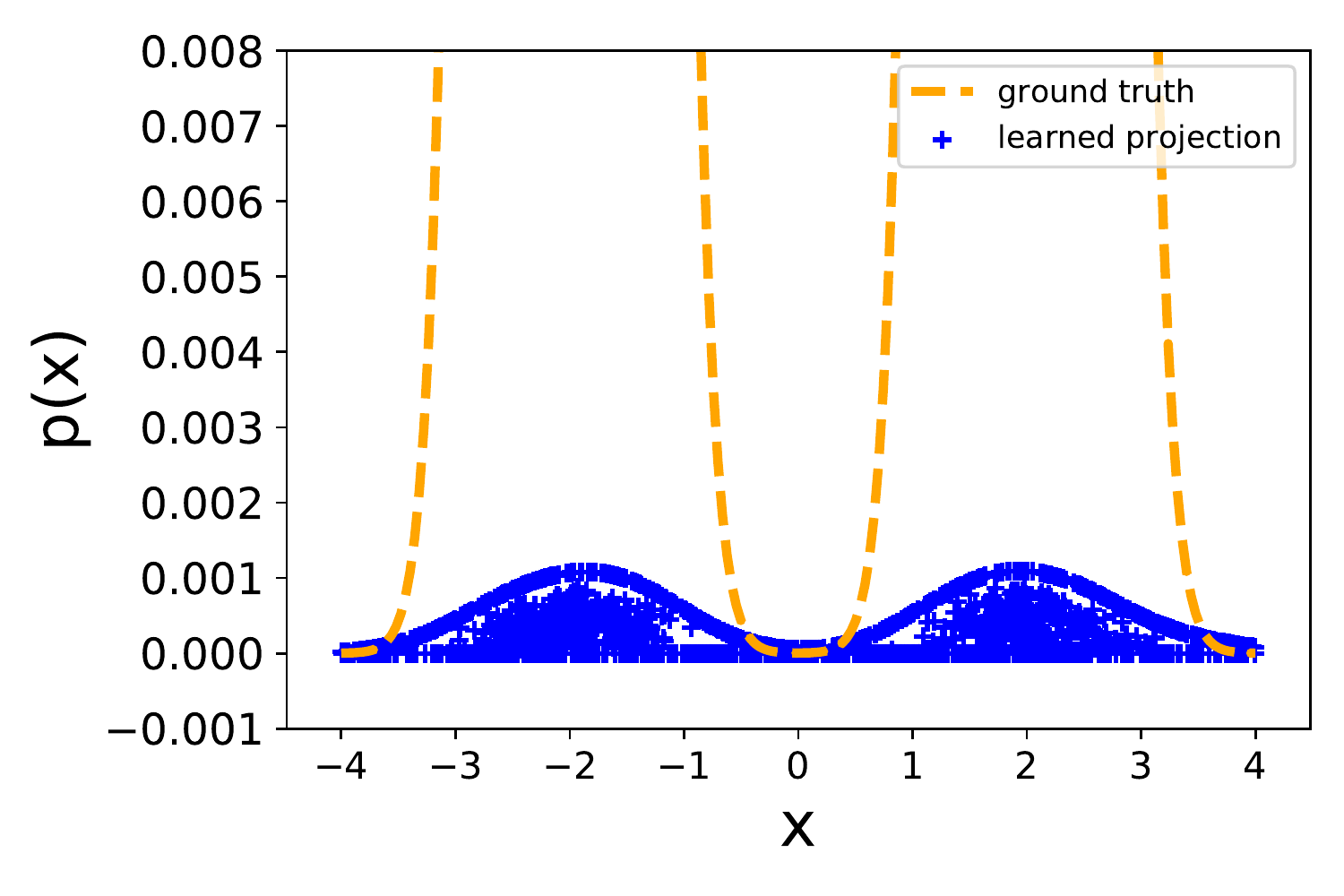}~
    \includegraphics[width=0.248\textwidth]{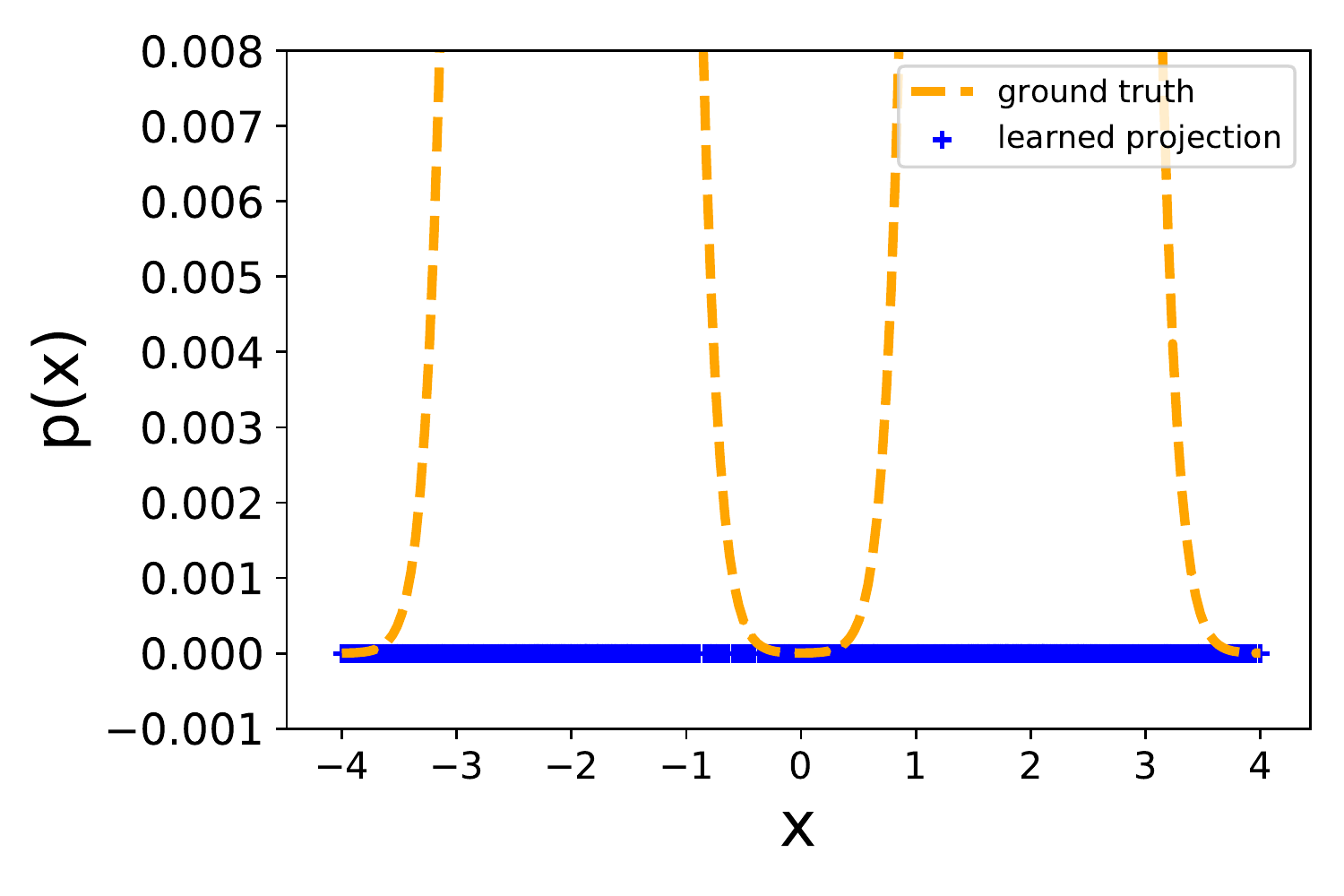}\\
    \vspace*{-.3cm}
    \caption{%
     Representation of the densities learned by the different models.}
    \label{fig:experimentdensity}
\end{figure}

Let us now analyse the results in \cref{fig:experimentdensity}. Note that in terms of performance, i.e. log likelihood on the test set, the first two models (PNM and our model) are quite close and are better than the two others.

\begin{itemize}
    \item \textbf{PNM}. As in $d=1$ we see that for $d=10$ the problems of non-negativity for PNM are exacerbated, making it not suitable to learn a probability distribution. Indeed there are low density regions where the optimization problem pushes the model to be negative. Since by constraint we have $\int fd \nu = 1$, the volume of the negative regions is used to push up the function in the regions with high density. So $\int |f| d\nu \gg 1$, while it should be $\int |f| d\nu = 1$. This is confirmed by the behavior of the cross validation.    

    \item \textbf{Our model} Our model seems to perform reasonably well. 
    \item \textbf{NCM}. This problem is particularly difficult for NCM.
    Indeed, as the width of the kernel decreases, the model is unable to learn since it overfits in the direction $e_1$ and it would require way more points than $n=1000$. However, as soon as the width of the kernel is good for $e_1$, the learnt distribution becomes too heavy tailed in the direction orthogonal to $e_1$. 
    \item \textbf{GLM}. It is interesting to note that GLM completely fails, because the measure $\nu$ which we take as a reference measure has a support which has only double variance compared to $p$, but in $10$ dimensions it corresponds to a support with way larger volume compared to the one of the target distribution. In particular, the estimation of the integral, which was possible in $d=1$ with $10000$ i.i.d. points from $\nu$, in $10$ dimensions becomes almost impossible (it would require way more sampling points). Note that we sample the points from $\nu$ to simulate the real-world situation where $p$ is a measure from which it is difficult to sample from, while $\nu$ is an simple measure to sample from which contains the support of $p$. Further experiments show that if one takes the target distribution to sample, one obtains a good model, which reassures us in the fact that this is not a coding error but a real phenomenon.
\end{itemize}

\section{Relationship to~\cite{bagnell2015learning}}
\label{app:kernelsos}
As mentioned in the main paper, the model in \cref{eq:model-non-negative} has already been considered in \cite{bagnell2015learning} with a similar goal as ours. This paper is a workshop publication that has only be lightly peer-reviewed and contains fundamental flaws. In particular, they provide an incorrect characterization of the solution of \cref{eq:problem-for-representer}, that limits the representation power of the model to the one of non-negative coefficients models, that, as we have seen in \cref{sec:background-other-models} and in \cref{ex:NCM}, has   poor approximation properties and cannot be universal. This severe limitation affects also the optimization framework (which also     only relies on general-purpose toolboxes such as CVX (\url{http://cvxr.com/cvx/}), which are not scalable to large $n$).

Indeed, in their main result, the representer theorem incorrectly characterizes $A^*$ the solution of \cref{eq:problem-for-representer} as 
\[A^* ~\in~ R_n \cap  {\cal S}(\hh)_+, \qquad R_n = \left\{\sum_{i=1}^n \alpha_i \phi(x_i) \otimes \phi(x_i) ~|~ \alpha \in \R^n \right\},\]
 and ${\cal S}(\hh)_+ = \{A \in {\cal S}(\hh)~|~ A \succeq 0\}$.
Note, however that $R_n \subseteq {\cal S}(\hh_n) \subset {\cal S}(\hh)$ by construction, where $\hh_n = \textrm{span}\{\phi(x_1),\dots, \phi(x_n)\}$. So their characterization corresponds to
\[A^* ~\in~ \set{A = \sum_{i=1}^n \alpha_i \phi(x_i) \otimes \phi(x_i) ~|~ \alpha \in \R^n, A \succeq 0}.\]
Now, for simplicity, consider the interesting case where $\phi$ is universal and $x_1,\dots,x_n$ are distinct points. Then $(\phi(x_i))_{i=1}^n$ forms a basis for $\hh_n$ and the only $\alpha_1,\dots, \alpha_n \in \R$ that guarantee $A\succeq 0$ are $\alpha_1 \geq 0,\dots, \alpha_n \geq 0$, i.e.,
\[R_n = \set{A = \sum_{i=1}^n \alpha_i \phi(x_i) \otimes \phi(x_i) ~|~ \alpha_1 \geq 0,\dots, \alpha_n \geq 0}.\]
Note that this class of operators leads only to {\em non-negative coefficients models}. Indeed, let $A \in R_n$ and denote by $k(x,x')$ the function $k(x,x') = (\phi(x)^\top \phi(x'))^2$, then
\[f_{A}(x) = \phi(x)^\top A \phi(x) = \sum_{i=1}^n \alpha_i (\phi(x)^\top \phi(x_i))^2 = \sum_{i=1}^n \alpha_i k(x,x_i), \quad \forall ~ x \in \X.\]
Since $k$ is a kernel (it is an integer power of $\phi(x)^\top \phi(x')$ that is a kernel \cite{scholkopf2002learning}) and $\alpha_1 \geq 0, \dots, \alpha_n \geq 0$, then $f_A$ belongs to the non-negative coefficients models.

Instead, we know by our \cref{thm:representer-non-negative} that $A^* \in {\cal S}(\hh_n)_+$ and more explicitly, by \cref{thm:dual} that $A^*$, the solution of \cref{eq:problem-for-representer} is characterized by the non-positive part operator of a symmetric matrix $[\cdot]_+$. By \cref{thm:universality} we already know that our model is universal while NCM and thus the characterization in \cite{bagnell2015learning} cannot be universal.



\putbib[biblio.bib]

\end{bibunit}

\end{document}

%% file: macros_bis.tex
\newcommand{\eqals}[1]{\begin{align*}#1\end{align*}}
\newcommand{\eqal}[1]{\begin{align}#1\end{align}}
\newcommand{\bpr}{\begin{proof}}
\newcommand{\epr}{\end{proof}}
\newcommand{\be}{\begin{equation}}
\newcommand{\ee}{\end{equation}}

 \newtheorem{example}{Example}
  \newtheorem{theorem}{Theorem}
  \newtheorem{lemma}{Lemma}
  \newtheorem{proposition}{Proposition}
  \newtheorem{remark}{Remark}
  \newtheorem{corollary}{Corollary}
  \newtheorem{definition}{Definition}

\newcommand{\bd}{\begin{definition}}
\newcommand{\ed}{\end{definition}}

\newcommand{\bi}{\begin{itemize}}
\newcommand{\ei}{\end{itemize}}

\newtheorem{ass}{Assumption}
\newcommand{\ba}{\begin{ass}}
\newcommand{\ea}{\end{ass}}

\newtheorem{mtd}{Method}
\newcommand{\bmtd}{\begin{mtd}}
\newcommand{\emtd}{\end{mtd}}

\newcommand{\bre}{\begin{restatable}}
\newcommand{\ere}{\end{restatable}}

\newcommand{\br}{\begin{remark}}
\newcommand{\er}{\end{remark}}

\newcommand{\bp}{\begin{proposition}}
\newcommand{\ep}{\end{proposition}}

\newcommand{\blm}{\begin{lemma}}
\newcommand{\elm}{\end{lemma}}

\newcommand{\bt}{\begin{theorem}}
\newcommand{\et}{\end{theorem}}

\newcommand{\bcor}{\begin{corollary}}
\newcommand{\ecor}{\end{corollary}}

\newcommand{\bex}{\begin{example}}
\newcommand{\eex}{\end{example}}

\crefname{proposition}{Proposition}{Propositions}

\crefname{definition}{Definition}{Definitions}

\crefname{ass}{Assumption}{Assumptions}
\crefname{equation}{Eq.}{Eqs.}
\crefname{figure}{Fig.}{Figs.}
\crefname{table}{Table}{Tables}
\crefname{section}{Sec.}{Secs.}

\crefname{theorem}{Thm.}{Thms.}

\crefname{lemma}{Lemma}{Lemmas}

\crefname{corollary}{Cor.}{Cors.}

\crefname{example}{Example}{Examples}

\crefname{appendix}{Appendix}{Appendixes}

\crefname{remark}{Remark}{Remark}

\newcommand{\X}{\mathcal{X}}
\newcommand{\Y}{\mathcal{Y}}
\newcommand{\hh}{\mathcal{H}}

\newcommand{\cx}{\mathcal{X}}
\newcommand{\cy}{\mathcal{Y}}

\newcommand{\la}{\lambda}

\newcommand{\R}{\mathbb{R}}
\newcommand{\N}{\mathbb{N}}

\newcommand{\set}[1]{\left\{#1\right\}}

\DeclareMathOperator{\tr}{Tr}
\DeclareMathOperator{\diag}{Diag}
\DeclareMathOperator{\lspan}{span}
\DeclareMathOperator{\range}{range}
\DeclareMathOperator{\noy}{Ker}
\DeclareMathOperator{\argmax}{argmax}
\DeclareMathOperator{\argmin}{argmin}
\DeclareMathOperator{\dom}{dom}